\documentclass[final,12pt]{arxiv}

\usepackage[utf8]{inputenc} 
\usepackage[T1]{fontenc}    
\usepackage{hyperref}       
\usepackage{url}            
\usepackage{booktabs}       
\usepackage{amsfonts}       
\usepackage{nicefrac}       
\usepackage{microtype}      
\usepackage{xcolor}         

\usepackage{dsfont}
\usepackage{enumitem}
\usepackage[toc,page,header]{appendix}
\usepackage{lipsum,wrapfig}

\usepackage[english]{babel}
\usepackage{amsmath, amssymb}
\usepackage{graphicx}
\usepackage{latexsym}
\usepackage{graphicx}
\usepackage{cases}
\usepackage[mathscr]{euscript}
\usepackage{xcolor}
\usepackage{colortbl}
\usepackage{thmtools}
\usepackage{thm-restate}

\usepackage{mathtools}
\usepackage{subcaption}

\usepackage{multirow}
\usepackage{wrapfig}

\usepackage{pifont}

\let\vec\undefined
\usepackage{MnSymbol}
\DeclareMathAlphabet\mathbb{U}{msb}{m}{n}
\usepackage{xpatch}

\def\Rset{\mathbb{R}}

\let\P\undefined

\DeclareMathOperator*{\P}{\mathbb{P}}
\DeclareMathOperator*{\E}{\mathbb E}
\DeclareMathOperator*{\argmax}{argmax}
\DeclareMathOperator*{\argmin}{argmin}

\DeclarePairedDelimiter{\abs}{\lvert}{\rvert} 
\DeclarePairedDelimiter{\bracket}{[}{]}
\DeclarePairedDelimiter{\curl}{\{}{\}}

\DeclarePairedDelimiter{\paren}{(}{)}

\DeclarePairedDelimiter{\ceil}{\lceil}{\rceil}

\newcommand{\sA}{{\mathscr A}}

\newcommand{\sC}{{\mathscr C}}
\newcommand{\sD}{{\mathscr D}}
\newcommand{\sE}{{\mathscr E}}
\newcommand{\sF}{{\mathscr F}}

\newcommand{\sH}{{\mathscr H}}

\newcommand{\sK}{{\mathscr K}}

\newcommand{\sM}{{\mathscr M}}

\newcommand{\sR}{{\mathscr R}}
\newcommand{\sS}{{\mathscr S}}

\newcommand{\sX}{{\mathscr X}}
\newcommand{\sY}{{\mathscr Y}}

\newcommand{\Rad}{\mathfrak R}

\newcommand{\h}{\widehat}
\newcommand{\ov}{\overline}
\newcommand{\wt}{\widetilde}
\newcommand{\e}{\epsilon}

\newcommand{\ignore}[1]{}

\def\Nset{\mathbb{N}}
\newcommand{\hh}{{\sf h}}
\newcommand{\g}{{\sf g}}
\newcommand{\rr}{{\sf r}}
\newcommand{\pp}{{\sf p}}
\newcommand{\cost}{{\mathsf{cost}}}
\newcommand{\q}{{q}}

\hypersetup{
  breaklinks   = true, 
  colorlinks   = true, 
  urlcolor     = blue, 
  linkcolor    = blue, 
  citecolor   = blue 
}

\usepackage[toc, page, header]{appendix}
\usepackage{times}
\title[Cardinality-Aware Set Prediction and Top-$k$ Classification]
      {Cardinality-Aware Set Prediction and Top-\texorpdfstring{$k$}{k} Classification}

\arxivauthor{%
 \Name{Corinna Cortes} \Email{corinna@google.com}\\
 \addr Google Research, New York%
 \AND
 \Name{Anqi Mao} \Email{aqmao@cims.nyu.edu}\\
 \addr Courant Institute of Mathematical Sciences, New York%
 \AND
 \Name{Christopher Mohri} \Email{xmohri@stanford.edu}\\
 \addr Stanford University, Stanford%
 \AND
 \Name{Mehryar Mohri} \Email{mohri@google.com}\\
 \addr Google Research and Courant Institute of Mathematical Sciences, New York%
 \AND
 \Name{Yutao Zhong} \Email{yutao@cims.nyu.edu}\\
 \addr Courant Institute of Mathematical Sciences, New York%
}

\begin{document}

\maketitle

\begin{abstract}

We present a detailed study of cardinality-aware top-$k$
classification, a novel approach that aims to learn an accurate top-$k$ set
predictor while maintaining a low cardinality.
We introduce a new target loss function tailored to this setting that
accounts for both the classification error and the cardinality of the
set predicted.  To optimize this loss function, we propose two
families of surrogate losses: cost-sensitive comp-sum losses and
cost-sensitive constrained losses. Minimizing these loss functions
leads to new cardinality-aware algorithms that we describe in detail
in the case of both top-$k$ and threshold-based classifiers.
We establish $\sH$-consistency bounds for our cardinality-aware surrogate loss
functions, thereby providing a strong theoretical foundation for our
algorithms.\ignore{To do so, we give a theoretical analysis of top-$k$
classification and demonstrate that several common surrogate loss
functions in multi-class classification, such as comp-sum and
constrained losses, are supported by $\sH$-consistency bounds with
respect to the top-$k$ loss.}
We report the results of extensive experiments on CIFAR-10, CIFAR-100,
ImageNet, and SVHN datasets demonstrating the effectiveness and
benefits of our cardinality-aware algorithms.

\end{abstract}



\section{Introduction}
\label{sec:intro}

Top-$k$ classification consists of predicting the $k$ most likely
classes for a given input, as opposed to solely predicting the single
most likely class. Several compelling reasons support the adoption of this framework.\ignore{top-$k$ classification} First, it enhances accuracy by allowing the
model to consider the top $k$ predictions, accommodating uncertainty
and providing a more comprehensive prediction. This is
particularly valuable in scenarios where multiple correct answers
exist, such as image tagging, where a top-$k$ classifier can identify
multiple relevant objects in an image. Second, top-$k$ classification
is applicable in ranking and recommendation tasks such as suggesting
the top $k$ most relevant products in e-commerce based on user
queries. The confidence scores associated with the top $k$ predictions
also serve as a means to estimate the model's uncertainty, which is crucial in applications 
requiring insight into the model's confidence
level.

The predictions of a top-$k$ classifier are also useful in several natural settings. For example, ensemble learning can benefit from top-$k$ predictions as they can be
combined from multiple models, contributing to improved overall
performance by introducing a more robust and diverse set of
predictions. In addition, top-$k$ predictions can serve as input for
downstream tasks like natural language generation or dialogue systems,
enhancing the performance of these tasks by providing a broader range
of potential candidates. Finally, the interpretability of the model's
decision-making process is enhanced by examining the top $k$ predicted
classes, allowing users to gain insights into the rationale behind the
model's predictions.

The appropriate $k$ for a task at hand may be determined by the
application itself like a recommendor system always expecting a fixed
set size to be 
returned. For other applications, it may be natural to let the
cardinality of the returned set vary with the model's confidence or other
properties of the task. Designing effective algorithms with learning guarantees for
this setting is our main goal.

In this paper, we introduce the problem of cardinality-aware set prediction\ignore{In Section~\ref{sec:cardinality}, we introduce the problem formally
and discuss cardinality-aware
set prediction}, which is to learn an accurate set predictor while
maintaining a low cardinality. The core idea is that an effective
algorithm should dynamically adjust the cardinality of its prediction sets
based on input instances.  For top-$k$ classifiers, this means
selecting a larger $k$ for difficult inputs to ensure high accuracy,
while opting for a smaller $k$ for simpler inputs to maintain low
cardinality. Similarly, for threshold-based classifiers, a lower
threshold can be used for difficult inputs to minimize the risk of
misclassification, whereas a higher threshold can be applied to
simpler inputs to reduce cardinality.

To tackle this problem, we introduce a novel target loss function\ignore{tailored to this problem} which
captures both the classification error and the cardinality of a prediction set\ignore{(Section~\ref{sec:cardinality-aware-loss})}.  Minimizing this target loss function
directly is an instance-dependent cost-sensitive learning problem,
which is intractable for most hypothesis sets.  Instead, we derive two families of general surrogate loss functions that
benefit from 
smooth properties and favorable optimization solutions.

\ignore{(Section~\ref{sec:cost-sensitive-comp-sum})
(Section~\ref{sec:cost-sensitive-constrained})}

To provide theoretical guarantees for our cardinality-aware top-$k$ approach, we first
study consistency properties for surrogate loss functions for the general top-$k$ problem for a fixed $k$. Unlike standard classification, the consistency of surrogate loss functions for the top-$k$ problem has been relatively
unexplored. A crucial property in this context is the asymptotic notion of
\emph{Bayes-consistency}, which has been extensively studied in standard binary
and multi-class classification
\citep{Zhang2003,bartlett2006convexity,zhang2004statistical,
  bartlett2008classification}. While Bayes-consistency has been
explored for various top-$k$ surrogate losses
\citep{lapin2015top,lapin2016loss,lapin2018analysis,yang2020consistency,
  thilagar2022consistent}, some face limitations.  Non-convex
``hinge-like" surrogates \citep{yang2020consistency}, surrogates inspired by
ranking \citep{usunier2009ranking}, and polyhedral surrogates
\citep{thilagar2022consistent} cannot lead to effective algorithms as
they cannot be efficiently computed and optimized.  Negative results also
indicate that several convex "hinge-like" surrogates
\citep{lapin2015top,lapin2016loss,lapin2018analysis} fail to achieve
Bayes-consistency \citep{yang2020consistency}. On the positive side, it has been shown that the logistic loss (or
cross-entropy loss used with the softmax activation) is a
Bayes-consistent loss for top-$k$ classification
\citep{lapin2015top,yang2020consistency}.
\ignore{Both the loss functions of top-$k$ and
its cardinalilty-aware extension are non-continuous and
non-differentiable, and their direct optimization is
intractable.}\ignore{A number of algorithms minimizing surrogate losses for the top-$k$ classification problem  and inspired by this
literature, we develop similar
surrogate loss functions for the cardinality-aware task. }\ignore{This raises critical
questions: Which surrogate loss functions admit theoretical guarantees
and efficient minimization properties?  Can we design accurate top-$k$
and hence cardinality-aware top-$k$ algorithms with strong consistency guarantees?}\ignore{In Section~\ref{sec:comp}}

We show that, remarkably, several widely used families of
surrogate losses used in standard multi-class classification admit
\emph{$\sH$-consistency bounds} \citep{awasthi2022h,awasthi2022multi,
  mao2023cross,MaoMohriZhong2023characterization} with respect to the
top-$k$ loss.  These are strong non-asymptotic consistency guarantees that are
 specific to the actual hypothesis set $\sH$ adopted, and therefore also imply asymptotic Bayes-consistency.  \ignore{In Section~\ref{sec:pre}, w}We 
establish this property for the broad family of \emph{comp-sum
  losses} \citep{mao2023cross}, comprised of the composition of a non-decreasing and non-negative function with the
sum exponential losses. This includes the logistic loss, the
sum-exponential loss, the mean absolute error loss, and the
generalized cross-entropy loss. 
Additionally, we extend these results
\ignore{in Appendix~\ref{app:cstnd}} to \emph{constrained losses}, a family
originally introduced for multi-class SVM
\citep{lee2004multicategory}, which includes the constrained
exponential, hinge, squared hinge, and $\rho$-margin losses.\ignore{ Many of these loss functions are known for their
smooth properties and favorable optimization solutions. } The guarantees of $\sH$-consistency provide a strong foundation for principled algorithms in
top-$k$ classification by  directly minimizing  these surrogate
loss functions. \ignore{
Finally, in Section~\ref{sec:cardinality-aware-sur}}

We then leverage these results to derive strong guarantees for the two
families of cardinality-aware surrogate losses: cost-sensitive comp-sum and
cost-sensitive constrained losses.
Both families  are obtained by
augmenting their top-$k$ counterparts \citep{lapin2015top,lapin2016loss,
  berrada2018smooth,reddi2019stochastic,
  yang2020consistency,thilagar2022consistent}\ignore{ studied in
Section~\ref{sec:comp}} with instance-dependent cost terms.
We establish strong $\sH$-consistency
bounds, implying Bayes-consistency, for both families relative to the
cardinality-aware target loss. Our $\sH$-consistency bounds
for the top-$k$ problem are further beneficial here in that the cardinality-aware 
problem can consist of fixing and selecting from a family top-$k$ classifiers--we 
now know how to effectively learn each top-$k$ classifier.

The rest of the paper is organized as follows. In Section~\ref{sec:cardinality}, we formally introduce the cardinality-aware set prediction problem along with our new families of surrogate loss functions. Section~\ref{sec:cardinality-aware-algorithms} instantiates our algorithms in the case of both top-$k$ classifiers and threshold-based classifiers, and Section~\ref{sec:pre} presents strong theoretical guarantees. In Section~\ref{sec:experiments}, as well as in Appendix~\ref{app:add} and
Appendix~\ref{app:add-conformal}, we present experimental results on
the CIFAR-10, CIFAR-100, ImageNet, and SVHN datasets, demonstrating
the effectiveness of our algorithms.

\section{Cardinality-aware set prediction}
\label{sec:cardinality}

In this section, we introduce cardinality-aware set prediction, where the goal is to devise
algorithms that dynamically adjust the prediction set's size based on
the input instance to both achieve high accuracy and maintain a low
average cardinality. Specifically, for top-$k$ classifiers, our objective is to determine a
suitable cardinality $k$ for each input $x$, with higher values of $k$
for instances that are more difficult to classify. \ignore{ More generally,
given a family of set predictors $\curl*{\g_k \colon k \in \sK}$, we
aim to select, for each input $x$, the most suitable $\g_k$, $k \in
\sK$, to ensure both a high accuracy and limited cardinality.  In
particular, $\g_k$s may be threshold-based classifiers based on some
scoring function $s$: }

To address this problem, we first define a cardinality-aware loss
function that accounts for both the classification error and the
cardinality of the set predicted
(Section~\ref{sec:cardinality-aware-loss}). However, minimizing this
loss function directly is computationally intractable for non-trivial
hypothesis sets.  Thus, to optimize it, we introduce two families of
surrogate losses: cost-sensitive comp-sum losses
(Section~\ref{sec:cost-sensitive-comp-sum}) and cost-sensitive
constrained losses (Section~\ref{sec:cost-sensitive-constrained}). We will later show that these loss functions benefits from favorable guarantees in terms of $\sH$-consistency (Section~\ref{sec:cardinality-aware-sur}). \ignore{Minimizing these loss functions
leads to novel cardinality-aware algorithms, which we describe in detail
in the case of top-$k$ classifiers and threshold-based classifiers
(Section~\ref{sec:cardinality-aware-algorithms}).}

\subsection{Cardinality-aware problem formulation and loss function}
\label{sec:cardinality-aware-loss}
\ignore{
Here, we describe the learning setup for cardinality-aware set
prediction, in particular the formulation of a cardinality-aware loss
function.}

The learning setup for cardinality-aware set
prediction is as follows.

\textbf{Problem setup.} \ignore{We first introduce the learning task of top-$k$ classification with $n \geq
2$ classes, that is seeking to ensure that the correct class label
for a given input sample is among the top $k$\textbf{} predicted classes.
}
We denote by $\sX$ the input space and $\sY = [n] \colon = \curl*{1,
  \ldots, n}$ the label space. Let $\curl*{\g_k \colon k \in \sK}$ denote a
collection of given set predictors, induced by a parameterized set predictor $g_k \colon \sX \mapsto 2^\sY$, where each $\sK \subset \Rset$ is a set
of indices. This could be a subset of the family
of top-$k$ classifiers induced by some classifier $h$, or a family of
threshold-based classifiers based on some scoring function $s\colon \sX \times\sY \mapsto \Rset$. In that case, $\g_k(x)$ then comprises the set of $y$s with a
score $s(x, y)$ exceeding the threshold $\tau_k$ defining
$\g_k$. This formulation covers as a special case standard conformal
prediction set predictors \citep{shafer2008tutorial}, as well as set
predictors defined as confidence sets described in
\citep{denis2017confidence}.
We will denote by $\abs*{\g_k(x)}$ the cardinality of the set
$\g_k(x)$ predicted by $\g_k$ for the input $x$.  To simplify the
discussion, we will assume that $\abs*{\g_k(x)}$ is an increasing
function of $k$, for any $x$. For a family of top-$k$ classifiers or
threshold-based classifiers, this simply means that they are sorted in
increasing order of $k$ or decreasing order of the threshold values.

To account for the cost associated with cardinality, we introduce a
non-negative and increasing function $\cost\colon \Rset_+ \to \Rset_+$,
where $\cost(\abs*{\g_k(x)})$ represents the \emph{cost} associated to
the cardinality $\abs*{\g_k(x)}$. Common choices for $\cost$ include
$\cost(\abs*{\g_k(x)}) = \abs*{\g_k(x)}$, or a logarithmic function
$\cost(\abs*{\g_k(x)}) = \log (\abs*{\g_k(x)})$ as in our experiments
(see Section~\ref{sec:experiments}), to moderate the magnitude of the
cost relative to the binary classification loss.  Our analysis is
general and requires no assumption about $\cost$.

Our goal is to learn to assign to each input instance $x$ the most
appropriate index $k \in \sK$ to both achieve high accuracy and
maintain a low average cardinality.

\textbf{Cardinality-aware loss function.} As in the ordinary multi-class
classification problem, we consider a family $\sR$ of scoring
functions $r \colon \sX \times \sK \to \Rset$.  For any $x$, $r(x, k)$
denotes the score assigned to the \emph{label} (or index) $k \in \sK$,
given $x \in \sX$.  The label predicted is $\rr(x) = \argmax_{k \in
  \sK} r(x, k)$, with ties broken in favor of the largest index.
To account for both
classification accuracy and cardinality cost, we define the
\emph{cardinality-aware loss function} for a scoring function $r$ and
input-output label pair $(x, y) \in \sX \times \sY$ as a linearized
loss of these two criteria:
\begin{equation}
\label{eq:target-cardinality-loss}
\ell(r, x, y) = 1_{y \notin \g_{\rr(x)}(x)} + \lambda \, \cost(\abs*{\g_{\rr(x)}(x)}),
\end{equation}
where the first term is the standard loss for a top-$k$ prediction
taking the value one when the correct label $y$ is not included in the
top-$k$ set and zero otherwise, and $\lambda > 0$ is a hyperparameter that governs the balance
between prioritizing accuracy versus limiting cardinality. The
learning problem then consists of using a labeled training sample
$(x_1, y_1), \ldots (x_m, y_m)$ drawn i.i.d.\ from some (unknown)
distribution $\sD$\ignore{over $\sX \to \sY$} to select $r \in \sR$ with a
small expected cardinality-aware loss $\E_{(x, y) \sim \sD}[\ell (r, x, y)]$.

\ignore{Note that}The loss function \eqref{eq:target-cardinality-loss} can be equivalently
expressed in terms
of an instance-dependent cost function $c \colon \sX \times \sK \times
\sY \to \Rset_{+}$:
\begin{equation}
\label{eq:target-cardinality}
\ell (r, x, y) = c(x, \rr(x), y),
\end{equation}
where $c(x, k, y) = 1_{y \notin \g_{k}(x)} + \lambda \cost(\abs*{\g_{k}(x)})$. Minimizing \eqref{eq:target-cardinality} is an
instance-dependent cost-sensitive learning problem. However, directly
minimizing this target loss is intractable.
To optimize this loss function, we introduce two families of surrogate
losses in the next sections: cost-sensitive comp-sum losses and
cost-sensitive constrained losses. \ignore{We will later establish $\sH$-consistency
bounds for these loss functions and thus their Bayes-consistency (see
Section~\ref{sec:cardinality-aware-sur}).} Note that throughout this
paper, we will denote all target (or true) losses on which performance
is measured with an $\ell$, while surrogate losses introduced for ease
of optimization are denoted by $\wt \ell$. 

\subsection{Cost-sensitive comp-sum surrogate losses}
\label{sec:cost-sensitive-comp-sum}
Our surrogate cost-sensitive comp-sum, \emph{c-comp}, losses are defined as follows:
for all $(r, x, y) \in \sR \times \sX \times \sY$,
$
\wt \ell_{\rm{c-comp}}(r, x, y) = \sum_{k \in \sK} \paren*{1 - c(x, k,
  y)} \wt \ell_{\rm{comp}}(r, x, k),
$
where the comp-sum loss $\wt \ell_{\rm{comp}}$ is defined as in \citep{mao2023cross}. That is, for
any $r$ in a hypothesis set $\sR$ and $ (x, y) \in \sX \times \sY$, 
$
\wt \ell_{\rm{comp}}(r, x, y)
= \Phi\paren*{ \sum_{y' \neq y}e^{ r(x, y') - r(x, y) } },
$
where $\Phi \colon \Rset_{+} \to \Rset_{+} $ is non-decreasing. See
Section~\ref{sec:comp} for more details. 
For example, when the logistic loss is used, we obtain the
cost-sensitive logistic loss:
\begin{equation*}
\wt \ell_{\rm{c}-\log}(r, x, y) 
= \sum_{k \in \sK} \paren*{1 - c(x, k, y)} \wt \ell_{\log}(r, x, k)
= \sum_{k \in \sK} \paren*{c(x, k, y) - 1}
\bracket*{-\log \paren*{\sum_{k' \in \sK}e^{ r(x, k') - r(x, k) }}}.
\end{equation*}
The negative log-term becomes larger as the score $r(x, k)$ increases.
Thus, the loss function imposes a greater penalty on higher scores
$r(x, k)$ through a penalty term $(c(x, k, y) - 1)$ that depends on the
cost assigned to the expert's prediction $\g_k(x)$.

\ignore{
For example, when the logistic loss is used, we obtain the
cost-sensitive logistic loss:
\begin{equation}
\label{eq:cost-log}
\mspace{-5mu} \wt \ell_{\rm{c}-\log}(r, x, y) 
= \sum_{k \in \sK} \paren*{1 - c(x, k, y)} \wt \ell_{\log}(r, x, k)
= \sum_{k \in \sK} \paren*{1 - c(x, k, y)} \log \paren*{\sum_{k' \in \sK}e^{ r(x, k') - r(x, k) }}. \mspace{-5mu}
\end{equation}
Intuitively, the target loss in  \eqref{eq:target-cardinality} penalizes $r(x, k)$ when $c(x, k, y)$
has a smallest value. Similarly, for the cost-sensitive logistic loss, it penalizes $r(x, k)$ when $c(x, k, y)$
has a smallest value as well.
}

\subsection{Cost-sensitive constrained surrogate losses}
\label{sec:cost-sensitive-constrained}
Constrained losses are defined as a summation of a function
$\Phi$ applied to the scores, subject to a constraint, as in
\citep{lee2004multicategory}. For any $r \in \sR$ and
$ (x, y) \in \sX \times \sY$, they are expressed as
\begin{align*}
\wt \ell_{\rm{cstnd}}(h, x, y)
= \sum_{y'\neq y} \Phi\paren*{-r(x, y')}, \text{ with the constraint  } \sum_{y\in \sY} r(x,y) = 0,
\end{align*}
where $\Phi \colon
\Rset \to \Rset_{+} $ is non-increasing. See Section~\ref{sec:comp}
for a detailed discussion.  Inspired by these constrained losses, we introduce a new family of
surrogate losses, \emph{cost-sensitive constrained} (\emph{c-cstnd}
losses) which are defined, for all $(r, x, y) \in \sR
\times \sX \times \sY$, by
$
  \wt \ell_{\rm{c-cstnd}}(r, x, y)
  = \sum_{k \in \sK} c(x, k, y) \Phi\paren*{-r(x, k)},
$
with the constraint $\sum_{y\in \sY} r(x,y) = 0$, where $\Phi
\colon \Rset \to \Rset_{+} $ is non-increasing.  For example, for
$\Phi(t) = e^{-t}$, we obtain the cost-sensitive constrained
exponential loss:
\[
  \wt \ell^{\rm{cstnd}}_{\rm{c}{-\exp}}(r, x, y)
  = \sum_{k \in \sK} c(x, k, y) e^{r(x, k)},
   \text{ with the constraint  } \sum_{y\in \sY} r(x,y) = 0.
\]

\section{Cardinality-aware algorithms}
\label{sec:cardinality-aware-algorithms}
Minimizing the cost-sensitive surrogate loss functions described in the previous section
directly leads to novel cardinality-aware algorithms. In this section, we briefly detail the instantiation of our algorithms in
the specific cases of top-$k$ classifiers (our main focus) and
threshold-based classifiers.

\textbf{Top-$k$ classifiers.}  Here, the collection of set predictors
is a subset of the top-$k$ classifiers, defined by $\g_k(x) =
\curl*{\hh_1(x), \ldots, \hh_k(x)}$, where ${\hh_1(x), \ldots,
  \hh_k(x)}$ are the induced top-$k$ labels for a classifier $h$.
The cardinality in this case coincides
with the index: $\abs*{\g_k(x)} = k$, for any $x \in \sX$. The
cost is defined as $c(x, k, y) = 1_{y \notin \curl*{\hh_1(x),
    \ldots, \hh_k(x)}} + \lambda \cost(k)$, where $\cost(k)$ can be chosen
to be $k$ or $\log(k)$. Thus, our cardinality-aware algorithms for
top-$k$ classification can be described as follows.
At training time, we assume access to a sample set $\curl*{(x_i,
  y_i)}_{i = 1}^m$ and the costs each top-$k$ set incurs,
$\curl*{c(x_i, k, y_i)}_{i = 1}^m$, where $k \in \sK$, a pre-fixed
subset. The goal is to minimize the target cardinality-aware loss
function $\sum_{i = 1}^m \ell (r, x_i, y_i) = \sum_{i = 1}^m
c(x_i, \rr(x_i), y_i)$ over a hypothesis set $\sR$. Our algorithm
consists of minimizing a surrogate loss such as the cost-sensitive
logistic loss, defined as $\hat r = \argmin_{r \in \sR} \sum_{i = 1}^m
\sum_{k \in \sK} \paren[\big]{1 - c(x_i, k, y_i)} \log
\paren*{\sum_{k' \in \sK} e^{r(x, k') - r(x, k)}}$.  At inference
time, we use the top-$\hat \rr (x)$ set $\curl*{\hh_1(x), \ldots,
  \hh_{\hat \rr(x)}(x)}$ for prediction, with the accuracy $1_{y \in
  \curl*{\hh_1(x), \ldots, \hh_{\hat \rr(x)}(x)}}$ and cardinality
$\hat \rr(x)$ for that instance.

In Section~\ref{sec:experiments}, we compare the
accuracy-versus-cardinality curves of our cardinality-aware algorithms obtained by varying $\lambda$
with those of top-$k$ classifiers, demonstrating the effectiveness of
our algorithms.  What $\lambda$ to select for a given application will depend on the desired accuracy. Note that the performance of the algorithm in \citep{denis2017confidence} in this setting is theoretically the same as that of top-$k$ classifiers. The algorithm is designed to maximize accuracy within a constrained cardinality of $k$, and it always reaches maximal accuracy at the boundary $K$ after the cardinality is constrained to $k \leq K$.

\textbf{Threshold-based classifiers.}
Here, the set predictor is defined via a set of thresholds $\tau_k$: $\g_k(x) = \curl*{y \in \sY \colon s(x, y) > \tau_k}$. When the set is empty, we just return $\argmax_{y \in \sY} s(x, y)$ by default. The description
of the costs and other components of the algorithms is similar to that of
top-$k$ classifiers.
A special case of threshold-based classifier is conformal prediction
\citep{shafer2008tutorial}, which is a general framework that provides 
provably valid confidence intervals for a
black-box scoring function. Split conformal prediction guarantees that
$\P(Y_{m + 1} \in C_{s, \alpha}(X_{m + 1})) \geq 1 - \alpha$ for some
scoring function $s \colon \sX \times \sY \to \Rset$, where $C_{s,
  \alpha}(X_{m + 1}) = \{y\colon s(X_{m + 1}, y) \geq \hat q_\alpha\}$
and $\hat q_\alpha$ is the $\ceil{ \alpha (m + 1) } / m$ empirical
quantile of $s(X_i, Y_i)$ over a held-out set $\{(X_i, Y_i)\}_{i =
  1}^m$ drawn i.i.d.\ from some distribution $\sD$ (or just
exchangeably). Note, however, that the framework does not
supply an effective guarantee on the size of
the sets $C_{s, \alpha}(X_{m + 1})$.

In Appendix~\ref{app:add-conformal}, we present in detail a series of
early experiments for our algorithm used with threshold-based
classifiers and include more discussion. Our experiments suggest that, when the training sample
is sufficiently large, our algorithm can outperform
conformal prediction.

\section{Theoretical guarantees}
\label{sec:pre}
Here, we present theory for our cardinality-aware algorithms. Our analysis builds on theory of
top-$k$ algorithms, and we start by providing stronger results than
previously known for top-$k$ surrogates.

\subsection{Preliminaries}
We denote by $\sD$ a distribution over
$\sX \times \sY$ and write $p(x, y) = \sD\paren*{Y = y \mid X = x}$ for the conditional probability of $Y = y$ given $X = x$, and use $p(x) = \paren*{p(x, 1), \ldots, p(x, n)}$ to denote the
corresponding conditional probability vector. 
We denote by $\ell \colon \sH_{\rm{all}} \times \sX \times \sY \to
\Rset$ a loss function defined for the family of all measurable
functions $\sH_{\rm{all}}$.  Given a hypothesis set $\sH \subseteq
\sH_{\rm{all}}$, the conditional error of a hypothesis $h$ and the
best-in-class conditional error are defined as follows:
$
  \sC_{\ell}(h, x)
  = \E_{y \mid x}\bracket*{\ell(h, x, y)}
  = \sum_{y \in \sY} p(x, y) \ell(h, x, y)
  \text{ and } \sC^*_{\ell}(\sH, x)
  = \inf_{h \in \sH} \sC_{\ell}(h, x).
$
Accordingly, the generalization error of a hypothesis $h$ and the
best-in-class generalization error are defined by:
$
  \sE_{\ell}(h) 
  = \E_{(x, y) \sim \sD} \bracket*{\ell(h, x, y)}
  = \mathbb{E}_x \bracket*{\sC_{\ell}(h, x)}
  \text{ and } \sE^*_{\ell}(\sH) 
  = \inf_{h \in \sH} \sE_{\ell}(h)
  = \inf_{h \in \sH} \mathbb{E}_x \bracket*{\sC_{\ell}(h, x)}.   
$
Given a score vector $\paren*{h(x, 1), \ldots, h(x, n)}$ generated by
hypothesis $h$, we sort its components in decreasing order and write
$\hh_k(x)$ to denote the $k$-th label, that is $h(x, \hh_1(x)) \geq
h(x, \hh_2(x)) \geq \ldots \geq h(x,
\hh_n(x))$. Similarly, for a given conditional probability vector
$p(x) = \paren*{p(x, 1), \ldots, p(x, n)}$, we write $\pp_k(x)$ to
denote the $k$-th element in decreasing order, that is $p(x,\pp_1(x))
\geq p(x,\pp_2(x)) \geq \ldots \geq p(x, \pp_n(x))$. In the event of a
tie for the $k$-th highest score or conditional probability, the label
$\hh_k(x)$ or $\pp_k(x)$ is selected based on the highest index when
considering the natural order of labels.

The target generalization error for top-$k$ classification is given by
the top-$k$ loss, which is denoted by $\ell_{k}$ and defined, for any
hypothesis $h$ and $(x, y) \in \sX \times \sY$ by
\begin{equation*}
\ell_{k}(h, x, y) = 1_{y \notin \curl*{\hh_1(x), \ldots, \hh_k(x)}}.   
\end{equation*}
The loss takes value one when the correct label $y$ is not
included in the top-$k$ predictions made by the hypothesis $h$, zero
otherwise. In the special case where $k = 1$, this is precisely the
familiar zero-one classification loss. Like the zero-one loss,
optimizing the top-$k$ loss is NP-hard for common hypothesis
sets. Therefore, alternative surrogate losses are typically used
to design learning algorithms. A crucial property of these surrogate losses is
\emph{Bayes-consistency}. This requires that, asymptotically, nearly
minimizing a surrogate loss over the family of all measurable
functions leads to the near minimization of the top-$k$ loss over the
same family \citep{steinwart2007compare}.
\begin{definition}
A surrogate loss $\wt \ell$ is said to be \emph{Bayes-consistent with
respect to the top-$k$ loss $\ell_k$} if, for all given sequences of
hypotheses $\curl*{h_n}_{n \in \Nset} \subset \sH_{\rm{all}}$ and any
distribution, $\lim_{n \to \plus \infty} \sE_{\wt \ell}\paren*{h_n} -
\sE^*_{\wt \ell}\paren*{\sH_{\rm{all}}} = 0$ implies $\lim_{n \to \plus
  \infty} \sE_{\ell_{k}}\paren*{h_n} -
\sE^*_{\ell_{k}}\paren*{\sH_{\rm{all}}} = 0$.
\end{definition}
\vskip -0.05in
Bayes-consistency is an asymptotic guarantee and applies only to the
family of all measurable functions. Recently,
\citet*{awasthi2022h,awasthi2022multi} (see also
\citep{AwasthiMaoMohriZhong2023theoretically,awasthi2023dc,MaoMohriZhong2023ranking,MaoMohriZhong2023rankingabs,MaoMohriZhong2023structured,MaoMohriMohriZhong2023twostage,MaoMohriZhong2023score,MaoMohriZhong2023predictor,MaoMohriZhong2023deferral,mao2024h,mao2024regression,mao2024universal,MohriAndorChoiCollinsMaoZhong2023learning}) proposed a stronger consistency
guarantee, referred to as \emph{$\sH$-consistency bounds}. These are
upper bounds on the target estimation error in terms of the surrogate
estimation error that are non-asymptotic and hypothesis set-specific.
\begin{definition}
Given a hypothesis set $\sH$, a surrogate loss $\wt \ell$ is said to admit
an $\sH$-consistency bound with respect to the top-$k$ loss $\ell_k$
if, for some non-decreasing function $f$, the following inequality
holds for all $h \in \sH$ and for any distribution:
$
f \paren*{\sE_{\ell_{k}}\paren*{h} - \sE^*_{\ell_{k}}\paren*{\sH}}
  \leq \sE_{\wt \ell}\paren*{h} - \sE^*_{\wt \ell} \paren*{\sH}.
$
\end{definition}
\vskip -0.05in
We refer to $\sE_{\ell_{k}}\paren*{h} - \sE^*_{\ell_{k}}\paren*{\sH}$
as the target estimation error and $\sE_{\wt \ell}\paren*{h} -
\sE^*_{\wt \ell} \paren*{\sH}$ as the surrogate estimation error. These bounds imply Bayes-consistency when $\sH = \sH_{\rm{all}}$, by
taking the limit.

A key quantity appearing in $\sH$-consistency bounds is the
\emph{minimizability gap}, which measures the difference between the
best-in-class generalization error and the expectation of the
best-in-class conditional error, defined for a given hypothesis set
$\sH$ and a loss function $\ell$ by:
$
\sM_{\ell}(\sH) = \sE^*_{\ell}(\sH) - \mathbb{E}_x \bracket*{\sC^*_{\ell}(\sH, x)}.
$
As shown by \citet{mao2023cross}, the minimizability gap is
non-negative and is upper bounded by the approximation error
$\sA_{\ell}(\sH) = \sE^*_{\ell}(\sH) - \sE^*_{\ell}(\sH_{\rm{all}})$:
$0 \leq \sM_{\ell}(\sH) \leq \sA_{\ell}(\sH)$. When $\sH =
\sH_{\rm{all}}$ or more generally $\sA_{\ell}(\sH) = 0$, the
minimizability gap vanishes. However, in general, it is non-zero and
provides a finer measure than the approximation error. Thus,
$\sH$-consistency bounds provide a stronger guarantee than the excess
error bounds.

\subsection{Theoretical guarantees for top-$k$ surrogate losses}
\label{sec:comp}

We study the surrogate loss families of
\emph{comp-sum} losses and \emph{constrained} losses in multi-class
classification, which have been shown in the past to benefit from
$\sH$-consistency bounds with respect to the zero-one classification
loss, that is $\ell_k$ with $k = 1$
\citep{awasthi2022multi,mao2023cross} (see also
\citep{zheng2023revisiting,MaoMohriZhong2023characterization}). We 
extend these results to top-$k$ classification and prove
$\sH$-consistency bounds for these loss functions with respect to
$\ell_k$ for any $1 \leq k \leq n$.

Another commonly used family of surrogate losses in
multi-class classification is the \emph{max} losses, which are defined
through a convex function, such as the hinge loss function applied to
the margin \citep{crammer2001algorithmic,awasthi2022multi}. However,
as shown in \citep{awasthi2022multi}, no non-trivial $\sH$-consistency
guarantee holds for max losses with respect to $\ell_k$, even when $k
= 1$.

We first characterize the best-in-class conditional error and the
conditional regret of top-$k$ loss, which will be used in the analysis
of $\sH$-consistency bounds.  We denote by $S^{[k]} = \curl*{X \subset
  S \mid |X| = k}$ the set of all $k$-subsets of a set $S$. We will
study any hypothesis set that is regular.
\begin{definition}
  Let $A(n, k)$ be the set of ordered $k$-tuples with distinct
  elements in $[n]$.  We say that a hypothesis set $\sH$ is
  \emph{regular for top-$k$ classification}, if the top-$k$
  predictions generated by the hypothesis set cover all possible
  outcomes:
$
  \forall x \in \sX,\,
  \curl*{(\hh_1(x), \dots, \hh_k(x)) \colon h \in \sH} = A(n, k).
$
\end{definition}
Common hypothesis sets such as that of linear models or neural
networks, or the family of all measurable functions, are all regular
for top-$k$ classification.

\begin{restatable}{lemma}{RegretTarget}
\label{lemma:regret-target}
Assume that $\sH$ is regular. Then, for any $h \in \sH$ and $x \in
\sX$, the best-in-class conditional error and the conditional regret
of the top-$k$ loss can be expressed as follows:
\begin{align*}
  \sC^*_{\ell_k}(\sH, x)
  & = 1 - \sum_{i = 1}^k p(x, \pp_i(x))
\quad \Delta \sC_{\ell_k, \sH}(h, x)
= \sum_{i = 1}^k \bracket*{p(x, \pp_i(x)) - p(x, \hh_i(x))}.
\end{align*}
\end{restatable}
The proof is included in Appendix~\ref{app:regret-target}. For $k = 1$, the result coincides with the known identities for
standard multi-class classification with regular hypothesis sets
\citep[Lemma~3]{awasthi2022multi}.

As with \citep{awasthi2022multi,mao2023cross}, in the following
sections, we will consider hypothesis sets that are symmetric and
complete. This includes the class of linear models and neural networks
typically used in practice, as well as the family of all measurable
functions.
We say that a hypothesis set $\sH$ is \emph{symmetric} if it is
independent of the ordering of labels. That is, for all $y \in \sY$,
the scoring function $x \mapsto h(x, y)$ belongs to some real-valued
family of functions $\sF$. We say that a hypothesis set is
\emph{complete} if, for all $(x, y) \in \sX \times \sY$, the set of
scores $h(x, y)$ can span over the real numbers, that is, $\curl*{h(x,
  y) \colon h \in \sH} = \Rset$. Note that any symmetric and complete
hypothesis set is regular for top-$k$ classification.

Next, we analyze the broad family of comp-sum losses, which
includes the commonly used logistic loss (or cross-entropy loss used
with the softmax activation) as a special case.

Comp-sum losses are defined as the composition of a function $\Phi$
with the sum exponential losses, as in \citep{mao2023cross}. For
any $h \in \sH$ and $ (x, y) \in \sX \times \sY$, they are expressed
as
\[
\wt \ell_{\rm{comp}}(h, x, y)
= \Phi\paren*{ \sum_{y' \neq y}e^{ h(x, y') - h(x, y) } },
\]
where $\Phi \colon \Rset_{+} \to \Rset_{+} $ is non-decreasing. When
$\Phi$ is chosen as the function $t \mapsto \log(1 + t)$, $t \mapsto
t$, $t \mapsto 1 - \frac{1}{1 + t}$ and $t \mapsto \frac{1}{\q}
\paren*{1 - \paren*{\frac{1}{1 + t}}^{\q} }$, $\q \in (0, 1)$,
$\wt \ell_{\rm{comp}}(h, x, y)$ coincides with the most commonly used (multinomial) logistic
loss, defined as
$\wt \ell_{\log}(h, x, y) = \log \paren*{\sum_{y' \in \sY}e^{ h(x, y')
    - h(x, y) }}$
\citep{Verhulst1838,Verhulst1845,Berkson1944,Berkson1951}, the
sum-exponential loss $\wt \ell_{\exp}(h,
x, y) = \sum_{y' \neq y} e^{ h(x, y') - h(x, y) }$
\citep{weston1998multi,awasthi2022multi} which is widely used in multi-class boosting
\citep{saberian2011multiclass,mukherjee2013theory,KuznetsovMohriSyed2014}, the mean absolute error loss
$\wt \ell_{\rm{mae}}(h, x, y) = 1
- \bracket*{\sum_{y' \in \sY}e^{ h(x, y') - h(x, y) }}^{-1}$ known to
be robust to label noise for training neural networks
\citep{ghosh2017robust}, and the generalized cross-entropy loss $\wt \ell_{\rm{gce}}(h, x, y) =
\frac{1}{\q}\bracket*{1 - \bracket*{\sum_{y' \in \sY}e^{ h(x, y')
      - h(x, y) }}^{-\q}}$, $\q \in (0, 1)$, a
generalization of the logistic loss and mean absolute error loss for
learning deep neural networks with noisy labels \citep{zhang2018generalized},
respectively. We specifically study these loss functions and
show that they benefit from $\sH$-consistency bounds with respect to
the top-$k$ loss. 

\begin{restatable}{theorem}{BoundComp}
\label{thm:bound-comp}
Assume that $\sH$ is symmetric and complete. Then, for any $1 \leq k
\leq n$, the following $\sH$-consistency bound holds for the comp-sum loss:
\begin{align*}
\sE_{\ell_k}(h) - \sE^*_{\ell_k}(\sH) + \sM_{\ell_k}(\sH)
 \leq k \psi^{-1} \paren*{ \sE_{\wt \ell_{\rm{comp}}}(h)
    - \sE^*_{\wt \ell_{\rm{comp}}}(\sH) + \sM_{\wt \ell_{\rm{comp}}}(\sH) },
\end{align*}
In the special case where $\sA_{\wt \ell_{\rm{comp}}}(\sH)
= 0$, for any $1 \leq k \leq n$, the following upper bound holds:
\begin{align*}
  \sE_{\ell_k}(h) - \sE^*_{\ell_k}(\sH)
  \leq k \psi^{-1} \paren*{ \sE_{\wt \ell_{\rm{comp}}}(h) - \sE^*_{\wt
  \ell_{\rm{comp}}}(\sH) },
\end{align*}
where $\psi(t) = \frac{1 - t}{2}\log(1 - t) + \frac{1 + t}{2}\log(1+
t)$, $t \in [0,1]$ when $\wt \ell_{\rm{comp}}$ is $\wt \ell_{\log}$;
$\psi(t) = 1 - \sqrt{1 - t^2}$, $t \in [0,1]$ when $\wt
\ell_{\rm{comp}}$ is $\wt \ell_{\exp}$; $\psi(t) = t / n$
when $\wt \ell_{\rm{comp}}$ is $\wt \ell_{\rm{mae}}$; and $\psi(t) = \frac{1}{\q n^{\q}}
\bracket*{\bracket*{\frac{\paren*{1 + t}^{\frac1{1 - \q }} +
      \paren*{1 - t}^{\frac1{1 - \q }}}{2}}^{1 - \q } -1}$,
for all $\q \in (0,1)$, $t \in [0, 1]$ when $\wt \ell_{\rm{comp}}$ is
$\wt \ell_{\rm{gce}}$.
\end{restatable}

The proof is included in Appendix~\ref{app:bound-comp}. The second part
follows from the fact that when $\sA_{\wt \ell_{\rm{comp}}}(\sH) = 0$, the
minimizability gap $\sM_{\wt \ell_{\rm{comp}}}(\sH)$ vanishes.  By taking the
limit on both sides, Theorem~\ref{thm:bound-comp} implies the
$\sH$-consistency and Bayes-consistency of comp-sum losses with respect
to the top-$k$ loss. It further shows that, when the estimation error
of $\wt \ell_{\rm{comp}}$ is reduced to $\e > 0$, then the estimation error of
$\ell_{k}$ is upper bounded by $k \psi^{-1}(\e)$, which, for a sufficiently small $\e$, is
approximately $k \sqrt{2\e}$ for $\wt \ell_{\log}$ and $\wt
\ell_{\exp}$; $kn\e$ for $\wt \ell_{\rm{mae}}$; and $k
\sqrt{2 n^{\q} \e}$ for $\wt \ell_{\rm{gce}}$. Note that different from the other losses, the bound
for the mean absolute error loss is only linear. The downside  of this more
favorable linear rate is the dependency on the number of classes and
the fact that the mean absolute error loss is harder to optimize
\citep{zhang2018generalized}. The
bound for the generalized cross-entropy loss depends on both the
number of classes $n$ and the parameter $\q$.  

In the proof, we used the fact that the conditional regret of the top-$k$ loss is the sum of $k$ differences between two probabilities. We then upper bounded each difference with the conditional regret of the comp-sum loss, using a hypothesis based on the two probabilities. The final bound is derived by summing these differences.  In Appendix~\ref{app:novelty}, we detail the technical challenges and the novelty.

The key quantities in our $\sH$-consistency bounds are the
minimizability gaps, which can be upper bounded by the approximation
error, or more refined terms, depending on the magnitude of the
parameter space, as discussed by \citet{mao2023cross}. As pointed out
by these authors, these quantities, along with the functional form,
can help compare different comp-sum loss functions. In Appendix~\ref{app:min_re}, we further discuss the important role of minimizability gaps
under the realizability assumption, and the connection with some
negative results of \citet{yang2020consistency}.

Constrained losses are defined as a summation of a function
$\Phi$ applied to the scores, subject to a constraint, as shown in
\citep{lee2004multicategory}. For any $h \in \sH$ and
$ (x, y) \in \sX \times \sY$, they are expressed as
\begin{align*}
\wt \ell_{\rm{cstnd}}(h, x, y)
= \sum_{y'\neq y} \Phi\paren*{-h(x, y')}, \text{ with the constraint  } \sum_{y\in \sY} h(x,y) = 0,
\end{align*}
where $\Phi \colon
\Rset \to \Rset_{+} $ is non-increasing. In Appendix~\ref{app:cstnd},
we study this family of loss functions and show that several benefit from $\sH$-consistency bounds with respect to the top-$k$ loss. In Appendix~\ref{app:generalization}, we provide generalization bounds for the top-$k$ loss in terms of finite samples (Theorems~\ref{Thm:Gbound-comp} and \ref{Thm:Gbound-cstnd}).

\subsection{Theoretical guarantees for cardinality-aware surrogate losses}
\label{sec:cardinality-aware-sur}

The strong theoretical results of the previous sections establish the
effectiveness of comp-sum and constrained losses as surrogate losses
for the target top-$k$ loss for common hypothesis sets used in
practice.  Building on this foundation, we expand our analysis to their cost-sensitive variants in the 
study of cardinality-aware set prediction in Section~\ref{sec:cardinality}. We
derive $\sH$-consistency bounds for these loss functions, thereby also
establishing their Bayes-consistency.  To
do so, we characterize the conditional regret of the target
cardinality-aware loss function in
Lemma~\ref{lemma:regret-target-cost}, which can be found in
Appendix~\ref{app:cost}. For this analysis, we will assume, without
loss of generality, that the cost $c(x, k, y)$ takes values in $[0,
  1]$ for any $(x, k, y) \in \sX \times \sK \times \sY$, which can be
achieved by normalizing the cost function.

We will use $\wt \ell_{c-\log}$, $\wt \ell_{c-\exp}$, $\wt \ell_{c-\rm{gce}}$
and $\wt \ell_{c-\rm{mae}}$ to denote the corresponding cost-sensitive
counterparts for $\wt \ell_{\log}$, $\wt \ell_{\exp}$, $\wt \ell_{\rm{gce}}$
and $\wt \ell_{\rm{mae}}$, respectively.  Next, we show that
these cost-sensitive surrogate loss functions benefit from
$\sH$-consistency bounds with respect to the target loss $\ell$ given in \eqref{eq:target-cardinality-loss}.
\begin{restatable}{theorem}{BoundCostComp}
\label{thm:bound-cost-comp}
Assume that $\sR$ is symmetric and complete. Then, the following bound holds for the cost-sensitive comp-sum loss: for all $r \in \sR$ and for any distribution,
\begin{align*}
\sE_{\ell}(r) - \sE^*_{\ell}(\sR) + \sM_{\ell}(\sR) \leq \gamma \paren*{ \sE_{\wt \ell_{\rm{c-comp}}}(r)
    - \sE^*_{\wt \ell_{\rm{c-comp}}}(\sR) + \sM_{\wt \ell_{\rm{c-comp}}}(\sR) };
\end{align*}
When $\sR = \sR_{\rm{all}}$, the following holds:
$\sE_{\ell}(r) - \sE^*_{ \ell}(\sR_{\rm{all}})
  \leq \gamma \paren*{ \sE_{\wt \ell_{\rm{c-comp}}}(r)
    - \sE^*_{\wt \ell_{\rm{c-comp}}}(\sR_{\rm{all}})}$,
where $\gamma(t) = 2\sqrt{t}$ when
$\wt \ell_{\rm{c-comp}}$ is either $\wt \ell_{\rm{c-log}}$ or
$\wt \ell_{\rm{c-exp}}$; $\gamma(t) = 2\sqrt{\abs*{\sK}^{\q} t}$ when
$\wt \ell_{\rm{c-comp}}$ is $\wt \ell_{\rm{c-gce}}$; and
$\gamma(t) = \abs*{\sK} t$ when
$\wt \ell_{\rm{c-comp}}$ is $\wt \ell_{\rm{c-mae}}$.
\end{restatable}
The proof is included in Appendix~\ref{app:bound-cost-comp}. The
second part follows from the fact that when $\sR = \sR_{\rm{all}}$,
all the minimizability gaps vanish. In particular,
Theorem~\ref{thm:bound-cost-comp} implies the Bayes-consistency of
cost-sensitive comp-sum losses. The bounds for cost-sensitive
generalized cross-entropy and mean absolute error loss depend on the
number of set predictors, making them less favorable when $\abs*{\sK}$ is large. As pointed out earlier, while the cost-sensitive mean absolute error
loss admits a linear rate, it is difficult to optimize even in the
standard classification, as reported by \citet{zhang2018generalized}.

In the proof, we represented the comp-sum loss as a function of the softmax and introduced a softmax-dependent function $\sS_{\mu}$ to upper bound the conditional regret of the target cardinality-aware loss function by that of the cost-sensitive comp-sum loss. This technique is novel and differs from the approach used in the standard scenario (Section~\ref{sec:comp}).

We
will use $\wt \ell^{\rm{cstnd}}_{\rm{c}{-\exp}}$, $\wt \ell_{c-\rm{sq-hinge}}$, $\wt \ell_{c-\rm{hinge}}$ and $\wt
\ell_{c-\rho}$ to denote the corresponding cost-sensitive counterparts
for $\wt \ell^{\rm{cstnd}}_{\exp}$, $\wt \ell_{\rm{sq-hinge}}$, $\wt \ell_{\rm{hinge}}$ and $\wt
\ell_{\rho}$, respectively.  Next, we show that
these cost-sensitive surrogate losses benefit from
$\sH$-consistency bounds with respect to the target loss $\ell$ given in \eqref{eq:target-cardinality-loss}.
\begin{restatable}{theorem}{BoundCostCstnd}
\label{thm:bound-cost-cstnd}
Assume that $\sR$ is symmetric and complete. Then, the following bound holds for the cost-sensitive constrained loss: for all $r \in \sR$ and for any distribution,
\begin{align*}
\sE_{\ell}(r) - \sE^*_{\ell}(\sR) + \sM_{\ell}(\sR) \leq \gamma \paren*{ \sE_{\wt \ell_{\rm{c-cstnd}}}(r) - \sE^*_{\wt \ell_{\rm{c-cstnd}}}(\sR) + \sM_{\wt \ell_{\rm{c-cstnd}}}(\sR) };
\end{align*}
When $\sR = \sR_{\rm{all}}$, the following holds:
$\sE_{\ell}(r) - \sE^*_{\ell}(\sR_{\rm{all}}) \leq \gamma \paren[\big]{ \sE_{\wt \ell_{\rm{c-cstnd}}}(r) - \sE^*_{\wt \ell_{\rm{c-cstnd}}}(\sR_{\rm{all}})}$,
where $\gamma(t) = 2\sqrt{t}$ when
$\wt \ell_{\rm{c-cstnd}}$ is $\wt \ell^{\rm{cstnd}}_{\rm{c}{-\exp}}$ or
$\wt \ell_{c-\rm{sq-hinge}}$; $\gamma(t) = t$ when
$\wt \ell_{\rm{c-cstnd}}$ is $\wt \ell_{c-\rm{hinge}}$ or $\wt \ell_{c-\rho}$.
\end{restatable}
The proof is included in Appendix~\ref{app:bound-cost-cstnd}. The
second part follows from the fact that when $\sR = \sR_{\rm{all}}$,
all the minimizability gaps vanish.  In particular,
Theorem~\ref{thm:bound-cost-cstnd} implies the Bayes-consistency of
cost-sensitive constrained losses. Note that while the constrained
hinge loss and $\rho$-margin loss have a more favorable linear rate in
the bound, their optimization may be more challenging compared to
other smooth loss functions.

\section{Experiments}
\label{sec:experiments}

Here, we report empirical results for our cardinality-aware
algorithm and show that it consistently outperforms 
top-$k$ classifiers on benchmark datasets CIFAR-10, CIFAR-100
\citep{Krizhevsky09learningmultiple}, SVHN \citep{Netzer2011} and
ImageNet \citep{deng2009imagenet}.

We used the outputs of the
second-to-last layer of ResNet \citep{he2016deep} as features for the
CIFAR-10, CIFAR-100 and SVHN datasets. For the ImageNet dataset, we
used the CLIP \citep{radford2021learning} model to extract features. We adopted a linear model, trained using multinomial logistic loss, for the classifier $h$ on the
extracted features from the datasets. We used a two-hidden-layer feedforward neural network
with ReLU activation functions \citep{nair2010rectified} for the
cardinality selector $r$. Both the classifier $h$ and the cardinality
selector $r$ were trained using the Adam optimizer
\citep{kingma2014adam}, with a learning rate of $1\times 10^{-3}$, a
batch size of $128$, and a weight decay of $1\times 10^{-5}$.

Figure~\ref{fig:topk} compares the accuracy versus cardinality curves
of the cardinality-aware algorithm with that of top-$k$ classifiers induced by $h$ for the various datasets.
The accuracy of a top-$k$ classifier is measured by $\E_{(x, y) \sim
  S}\bracket*{1 - \ell_{k}(h, x, y)}$, that is the fraction of the
sample in which the top-$k$ predictions include the true label. It
naturally grows as the cardinality $k$ increases, as shown in
Figure~\ref{fig:topk}.
The accuracy of the carnality-aware algorithms is measured by $\E_{(x,
  y) \sim S} \bracket[\big]{1_{y \in \curl*{\hh_1(x), \ldots, \hh_{ \rr(x)}(x)}}}$, that is the fraction
of the sample in which the predictions selected by the model $r$
include the true label, and the corresponding cardinality is measured
by $\E_{(x, y) \sim S} \bracket*{\rr(x)}$, that is the average size of
the selected predictions.
The cardinality selector $r$ was trained by minimizing the
cost-sensitive logistic loss $\wt \ell_{\rm{c}-\log}$ with the cost $c(x, k, y)$ defined as
$\ell_{k}(h, x, y) + \lambda \log(k)$ and normalized to $[0, 1]$ through division by its maximum value over $\sX \times \sK \times \sY$. We allow for top-$k$ experts with $k \in \sK = \curl*{1, 2, 4,
  8}$ and vary $\lambda$. Starting from high values of $\lambda$, as $\lambda$
decreases in Figure~\ref{fig:topk}, our cardinality-aware algorithm
yields solutions with higher average cardinality and increased accuracy. This is because $\lambda$ controls the trade-off
between cardinality and accuracy. The plots end at $\lambda=0.01$.

Figure~\ref{fig:topk} shows that the cardinality-aware algorithm is
superior across the CIFAR-100, ImageNet, CIFAR-10 and SVHN
datasets. For a given cardinality, the cardinality-aware algorithm
always achieves higher accuracy than a top-$k$ classifier. In other
words, to achieve the same level of accuracy, the predictions made by
the cardinality-aware algorithm can be significantly smaller in size
compared to those made by the corresponding top-$k$ classifier.
In particular, on the CIFAR-100, CIFAR-10 and SVHN datasets, the
cardinality-aware algorithm achieves the same accuracy (98\%) as the
top-$k$ classifier while using roughly only half of the
cardinality. As with the ImageNet dataset, it achieves the same
accuracy (95\%) as the top-$k$ classifier with only two-thirds of the
cardinality. This illustrates the effectiveness of our
cardinality-aware algorithm.

\begin{figure}[t]
\vskip -.5in
\begin{center}
\begin{tabular}{@{\hspace{-.5cm}}c@{\hspace{0cm}}c@{\hspace{0cm}}c@{\hspace{0cm}}c@{\hspace{0cm}}}
\includegraphics[scale=0.24]{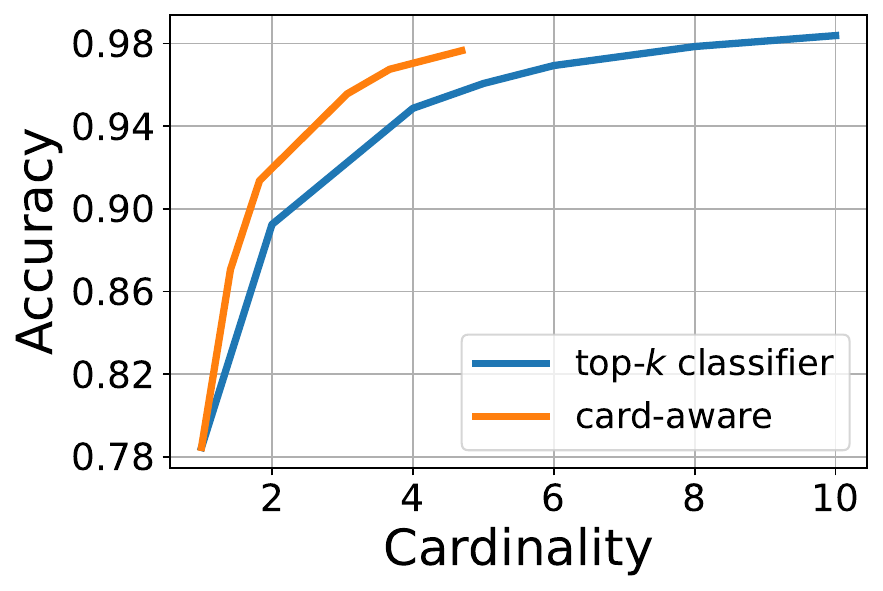} & 
\includegraphics[scale=0.24]{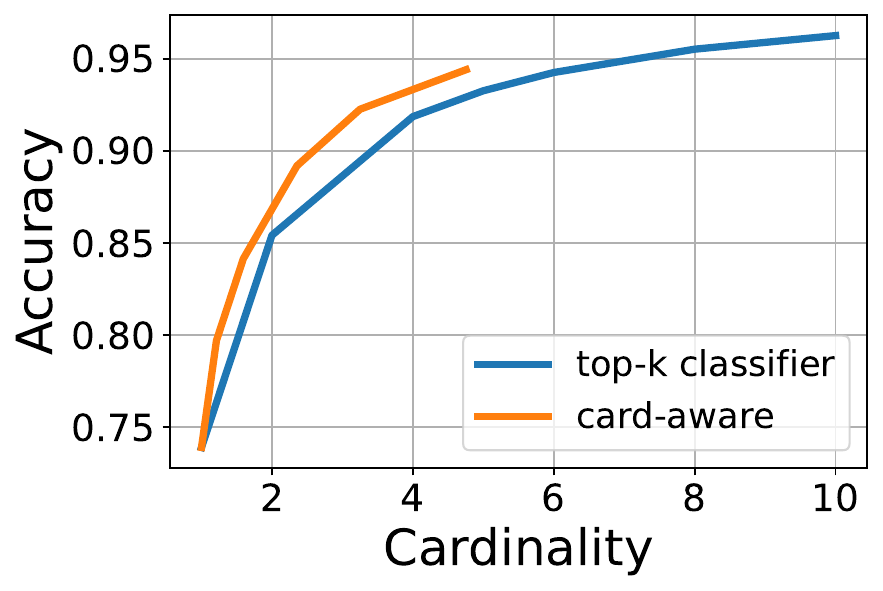} &
\includegraphics[scale=0.24]{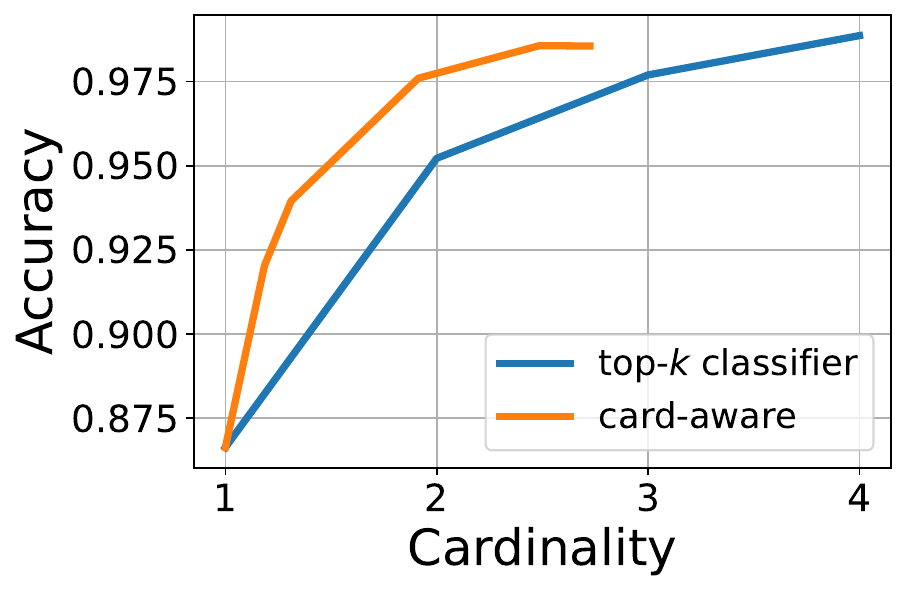} &
\includegraphics[scale=0.24]{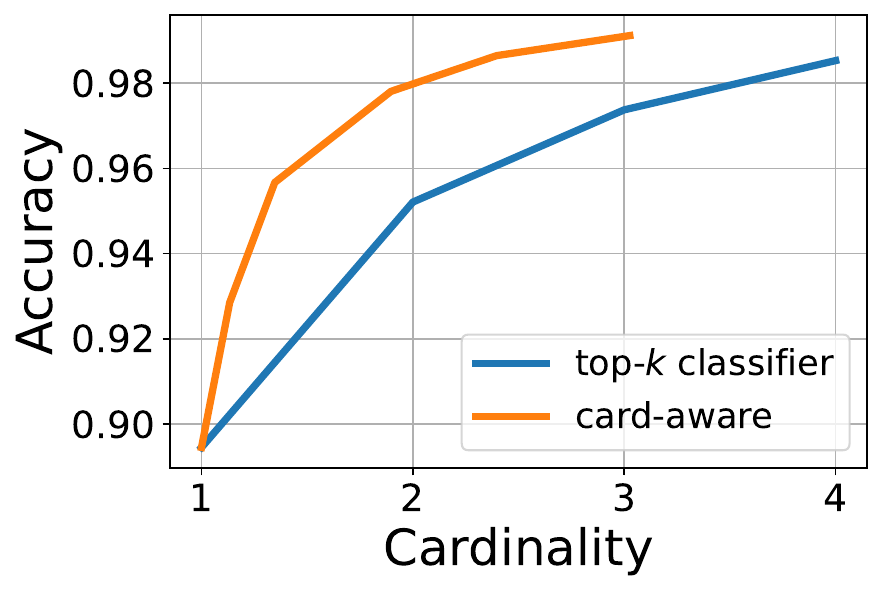}\\[-0.15cm]
\\[-0.15cm]
{\small CIFAR-100} & {\small ImageNet} 
& {\small CIFAR-10} & {\small SVHN} 
\end{tabular}
\vskip -0.1in
\caption{Accuracy versus cardinality for $\sK = \curl*{1, 2, 4, 8}$. Cardinality $\cost(k) = \log k$. }
\label{fig:topk}
\end{center}
\vskip -0.15in
\end{figure} 

\begin{figure}[t]
\vskip -.05in
\begin{center}
\begin{tabular}{@{\hspace{-.5cm}}c@{\hspace{0cm}}c@{\hspace{0cm}}c@{\hspace{0cm}}c@{\hspace{0cm}}}
\includegraphics[scale=0.24]{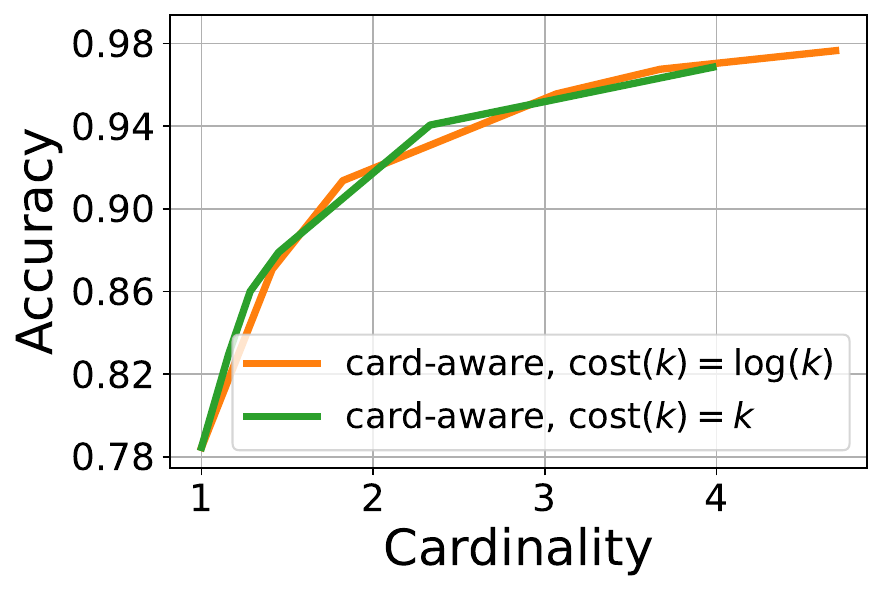} & 
\includegraphics[scale=0.24]{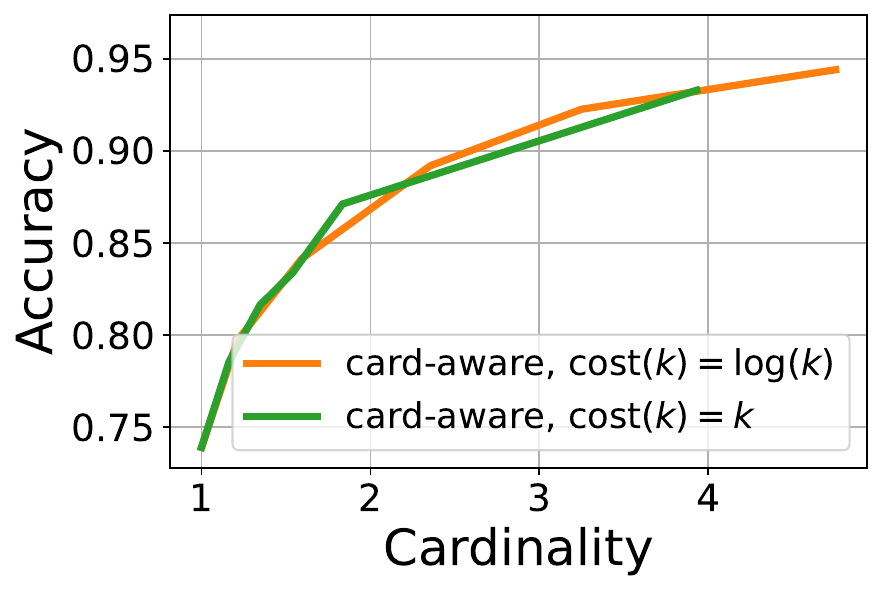} &
\includegraphics[scale=0.24]{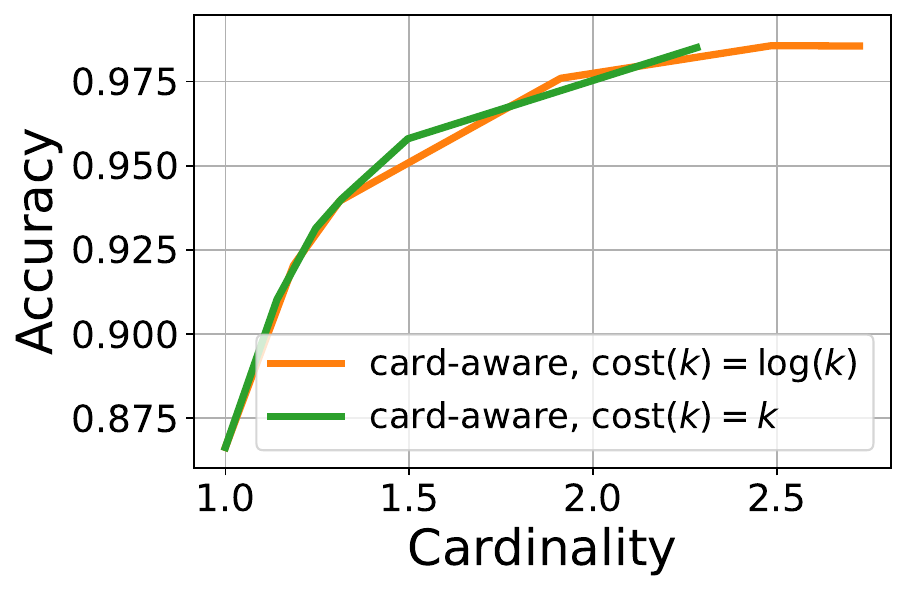} &
\includegraphics[scale=0.24]{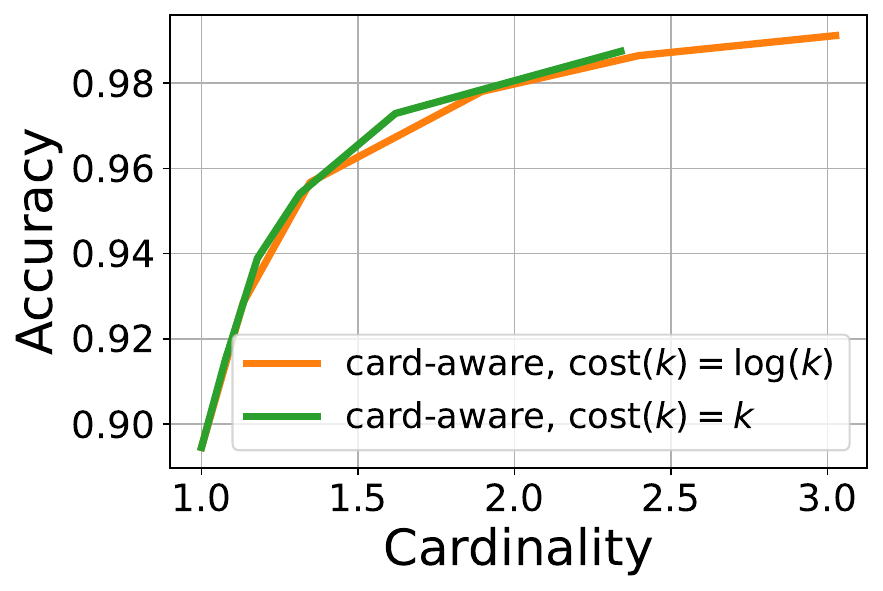}\\[-0.15cm]
\\[-0.15cm]
{\small CIFAR-100} & {\small ImageNet} 
& {\small CIFAR-10} & {\small SVHN} 
\end{tabular}
\vskip -0.1in
\caption{Accuracy versus cardinality for
  $\cost(k) = \log k$ and $\cost(k) = k$, with $\sK = \curl*{1, 2, 4, 8}$.}
\label{fig:topk_compare}
\end{center}
\vskip -0.3in
\end{figure} 

Figure~\ref{fig:topk_compare} presents the comparison of
$\cost(\abs*{\g_k(x)})\! =\! k$ and $\cost(\abs*{\g_k(x)})\! =\! \log k$ in
the same setting (for each dataset, the orange curve in
Figure~\ref{fig:topk_compare} coincides with the orange curve in
Figure~\ref{fig:topk}).
The comparison suggests that the choice between the
linear and logarithmic cardinality costs has negligible
impact on our algorithm's performance\ignore{ in this setting},
highlighting its robustness in this regard.
\ignore{These results demonstrate that different choices of $\cost(\abs*{\g_k(x)})$ basically lead to the same curve, which verifies the effectiveness and benefit of our algorithm.
}
We present additional experimental results with
different choices of set $\sK$ \ignore{varying the number of experts)} in
Figure~\ref{fig:topk_lambda} and Figure~\ref{fig:topk_Ck} in
Appendix~\ref{app:add}. Our
cardinality-aware algorithm consistently outperforms top-$k$
classifiers across all configurations.

\section{Conclusion}

We introduced a new cardinality-aware set prediction framework for
which we proposed two families of surrogate losses
with strong $\sH$-consistency guarantees:
cost-sensitive comp-sum and constrained losses. This leads to principled
and practical cardinality-aware algorithms for top-$k$ classification,
which we showed empirically to be very effective.
Additionally, we established a theoretical foundation for top-$k$
classification with fixed cardinality $k$ by proving that several
common surrogate loss functions, including comp-sum losses and
constrained losses in standard classification, admit $\sH$-consistency
bounds with respect to the top-$k$ loss. This provides a theoretical
justification for the use of these losses in top-$k$ classification
and opens new avenues for further research in this area.

\bibliography{topk}

\newpage
\appendix

\renewcommand{\contentsname}{Contents of Appendix}
\tableofcontents
\addtocontents{toc}{\protect\setcounter{tocdepth}{4}} 
\clearpage

\section{Proof of Lemma~\ref{lemma:regret-target}}
\label{app:regret-target}

\RegretTarget*
\begin{proof}
By definition, for any $h \in \sH$ and $x \in \sX$, the conditional
error of top-$k$ loss can be written as
\begin{equation*}
\sC_{ \ell_k }(h, x) =  \sum_{y\in \sY} p(x,y) 1_{y \notin \curl*{\hh_1(x), \ldots, \hh_k(x)}} = 1 - \sum_{i = 1}^k p(x, \hh_i(x)).
\end{equation*}
By definition of the labels $\pp_i(x)$, which are the most likely
top-$k$ labels, $\sC_{ \ell_k }(h, x)$ is minimized for $\hh_i(x) =
k_{\min}(x)$, $i \in [k]$. Since $\sH$ is regular, this choice is
realizable for some $h \in \sH$. Thus, we have
\begin{equation*}
  \sC^*_{\ell_k }(\sH, x)
  = \inf_{h \in \sH} \sC_{ \ell_k }(h, x)
  = 1 - \sum_{i = 1}^k p(x, \pp_i(x)).    
\end{equation*}
Furthermore, the calibration gap can be expressed as
\begin{align*}
\Delta\sC_{\ell_k, \sH}(h, x)  = \sC_{ \ell_k }(h, x) - \sC^*_{\ell_k}(\sH, x) = \sum_{i = 1}^k \paren*{p(x, \pp_i(x)) - p(x, \hh_i(x))},
\end{align*}
which completes the proof.
\end{proof}

\section{Proofs of \texorpdfstring{$\sH$}{H}-consistency bounds for comp-sum losses}
\label{app:bound-comp}

\BoundComp*
\begin{proof}
\textbf{Case I: $\wt \ell_{\rm{comp}} = \wt \ell_{\log}$.} For logistic loss $\wt \ell_{\log}$, the conditional regret can be written as 
\begin{align*}
\Delta\sC_{\wt \ell_{\log}, \sH}(h, x) 
& = \sum_{y = 1}^n p(x, y) \wt \ell_{\log}(h, x, y) - \inf_{h \in \sH} \sum_{y = 1}^n p(x, y) \wt \ell_{\log}(h, x, y)\\
& \geq \sum_{y = 1}^n p(x, y) \wt \ell_{\log}(h, x, y) - \inf_{\mu \in \Rset} \sum_{y = 1}^n p(x, y) \wt \ell_{\log}(h_{\mu, i}, x, y),
\end{align*}
where for any $i \in [k]$, $h_{\mu, i}(x, y) = \begin{cases}
h(x, y), & y \notin \curl*{\pp_i(x), \hh_i(x)}\\
\log\paren*{e^{h(x, \pp_i(x))} + \mu} & y = \hh_i(x)\\
\log\paren*{e^{h(x, \hh_i(x))} - \mu} & y = \pp_i(x).
\end{cases}$
Note that such a choice of $h_{\mu, i}$ leads to the following equality holds:
\begin{equation*}
\sum_{y \notin \curl*{\hh_i(x), \pp_i(x)}} p(x, y) \wt \ell_{\log}(h, x, y) = \sum_{y \notin \curl*{\hh_i(x), \pp_i(x)}}  p(x, y) \wt \ell_{\log}(h_{\mu, i}, x, y).
\end{equation*}
Therefore, for any $i \in [k]$, the conditional regret of logistic loss can be lower bounded as
\begin{align*}
\Delta\sC_{\wt \ell_{\log}, \sH}(h, x)  & \geq -p(x, \hh_i(x)) \log\paren*{\frac{e^{h(x, \hh_i(x))}}{\sum_{y \in \sY} e^{h(x, y)}}} - p(x, \pp_i(x)) \log\paren*{\frac{e^{h(x, \pp_i(x))}}{\sum_{y \in \sY} e^{h(x, y)}}}\\
& \qquad + \sup_{\mu \in \Rset} \paren*{ p(x, \hh_i(x)) \log\paren*{\frac{e^{h(x, \pp_i(x))} + \mu}{\sum_{y \in \sY} e^{h(x, y)}}} + p(x, \pp_i(x)) \log\paren*{\frac{e^{h(x, \hh_i(x))} - \mu}{\sum_{y \in \sY} e^{h(x, y)}}} }\\
& = \sup_{\mu \in \Rset} \paren*{ p(x, \hh_i(x))\log\paren*{\frac{e^{h(x, \pp_i(x))} + \mu}{e^{h(x, \hh_i(x))}}} + p(x, \pp_i(x)) \log\paren*{\frac{e^{h(x, \hh_i(x))} - \mu }{e^{h(x, \pp_i(x))}}} }.
\end{align*}
By the concavity of the function, differentiate with respect to $\mu$, we obtain that the supremum is achieved by $\mu^* = \frac{p(x, \hh_i(x)) e^{h(x, \hh_i(x))} - p(x, \pp_i(x)) e^{h(x, \pp_i(x))} }{p(x, \hh_i(x)) + p(x, \pp_i(x))}$. Plug in $\mu^*$, we obtain
\begin{align*}
& \Delta\sC_{\wt \ell_{\log}, \sH}(h, x)\\ 
& \geq p(x, \hh_i(x)) \log \paren*{\frac{p(x, \hh_i(x))}{p(x, \hh_i(x)) + p(x, \pp_i(x))} \frac{e^{h(x, \hh_i(x))} + e^{h(x, \pp_i(x))} }{e^{h(x, \hh_i(x))}}}\\
& \qquad + p(x, \pp_i(x)) \log \paren*{\frac{p(x, \pp_i(x))}{p(x, \hh_i(x)) + p(x, \pp_i(x))} \frac{e^{h(x, \hh_i(x))} + e^{h(x, \pp_i(x))} }{e^{h(x, \pp_i(x))}}}\\
& \geq p(x, \hh_i(x)) \log \paren*{\frac{2p(x, \hh_i(x))}{p(x, \hh_i(x)) + p(x, \pp_i(x))}} + p(x, \pp_i(x)) \log \paren*{\frac{2 p(x, \pp_i(x))}{p(x, \hh_i(x)) + p(x, \pp_i(x))}}
\tag{minimum is achieved when $h(x, \hh_i(x)) = h(x, \pp_i(x))$}.
\end{align*}
let $S_i = p(x, \pp_i(x)) + p(x, \hh_i(x))$ and $\Delta_i = p(x, \pp_i(x)) - p(x, \hh_i(x))$, we have
\begin{align*}
\Delta\sC_{\wt \ell_{\log}, \sH}(h, x) 
& \geq \frac{S_i - \Delta_i}{2}\log(\frac{S_i - \Delta_i}{S_i}) + \frac{S_i + \Delta_i}{2}\log(\frac{S_i + \Delta_i}{S_i})\\
& \geq \frac{1 - \Delta_i}{2}\log(1 - \Delta_i) + \frac{1 + \Delta_i}{2}\log(1 + \Delta_i)
\tag{minimum is achieved when $S_i = 1$}\\
&  = \psi \paren*{p(x, \pp_i(x)) - p(x, \hh_i(x))},
\end{align*}
where $\psi(t) = \frac{1 - t}{2}\log(1 - t) + \frac{1 + t}{2}\log(1+ t)$, $t \in [0,1]$.
Therefore, the conditional regret of the top-$k$ loss can be upper bounded as follows:
\begin{equation*}
\Delta \sC_{\ell_k, \sH}(h, x) = \sum_{i = 1}^k \paren*{p(x, \pp_i(x)) - p(x, \hh_i(x))} \leq k \psi^{-1} \paren*{\Delta\sC_{\wt \ell_{\log}, \sH}(h, x) }.
\end{equation*}
By the concavity of $\psi^{-1}$, taking expectations on both sides of the preceding equation, we obtain
\begin{equation*}
\sE_{\ell_k}(h) - \sE^*_{\ell_k}(\sH) + \sM_{\ell_k}(\sH) \leq k \psi^{-1} \paren*{ \sE_{\wt \ell_{\log}}(h) - \sE^*_{\wt \ell_{\log}}(\sH) + \sM_{\wt \ell_{\log}}(\sH) }.
\end{equation*}
The second part
follows from the fact that when $\sA_{\wt \ell_{\log}}(\sH) = 0$, the
minimizability gap $\sM_{\wt \ell_{\log}}(\sH)$ vanishes. 

\textbf{Case II: $\wt \ell_{\rm{comp}} = \wt \ell_{\exp}$.} For sum exponential loss $\wt \ell_{\exp}$, the conditional regret can be written as 
\begin{align*}
\Delta\sC_{\wt \ell_{\exp}, \sH}(h, x) 
& = \sum_{y = 1}^n p(x, y) \wt \ell_{\exp}(h, x, y) - \inf_{h \in \sH} \sum_{y = 1}^n p(x, y) \wt \ell_{\exp}(h, x, y)\\
& \geq \sum_{y = 1}^n p(x, y) \wt \ell_{\exp}(h, x, y) - \inf_{\mu \in \Rset} \sum_{y = 1}^n p(x, y) \wt \ell_{\exp}(h_{\mu, i}, x, y),
\end{align*}
where for any $i \in [k]$, $h_{\mu, i}(x, y) = \begin{cases}
h(x, y), & y \notin \curl*{\pp_i(x), \hh_i(x)}\\
\log\paren*{e^{h(x, \pp_i(x))} + \mu} & y = \hh_i(x)\\
\log\paren*{e^{h(x, \hh_i(x))} - \mu} & y = \pp_i(x).
\end{cases}$
Note that such a choice of $h_{\mu, i}$ leads to the following equality holds:
\begin{equation*}
\sum_{y \notin \curl*{\hh_i(x), \pp_i(x)}} p(x, y) \wt \ell_{\exp}(h, x, y) = \sum_{y \notin \curl*{\hh_i(x), \pp_i(x)}}  p(x, y) \wt \ell_{\exp}(h_{\mu, i}, x, y).
\end{equation*}
Therefore, for any $i \in [k]$, the conditional regret of sum exponential loss can be lower bounded as
\begin{align*}
\Delta\sC_{\wt \ell_{\exp}, \sH}(h, x)  & \geq \sum_{y' \in \sY} \exp \paren*{h(x, y')}\bracket*{\frac{p(x, \hh_i(x))} {\exp \paren*{h(x,  \hh_i(x))}} + \frac{p(x, \pp_i(x))} {\exp\paren*{h(x, \pp_i(x))}}}\\
& \qquad + \sup_{\mu \in \Rset} \paren*{- \sum_{y'\in \sY} \exp\paren*{h(x, y')}\bracket*{\frac{p(x, \hh_i(x))}{ \exp\paren*{h(x, \pp_i(x))} + \mu} + \frac{p(x, \pp_i(x))}{\exp\paren*{h(x, \hh_i(x))} - \mu}} }.
\end{align*}
By the concavity of the function, differentiate with respect to $\mu$, we obtain that the supremum is achieved by $\mu^* =\frac{\exp \bracket*{h(x,\hh_i(x))}\sqrt{p(x,\hh_i(x))} - \exp\bracket*{h(x, \pp_i(x))}\sqrt{p(x, \pp_i(x))}}{\sqrt{p(x, \hh_i(x))} + \sqrt{p(x,\pp_i(x))}}$. Plug in $\mu^*$, we obtain
\begin{align*}
& \Delta\sC_{\wt \ell_{\exp}, \sH}(h, x)\\ 
& \geq \sum_{y'\in \sY} \exp\paren*{h(x, y')} \bracket*{\frac{p(x, \hh_i(x))} {\exp\paren*{h(x, \hh_i(x))}} + \frac{p(x, \pp_i(x))} {\exp\paren*{h(x, \pp_i(x))}} - \frac{\paren*{\sqrt{p(x, \hh_i(x))} + \sqrt{p(x, \pp_i(x))}}^2}{\exp\paren*{h(x, \pp_i(x))} + \exp\paren*{h(x, \hh_i(x))}}}\\
&\geq \bracket*{1 + \frac{\exp\paren*{h(x, \pp_i(x))}}{\exp\paren*{h(x, \hh_i(x))}}}p(x, \hh_i(x)) + \bracket*{1 + \frac{\exp\paren*{h(x, \hh_i(x))}}{\exp\paren*{h(x, \pp_i(x))}}}p(x, \pp_i(x)) - \paren*{\sqrt{p(x, \hh_i(x))} + \sqrt{p(x, \pp_i(x))}}^2
\tag{$\sum_{y' \in \sY} \exp\paren*{h(x, y')}\geq \exp\paren*{h(x, \pp_i(x))} + \exp\paren*{h(x, \hh_i(x))}$}\\
& \geq 2 p(x,\hh_i(x)) + 2 p(x,\pp_i(x)) - \paren*{\sqrt{p(x, \hh_i(x))} + \sqrt{p(x, \pp_i(x))}}^2 \tag{minimum is attained when $\frac{\exp\paren*{h(x, \pp_i(x))}}{\exp\paren*{h(x,\hh_i(x))}} = 1$}.
\end{align*}
let $S_i = p(x, \pp_i(x)) + p(x, \hh_i(x))$ and $\Delta_i = p(x, \pp_i(x)) - p(x, \hh_i(x))$, we have
\begin{align*}
\Delta\sC_{\wt \ell_{\exp}, \sH}(h, x) 
& \geq 2 S_i - \paren*{\sqrt{\frac{S_i + \Delta_i}{2}} + \sqrt{\frac{S_i - \Delta_i}{2}}}^2\\
& \geq 2\bracket*{1 -\bracket*{\frac{\paren*{1 + \Delta_i}^{\frac1{2}} + \paren*{1 - \Delta_i}^{\frac1{2}}}{2}}^{2}} 
\tag{minimum is achieved when $S_i = 1$}\\
& = 1 - \sqrt{1 - (\Delta_i)^2} \\
&  = \psi \paren*{p(x, \pp_i(x)) - p(x, \hh_i(x))},
\end{align*}
where $\psi(t) = 1 - \sqrt{1 - t^2}$, $t \in [0,1]$.
Therefore, the conditional regret of the top-$k$ loss can be upper bounded as follows:
\begin{equation*}
\Delta \sC_{\ell_k, \sH}(h, x) = \sum_{i = 1}^k \paren*{p(x, \pp_i(x)) - p(x, \hh_i(x))} \leq k \psi^{-1} \paren*{\Delta\sC_{\wt \ell_{\exp}, \sH}(h, x) }.
\end{equation*}
By the concavity of $\psi^{-1}$, taking expectations on both sides of the preceding equation, we obtain
\begin{equation*}
\sE_{\ell_k}(h) - \sE^*_{\ell_k}(\sH) + \sM_{\ell_k}(\sH) \leq k \psi^{-1} \paren*{ \sE_{\wt \ell_{\exp}}(h) - \sE^*_{\wt \ell_{\exp}}(\sH) + \sM_{\wt \ell_{\exp}}(\sH) }.
\end{equation*}
The second part
follows from the fact that when $\sA_{\wt \ell_{\exp}}(\sH) =
0$, the minimizability gap $\sM_{\wt \ell_{\exp}}(\sH)$
vanishes.

\textbf{Case III: $\wt \ell_{\rm{comp}} = \wt \ell_{\rm{mae}}$.} For mean absolute error loss $\wt \ell_{\rm{mae}}$, the conditional regret can be written as 
\begin{align*}
\Delta\sC_{\wt \ell_{\rm{mae}}, \sH}(h, x) 
& = \sum_{y = 1}^n p(x, y) \wt \ell_{\rm{mae}}(h, x, y) - \inf_{h \in \sH} \sum_{y = 1}^n p(x, y) \wt \ell_{\rm{mae}}(h, x, y)\\
& \geq \sum_{y = 1}^n p(x, y) \wt \ell_{\rm{mae}}(h, x, y) - \inf_{\mu \in \Rset} \sum_{y = 1}^n p(x, y) \wt \ell_{\rm{mae}}(h_{\mu, i}, x, y),
\end{align*}
where for any $i \in [k]$, $h_{\mu, i}(x, y) = \begin{cases}
h(x, y), & y \notin \curl*{\pp_i(x), \hh_i(x)}\\
\log\paren*{e^{h(x, \pp_i(x))} + \mu} & y = \hh_i(x)\\
\log\paren*{e^{h(x, \hh_i(x))} - \mu} & y = \pp_i(x).
\end{cases}$
Note that such a choice of $h_{\mu, i}$ leads to the following equality holds:
\begin{equation*}
\sum_{y \notin \curl*{\hh_i(x), \pp_i(x)}} p(x, y) \wt \ell_{\rm{mae}}(h, x, y) = \sum_{y \notin \curl*{\hh_i(x), \pp_i(x)}}  p(x, y) \wt \ell_{\rm{mae}}(h_{\mu, i}, x, y).
\end{equation*}
Therefore, for any $i \in [k]$, the conditional regret of mean absolute error loss can be lower bounded as
\begin{align*}
& \Delta\sC_{\wt \ell_{\rm{mae}}, \sH}(h, x)\\  & \geq p(x,\hh_i(x)) \paren*{1-\frac{\exp\paren*{h(x,\hh_i(x))}}{\sum_{y'\in \sY}\exp\paren*{h(x,y')}}} +p(x,\pp_i(x)) \paren*{1-\frac{\exp\paren*{h(x,\pp_i(x))}}{\sum_{y'\in \sY}\exp\paren*{h(x,y')}}}\\
& \quad + \sup_{\mu \in \Rset} \paren*{-p(x,\pp_i(x)) \paren*{1-\frac{\exp\paren*{h(x,\hh_i(x))}-\mu}{\sum_{y'\in \sY}\exp\paren*{h(x,y')}}} -p(x,\hh_i(x)) \paren*{1-\frac{\exp\paren*{h(x,\pp_i(x))}+\mu}{\sum_{y'\in \sY}\exp\paren*{h(x,y')}}}}.
\end{align*}
By the concavity of the function, differentiate with respect to $\mu$, we obtain that the supremum is achieved by $\mu^* = -\exp\bracket*{h(x, \pp_i(x)}$. Plug in $\mu^*$, we obtain
\begin{align*}
& \Delta\sC_{\wt \ell_{\rm{mae}}, \sH}(h, x)\\ 
& \geq p(x,\pp_i(x))\frac{\exp\paren*{h(x,\hh_i(x))}}{\sum_{y'\in \sY}\exp\paren*{h(x,y')}}-p(x,\hh_i(x))\frac{\exp\paren*{h(x,\hh_i(x))}}{\sum_{y'\in \sY}\exp\paren*{h(x,y')}}\\
&  \geq \frac{1}{n}
\paren*{p(x,\pp_i(x)) - p(x,\hh_i(x))}
\tag{$\frac{\exp\paren*{h(x,\hh_i(x))}}{\sum_{y'\in \sY}\exp\paren*{h(x,y')}}\geq \frac1{n}$}
\end{align*}
Therefore, the conditional regret of the top-$k$ loss can be upper bounded as follows:
\begin{equation*}
\Delta \sC_{\ell_k, \sH}(h, x) = \sum_{i = 1}^k \paren*{p(x, \pp_i(x)) - p(x, \hh_i(x))} \leq k n \paren*{\Delta\sC_{\wt \ell_{\rm{mae}}, \sH}(h, x) }.
\end{equation*}
Take expectations on both sides of the preceding equation, we obtain
\begin{equation*}
\sE_{\ell_k}(h) - \sE^*_{\ell_k}(\sH) + \sM_{\ell_k}(\sH) \leq k n \paren*{ \sE_{\wt \ell_{\rm{mae}}}(h) - \sE^*_{\wt \ell_{\rm{mae}}}(\sH) + \sM_{\wt \ell_{\rm{mae}}}(\sH) }.
\end{equation*}
The second part
follows from the fact that when $\sA_{\wt \ell_{\rm{mae}}}(\sH) = 0$, the
minimizability gap $\sM_{\wt \ell_{\rm{mae}}}(\sH)$ vanishes.

\textbf{Case IV: $\wt \ell_{\rm{comp}} = \wt \ell_{\rm{gce}}$.} For generalized cross-entropy loss $\wt \ell_{\rm{gce}}$, the conditional regret can be written as 
\begin{align*}
& \Delta\sC_{\wt \ell_{\rm{gce}}, \sH}(h, x)\\ 
& = \sum_{y = 1}^n p(x, y) \wt \ell_{\rm{gce}}(h, x, y) - \inf_{h \in \sH} \sum_{y = 1}^n p(x, y) \wt \ell_{\rm{gce}}(h, x, y)\\
& \geq \sum_{y = 1}^n p(x, y) \wt \ell_{\rm{gce}}(h, x, y) - \inf_{\mu \in \Rset} \sum_{y = 1}^n p(x, y) \wt \ell_{\rm{gce}}(h_{\mu, i}, x, y),
\end{align*}
where for any $i \in [k]$, $h_{\mu, i}(x, y) = \begin{cases}
h(x, y), & y \notin \curl*{\pp_i(x), \hh_i(x)}\\
\log\paren*{e^{h(x, \pp_i(x))} + \mu} & y = \hh_i(x)\\
\log\paren*{e^{h(x, \hh_i(x))} - \mu} & y = \pp_i(x).
\end{cases}$
Note that such a choice of $h_{\mu, i}$ leads to the following equality holds:
\begin{equation*}
\sum_{y \notin \curl*{\hh_i(x), \pp_i(x)}} p(x, y) \wt \ell_{\rm{gce}}(h, x, y) = \sum_{y \notin \curl*{\hh_i(x), \pp_i(x)}}  p(x, y) \wt \ell_{\rm{gce}}(h_{\mu, i}, x, y).
\end{equation*}
Therefore, for any $i \in [k]$, the conditional regret of generalized cross-entropy loss can be lower bounded as
\begin{align*}
& \q \Delta\sC_{\wt \ell_{\rm{gce}}, \sH}(h, x)\\
& \geq  p(x, \hh_i(x)) \paren*{1 - \bracket*{\frac{\exp\paren*{h(x, \hh_i(x))}}{\sum_{y' \in \sY}\exp\paren*{h(x, y')}}}^{\q}} + p(x, \pp_i(x)) \paren*{1 - \bracket*{\frac{\exp\paren*{h(x, \pp_i(x))}}{\sum_{y'\in \sY}\exp\paren*{h(x, y')}}}^{\q}}\\
& + \sup_{\mu \in \Rset} \paren*{ -p(x, \hh_i(x)) \paren*{1 - \bracket*{\frac{\exp\paren*{h(x, \pp_i(x))} + \mu}{\sum_{y' \in \sY}\exp\paren*{h(x, y')}}}^{\q}} - p(x, \pp_i(x)) \paren*{1 - \bracket*{\frac{\exp\paren*{h(x, \hh_i(x))} - \mu}{\sum_{y' \in \sY}\exp\paren*{h(x, y')}}}^{\q}} }.
\end{align*}
By the concavity of the function, differentiate with respect to $\mu$, we obtain that the supremum is achieved by $\mu^* = \frac{\exp\bracket*{h(x,\hh_i(x))}p(x, \pp_i(x))^{\frac1{\q - 1}} - \exp\bracket*{h(x, \pp_i(x))}p(x,\hh_i(x))^{\frac1{\q - 1}}}{p(x,\hh_i(x))^{\frac1{\q - 1}} + p(x,\pp_i(x))^{\frac1{\q - 1}}}$. Plug in $\mu^*$, we obtain
\begin{align*}
& \q \Delta\sC_{\wt \ell_{\rm{gce}}, \sH}(h, x)\\ 
& \geq p(x, \hh_i(x))\bracket*{\frac{\bracket*{\exp\paren*{h(x, \hh_i(x))} + \exp\paren*{h(x, \pp_i(x))}}p(x, \pp_i(x))^{\frac1{\q - 1}}}{\sum_{y'\in \sY}\exp\paren*{h(x, y')}\bracket*{p(x, \hh_i(x))^{\frac1{\q - 1}} + p(x, \pp_i(x))^{\frac1{\q - 1}}}}}^{\q } - p(x, \hh_i(x))\bracket*{\frac{\exp\paren*{h(x, \hh_i(x))}}{\sum_{y'\in \sY}\exp\paren*{h(x, y')}}}^{\q }\\
&\quad +p(x, \pp_i(x))\bracket*{\frac{\bracket*{\exp\paren*{h(x, \hh_i(x))} + \exp\paren*{h(x, \pp_i(x))}}p(x, \hh_i(x))^{\frac1{\q - 1}}}{\sum_{y'\in \sY}\exp\paren*{h(x, y')}\bracket*{p(x, \hh_i(x))^{\frac1{\q - 1}} + p(x, \pp_i(x))^{\frac1{\q - 1}}}}}^{\q } - p(x, \pp_i(x))\bracket*{\frac{\exp\paren*{h(x, \pp_i(x))}}{\sum_{y'\in \sY}\exp\paren*{h(x, y')}}}^{\q }\\
&\geq \frac{1}{n^{\q}}\paren*{p(x,\hh_i(x))\bracket*{\frac{2p(x,\pp_i(x))^{\frac1{\q - 1}}}{p(x,\hh_i(x))^{\frac1{\q - 1}}+p(x,\pp_i(x))^{\frac1{\q - 1}}}}^{\q}-p(x,\hh_i(x))}\\
& + \frac{1}{n^{\q}}\paren*{p(x,\pp_i(x))\bracket*{\frac{2p(x,\hh_i(x))^{\frac1{\q - 1}}}{p(x,\hh_i(x))^{\frac1{\q - 1}}+p(x,\pp_i(x))^{\frac1{\q - 1}}}}^{\q}-p(x,\pp_i(x))}
\tag{$\paren*{\frac{\exp\paren*{h(x,\pp_i(x))}}{\sum_{y'\in \sY}\exp\paren*{h(x,y')}}}^{\q }\geq \frac1{n^{\q }}$ and minimum is attained when $\frac{\exp\paren*{h(x,\pp_i(x))}}{\exp\paren*{h(x,\hh_i(x))}}=1$}\\
\end{align*}
let $S_i = p(x, \pp_i(x)) + p(x, \hh_i(x))$ and $\Delta_i = p(x, \pp_i(x)) - p(x, \hh_i(x))$, we have
\begin{align*}
\Delta\sC_{\wt \ell_{\rm{gce}}, \sH}(h, x) 
& \geq \frac{1}{\q n^{\q}}\paren*{\bracket*{\frac{\paren*{S_i + \Delta_i}^{\frac1{1 - \q }}+\paren*{S_i - \Delta_i}^{\frac1{1 - \q}}}{2}}^{1 - \q} - S_i}\\
& \geq \frac{1}{\q n^{\q}}\paren*{\bracket*{\frac{\paren*{1 + \Delta_i}^{\frac1{1 - \q }}+\paren*{1 - \Delta_i}^{\frac1{1 - \q}}}{2}}^{1 - \q} - 1}\\
\tag{minimum is achieved when $S_i = 1$}\\
&  = \psi \paren*{p(x, \pp_i(x)) - p(x, \hh_i(x))},
\end{align*}
where $\psi(t) = \frac{1}{\q n^{\q}}
\bracket*{\bracket*{\frac{\paren*{1 + t}^{\frac1{1 - \q }} +
      \paren*{1 - t}^{\frac1{1 - \q }}}{2}}^{1 - \q } -1}$, $t \in [0,1]$.
Therefore, the conditional regret of the top-$k$ loss can be upper bounded as follows:
\begin{equation*}
\Delta \sC_{\ell_k, \sH}(h, x) = \sum_{i = 1}^k \paren*{p(x, \pp_i(x)) - p(x, \hh_i(x))} \leq k \psi^{-1} \paren*{\Delta\sC_{\wt \ell_{\rm{gce}}, \sH}(h, x) }.
\end{equation*}
By the concavity of $\psi^{-1}$, taking expectations on both sides of the preceding equation, we obtain
\begin{equation*}
\sE_{\ell_k}(h) - \sE^*_{\ell_k}(\sH) + \sM_{\ell_k}(\sH) \leq k \psi^{-1} \paren*{ \sE_{\wt \ell_{\rm{gce}}}(h) - \sE^*_{\wt \ell_{\rm{gce}}}(\sH) + \sM_{\wt \ell_{\rm{gce}}}(\sH) }.
\end{equation*}
The second
part follows from the fact that when $\sA_{\wt \ell_{\rm{gce}}}(\sH) = 0$,
the minimizability gap $\sM_{\wt \ell_{\rm{gce}}}(\sH)$ vanishes.
\end{proof}

\section{Minimizability gaps and realizability}
\label{app:min_re}

The key quantities in our $\sH$-consistency bounds are the
minimizability gaps, which can be upper bounded by the approximation
error, or more refined terms, depending on the magnitude of the
parameter space, as discussed by \citet{mao2023cross}. As pointed out
by these authors, these quantities, along with the functional form,
can help compare different comp-sum loss functions.

Here, we further discuss the important role of minimizability gaps
under the realizability assumption, and the connection with some
negative results of \citet{yang2020consistency}.

\begin{definition}[\textbf{top-$k$-$\sH$-realizability}]
\label{def:rel}
A distribution $\sD$ over $\sX\times\sY$ is
\emph{top-$k$-$\sH$-realizable}, if there exists a hypothesis $h\in
\sH$ such that $\mathbb{P}_{(x,y)\sim \sD}\paren*{h(x, y) > h(x,
  \hh_{k + 1}(x))} = 1$.
\end{definition}
This extends the $\sH$-realizability definition from standard
(top-$1$) classification \citep{long2013consistency} to top-$k$
classification for any $k \geq 1$.
\begin{definition}
  We say that a hypothesis set \emph{$\sH$ is closed under scaling},
  if it is a cone, that is for all $h \in \sH$ and $\beta \in
  \Rset_{+}$, $\beta h \in \sH$.
\end{definition}
\begin{definition}
We say that a surrogate loss $\wt \ell$ is \emph{realizable
$\sH$-consistent with respect to $\ell_k$}, if for all $k \in [1, n]$,
and for any sequence of hypotheses $\curl*{h_n}_{n \in \Nset} \subset
\sH$ and top-$k$-$\sH$-realizable distribution, $\lim_{n \to \plus
  \infty} \sE_{\wt \ell}\paren*{h_n} - \sE^*_{\wt \ell}\paren*{\sH} = 0$
implies $\lim_{n \to \plus \infty} \sE_{\ell_{k}}\paren*{h_n} -
\sE^*_{\ell_{k}}\paren*{\sH} = 0$.
\end{definition}
When $\sH$ is closed under scaling, for $k = 1$ and all comp-sum loss
functions $\ell = \wt \ell_{\log}$, $\wt \ell_{\exp}$,
$\wt \ell_{\rm{gce}}$ and $\wt \ell_{\rm{mae}}$, it can be shown that
$\sE^*_{\wt \ell}(\sH) = \sM_{\wt \ell}(\sH) = 0$ for any $\sH$-realizable
distribution. For example, for $\ell = \wt \ell_{\log}$, by using the
Lebesgue dominated convergence theorem, we have
\begin{align*}
  \sM_{\wt \ell_{\log}}(\sH) &\leq \sE^*_{\wt \ell_{\log}}(\sH)
  \leq \lim_{\beta \to \plus \infty} \sE_{\wt \ell_{\log}}(\beta h^*)
= \lim_{\beta \to \plus \infty} \log \bracket[\bigg]{1 + \sum_{y' \neq y} e^{\beta \paren*{ h^*(x, y') - h^*(x, y)}}} = 0,
\end{align*}
where $h^*$ satisfies $\mathbb{P}_{(x,y)\sim \sD}\paren*{h^*(x, y) > h^*(x, \hh_{2}(x))} = 1$
Therefore, Theorem~\ref{thm:bound-comp} implies that all these loss
functions are realizable $\sH$-consistent with respect to $\ell_{0-1}$
($\ell_k$ for $k = 1$) when $\sH$ is closed under scaling.
\begin{restatable}{theorem}{RealizabilityP}
\label{thm:realizability-p}
Assume that $\sH$ is closed under scaling. Then, $\wt \ell_{\log}$,
$\wt \ell_{\exp}$, $\wt \ell_{\rm{gce}}$ and $\wt \ell_{\rm{mae}}$ are
realizable $\sH$-consistent with respect to $\ell_{0-1}$.
\end{restatable}
The formal proof is presented in
Appendix~\ref{app:realizability}. However, for $ k > 1 $, since in the
realizability assumption, $h(x, y)$ is only larger than $h(x, \hh_{k +
  1}(x))$ and can be smaller than $h(x, \hh_{1}(x))$, there may exist an
$\sH$-realizable distribution $\sD$ such that $\sM_{\wt \ell_{\log}}(\sH)
> 0$. This explains the inconsistency of the logistic loss
on top-$k$ separable data with linear predictors, when $k = 2$ and $n
> 2$, as shown in \citep{yang2020consistency}. More generally, the exact same example in \citep[Proposition~5.1]{yang2020consistency} can be used to show that all the comp-sum losses, $\wt \ell_{\log}$,
$\wt \ell_{\exp}$, $\wt \ell_{\rm{gce}}$ and $\wt \ell_{\rm{mae}}$ are not realizable $\sH$-consistent with
respect to $\ell_k$. Nevertheless, as previously shown, when the
hypothesis set $\sH$ adopted is sufficiently rich such that
$\sM_{\wt \ell}(\sH) = 0$ or even $\sA_{\wt \ell}(\sH) = 0$, they are
guaranteed to be $\sH$-consistent. This is typically the case in
practice when using deep neural networks.

\section{Proofs of realizable \texorpdfstring{$\sH$}{H}-consistency for comp-sum losses}
\label{app:realizability}

\RealizabilityP*
\begin{proof}
Since the distribution is realizable, there exists a hypothesis $h \in \sH$ such that \[\mathbb{P}_{(x,y)\sim \sD}\paren*{h^*(x, y) > h^*(x, \hh_{2}(x))} = 1.\]
Therefore, for the logistic loss, by using the
Lebesgue dominated convergence theorem,
\begin{align*}
  \sM_{\wt \ell_{\log}}(\sH) &\leq \sE^*_{\wt \ell_{\log}}(\sH)
  \leq \lim_{\beta \to \plus \infty} \sE_{\wt \ell_{\log}}(\beta h) = \lim_{\beta \to \plus \infty} \log \bracket[\bigg]{1 + \sum_{y' \neq y} e^{\beta \paren*{ h^*(x, y') - h^*(x, y)}}} = 0.
\end{align*}
For the sum exponential loss, by using the
Lebesgue dominated convergence theorem,
\begin{align*}
  \sM_{\wt \ell_{\exp}}(\sH) &\leq \sE^*_{\wt \ell_{\exp}}(\sH)
  \leq \lim_{\beta \to \plus \infty} \sE_{\wt \ell_{\exp}}(\beta h) = \lim_{\beta \to \plus \infty} \sum_{y' \neq y} e^{ \beta\paren*{h^*(x, y') - h^*(x, y) }} = 0.
\end{align*}
For the generalized cross entropy loss, by using the
Lebesgue dominated convergence theorem,
\begin{align*}
  \sM_{\wt \ell_{\rm{gce}}}(\sH) &\leq \sE^*_{\wt \ell_{\rm{gce}}}(\sH)
  \leq \lim_{\beta \to \plus \infty} \sE_{\wt \ell_{\rm{gce}}}(\beta h) = \lim_{\beta \to \plus \infty} \frac{1}{\q}\bracket*{1 - \bracket*{\sum_{y' \in \sY}e^{ \beta( h^*(x, y')
      - h^*(x, y)) }}^{-\q}} = 0.
\end{align*}
For the mean absolute error loss, by using the
Lebesgue dominated convergence theorem,
\begin{align*}
  \sM_{\wt \ell_{\rm{mae}}}(\sH) &\leq \sE^*_{\wt \ell_{\rm{mae}}}(\sH)
  \leq \lim_{\beta \to \plus \infty} \sE_{\wt \ell_{\rm{mae}}}(\beta h) = \lim_{\beta \to \plus \infty} 1
- \bracket*{\sum_{y' \in \sY}e^{\beta \paren*{h^*(x, y') - h^*(x, y)} }}^{-1} = 0.
\end{align*}
Therefore, by Theorem~\ref{thm:bound-comp}, the proof is completed.
\end{proof}

\section{\texorpdfstring{$\sH$}{H}-Consistency bounds for constrained losses}
\label{app:cstnd}

Constrained losses are defined as a summation of a function
$\Phi$ applied to the scores, subject to a constraint, as shown in
\citep{lee2004multicategory,awasthi2022multi}. For any $h \in \sH$ and
$ (x, y) \in \sX \times \sY$, they are expressed as
\begin{align*}
\wt \ell_{\rm{cstnd}}(h, x, y)
= \sum_{y'\neq y} \Phi\paren*{-h(x, y')},
\end{align*}
with the constraint $\sum_{y\in \sY} h(x,y) = 0$, where $\Phi \colon
\Rset \to \Rset_{+} $ is non-increasing. When $\Phi$ is chosen as the
function $t \mapsto e^{-t}$, $t \mapsto \max \curl*{0, 1 - t}^2$, $t
\mapsto \max \curl*{0, 1 - t}$ and $t \mapsto \min\curl*{\max\curl*{0,
    1 - t/\rho}, 1}$, $\rho > 0$, $\wt \ell_{\rm{cstnd}}(h, x, y)$ are
referred to as the constrained exponential loss
$\wt \ell^{\rm{cstnd}}_{\exp}(h, x, y) = \sum_{y'\neq y} e^{h(x,
  y')}$, the constrained squared hinge loss
$\wt \ell_{\rm{sq-hinge}}(h, x, y) = \sum_{y'\neq y} \max \curl*{0, 1 + h(x,
  y')}^2$, the constrained hinge loss
$\wt \ell_{\rm{hinge}}(h, x, y) = \sum_{y'\neq y} \max \curl*{0, 1 +
  h(x, y')}$, and the constrained $\rho$-margin loss
$\wt \ell_{\rho}(h, x, y)
= \sum_{y'\neq y} \min\curl*{\max \curl*{0, 1 + h(x, y') / \rho},
  1}$, respectively \citep{awasthi2022multi}. We
now study these loss functions and show that they benefit
from $\sH$-consistency bounds with respect to the top-$k$ loss. 


\begin{restatable}{theorem}{BoundCstnd}
\label{thm:bound-cstnd}
Assume that $\sH$ is symmetric and complete. Then, for any $1 \leq k \leq n$, the following $\sH$-consistency bound holds for the constrained loss:
\begin{align*}
\sE_{\ell_k}(h) - \sE^*_{\ell_k}(\sH) + \sM_{\ell_k}(\sH) \leq k \gamma \paren*{ \sE_{\wt \ell_{\rm{cstnd}}}(h)
    - \sE^*_{\wt \ell_{\rm{cstnd}}}(\sH)
    + \sM_{\wt \ell_{\rm{cstnd}}}(\sH) }.
\end{align*}
In the special case where $\sA_{\wt \ell_{\rm{cstnd}}}(\sH) = 0$,
for any $1 \leq k \leq n$, the following bound holds:
\begin{align*}
  \sE_{\ell_k}(h) - \sE^*_{\ell_k}(\sH) \leq k \gamma
  \paren*{ \sE_{\wt \ell_{\rm{cstnd}}}(h)
    - \sE^*_{\wt \ell_{\rm{cstnd}}}(\sH)},
\end{align*}
where $\gamma(t) = 2\sqrt{t}$ when
$\wt \ell_{\rm{cstnd}}$ is either $\wt \ell^{\rm{cstnd}}_{\exp}$ or
$\wt \ell_{\rm{sq-hinge}}$; $\gamma(t) = t$ when
$\wt \ell_{\rm{cstnd}}$ is either $\wt \ell_{\rm{hinge}}$ or $\wt \ell_{\rho}$.
\end{restatable}
The proof is included in Appendix~\ref{app:bound-cstnd}. The
second part follows from the fact that when the hypothesis set $\sH$ is
sufficiently rich such that
$\sA_{\wt \ell_{\rm{cstnd}}}(\sH) = 0$, we have
$\sM_{\wt \ell_{\rm{cstnd}}}(\sH) = 0$.  Therefore, the constrained loss is $\sH$-consistent and Bayes-consistent with respect
to $\ell_k$. If the surrogate estimation error
$\sE_{\wt \ell_{\rm{cstnd}}}(h) -
\sE^*_{\wt \ell_{\rm{cstnd}}}(\sH)$ is $\e$, then, the target
estimation error satisfies $\sE_{\ell_k}(h) - \sE^*_{\ell_k}(\sH) \leq
k \gamma(\e)$. Note that the constrained exponential loss and the constrained squared hinge loss both admit a square root $\sH$-consistency bound while the bounds for the constrained hinge loss and $\rho$-margin loss are both linear.

\section{Proofs of \texorpdfstring{$\sH$}{H}-consistency bounds for constrained losses}
\label{app:bound-cstnd}

The conditional error for the constrained loss can be expressed as follows:
\begin{align*}
\sC_{\wt \ell_{\rm{cstnd}}}(h, x) = \sum_{y = 1}^n p(x, y)  \wt \ell_{\rm{cstnd}}(h, x, y) = \sum_{y = 1}^n p(x, y) \sum_{y'\neq y} \Phi\paren*{-h(x, y')} = \sum_{y\in \sY} \paren*{1-p(x, y)}\Phi\paren*{-h(x, y)}.
\end{align*}

\BoundCstnd*
\begin{proof}
\textbf{Case I: $\wt \ell_{\rm{cstnd}} = \wt \ell^{\rm{cstnd}}_{\exp}$.} For the constrained exponential loss $\wt \ell^{\rm{cstnd}}_{\exp}$, the conditional regret can be written as 
\begin{align*}
\Delta\sC_{\wt \ell^{\rm{cstnd}}_{\exp}, \sH}(h, x) 
& = \sum_{y = 1}^n p(x, y) \wt \ell^{\rm{cstnd}}_{\exp}(h, x, y) - \inf_{h \in \sH} \sum_{y = 1}^n p(x, y) \wt \ell^{\rm{cstnd}}_{\exp}(h, x, y)\\
& \geq \sum_{y = 1}^n p(x, y) \wt \ell^{\rm{cstnd}}_{\exp}(h, x, y) - \inf_{\mu \in \Rset} \sum_{y = 1}^n p(x, y) \wt \ell^{\rm{cstnd}}_{\exp}(h_{\mu, i}, x, y),
\end{align*}
where for any $i \in [k]$, $h_{\mu, i}(x, y) = \begin{cases}
h(x, y), & y \notin \curl*{\pp_i(x), \hh_i(x)}\\
h(x, \pp_i(x)) + \mu & y = \hh_i(x)\\
h(x, \hh_i(x)) - \mu & y = \pp_i(x).
\end{cases}$
Note that such a choice of $h_{\mu, i}$ leads to the following equality holds:
\begin{equation*}
\sum_{y \notin \curl*{\hh_i(x), \pp_i(x)}} p(x, y) \wt \ell^{\rm{cstnd}}_{\exp}(h, x, y) = \sum_{y \notin \curl*{\hh_i(x), \pp_i(x)}}  p(x, y) \wt \ell^{\rm{cstnd}}_{\exp}(h_{\mu, i}, x, y).
\end{equation*}
Let $q(x, \pp_i(x)) = 1 - p(x, \pp_i(x))$ and $q(x, \hh_i(x)) = 1 - p(x, \hh_i(x))$.
Therefore, for any $i \in [k]$, the conditional regret of constrained exponential loss can be lower bounded as
\begin{align*}
& \Delta\sC_{\wt \ell^{\rm{cstnd}}_{\exp}, \sH}(h, x)\\ 
& \geq  \inf_{h \in \sH}\sup_{\mu\in \Rset} \curl*{q(x, \pp_i(x))\paren*{e^{h(x, \pp_i(x))}-e^{h(x,\hh_i(x))-\mu}}+q(x,\hh_i(x))\paren*{e^{h(x,\hh_i(x))}-e^{h(x, \pp_i(x))+\mu}}}\\
& = \paren*{\sqrt{q(x, \pp_i(x))}-\sqrt{q(x,\hh_i(x))}}^2 \tag{differentiating with respect to $\mu$, $h$ to optimize}\\
&   =   \paren*{\frac{q(x,\hh_i(x)) - q(x, \pp_i(x))}{\sqrt{q(x, \pp_i(x))} + \sqrt{q(x, \hh_i(x))}}}^2\\
& \geq \frac1{4} \paren*{q(x,\hh_i(x)) - q(x, \pp_i(x))}^2 \tag{$0\leq q(x, y)\leq 1$}\\
& = \frac1{4} \paren*{p(x,\pp_i(x)) - p(x, \hh_i(x))}^2.
\end{align*}
Therefore, by Lemma~\ref{lemma:regret-target}, the conditional regret of the top-$k$ loss can be upper bounded as follows:
\begin{equation*}
\Delta \sC_{\ell_k, \sH}(h, x) = \sum_{i = 1}^k \paren*{p(x, \pp_i(x)) - p(x, \hh_i(x))} \leq 2 k \paren*{\Delta\sC_{\wt \ell^{\rm{cstnd}}_{\exp}, \sH}(h, x) }^{\frac12}.
\end{equation*}
By the concavity, taking expectations on both sides of the preceding equation, we obtain
\begin{equation*}
\sE_{\ell_k}(h) - \sE^*_{\ell_k}(\sH) + \sM_{\ell_k}(\sH) \leq 2 k\paren*{ \sE_{\wt \ell^{\rm{cstnd}}_{\exp}}(h) - \sE^*_{\wt \ell^{\rm{cstnd}}_{\exp}}(\sH) + \sM_{\wt \ell^{\rm{cstnd}}_{\exp}}(\sH) }^{\frac12}.
\end{equation*}
The second part follows from the fact that when
$\sA_{\wt \ell^{\rm{cstnd}}_{\exp}}(\sH) = 0$, we have
$\sM_{\wt \ell^{\rm{cstnd}}_{\exp}}(\sH) = 0$.

\textbf{Case II: $\wt \ell_{\rm{cstnd}} = \wt \ell_{\rm{sq-hinge}}$.} For the constrained squared hinge loss $\wt \ell_{\rm{sq-hinge}}$, the conditional regret can be written as 
\begin{align*}
\Delta\sC_{\wt \ell_{\rm{sq-hinge}}, \sH}(h, x) 
& = \sum_{y = 1}^n p(x, y) \wt \ell_{\rm{sq-hinge}}(h, x, y) - \inf_{h \in \sH} \sum_{y = 1}^n p(x, y) \wt \ell_{\rm{sq-hinge}}(h, x, y)\\
& \geq \sum_{y = 1}^n p(x, y) \wt \ell_{\rm{sq-hinge}}(h, x, y) - \inf_{\mu \in \Rset} \sum_{y = 1}^n p(x, y) \wt \ell_{\rm{sq-hinge}}(h_{\mu, i}, x, y),
\end{align*}
where for any $i \in [k]$, $h_{\mu, i}(x, y) = \begin{cases}
h(x, y), & y \notin \curl*{\pp_i(x), \hh_i(x)}\\
h(x, \pp_i(x)) + \mu & y = \hh_i(x)\\
h(x, \hh_i(x)) - \mu & y = \pp_i(x).
\end{cases}$
Note that such a choice of $h_{\mu, i}$ leads to the following equality holds:
\begin{equation*}
\sum_{y \notin \curl*{\hh_i(x), \pp_i(x)}} p(x, y) \wt \ell_{\rm{sq-hinge}}(h, x, y) = \sum_{y \notin \curl*{\hh_i(x), \pp_i(x)}}  p(x, y) \wt \ell_{\rm{sq-hinge}}(h_{\mu, i}, x, y).
\end{equation*}
Let $q(x, \pp_i(x)) = 1 - p(x, \pp_i(x))$ and $q(x, \hh_i(x)) = 1 - p(x, \hh_i(x))$.
Therefore, for any $i \in [k]$, the conditional regret of the constrained squared hinge loss can be lower bounded as
\begin{align*}
\Delta\sC_{\wt \ell_{\rm{sq-hinge}}, \sH}(h, x) 
&  \geq  \inf_{h \in \sH}  \sup_{\mu\in \Rset} \bigg\{q(x, \pp_i(x))\paren*{\max\curl*{0, 1 + h(x, \pp_i(x))}^2-\max\curl*{0, 1 + h(x,\hh_i(x))-\mu}^2 }\\
& \qquad + q(x,\hh_i(x))\paren*{\max\curl*{0, 1 + h(x,\hh_i(x))}^2-\max\curl*{0, 1 + h(x, \pp_i(x))+\mu}^2}\bigg\}\\
& \geq \frac14 \paren*{q(x,\pp_i(x))-q(x, \hh_i(x))}^2
\tag{differentiating with respect to $\mu$, $h$ to optimize}\\
& = \frac14 \paren*{p(x,\pp_i(x))-p(x, \hh_i(x))}^2
\end{align*}
Therefore, by Lemma~\ref{lemma:regret-target}, the conditional regret of the top-$k$ loss can be upper bounded as follows:
\begin{equation*}
\Delta \sC_{\ell_k, \sH}(h, x) = \sum_{i = 1}^k \paren*{p(x, \pp_i(x)) - p(x, \hh_i(x))} \leq 2 k \paren*{\Delta\sC_{\wt \ell_{\rm{sq-hinge}}, \sH}(h, x) }^{\frac12}.
\end{equation*}
By the concavity, taking expectations on both sides of the preceding equation, we obtain
\begin{equation*}
\sE_{\ell_k}(h) - \sE^*_{\ell_k}(\sH) + \sM_{\ell_k}(\sH) \leq 2 k\paren*{ \sE_{\wt \ell_{\rm{sq-hinge}}}(h) - \sE^*_{\wt \ell_{\rm{sq-hinge}}}(\sH) + \sM_{\wt \ell_{\rm{sq-hinge}}}(\sH) }^{\frac12}.
\end{equation*}
The second
part follows from the fact that when the hypothesis set $\sH$ is
sufficiently rich such that $\sA_{\wt \ell_{\rm{sq-hinge}}}(\sH) = 0$, we
have $\sM_{\wt \ell_{\rm{sq-hinge}}}(\sH) = 0$.

\textbf{Case III: $\wt \ell_{\rm{cstnd}} = \wt \ell_{\rm{hinge}}$.} For the constrained hinge loss $\wt \ell_{\rm{hinge}}$, the conditional regret can be written as 
\begin{align*}
\Delta\sC_{\wt \ell_{\rm{hinge}}, \sH}(h, x) 
& = \sum_{y = 1}^n p(x, y) \wt \ell_{\rm{hinge}}(h, x, y) - \inf_{h \in \sH} \sum_{y = 1}^n p(x, y) \wt \ell_{\rm{hinge}}(h, x, y)\\
& \geq \sum_{y = 1}^n p(x, y) \wt \ell_{\rm{hinge}}(h, x, y) - \inf_{\mu \in \Rset} \sum_{y = 1}^n p(x, y) \wt \ell_{\rm{hinge}}(h_{\mu, i}, x, y),
\end{align*}
where for any $i \in [k]$, $h_{\mu, i}(x, y) = \begin{cases}
h(x, y), & y \notin \curl*{\pp_i(x), \hh_i(x)}\\
h(x, \pp_i(x)) + \mu & y = \hh_i(x)\\
h(x, \hh_i(x)) - \mu & y = \pp_i(x).
\end{cases}$
Note that such a choice of $h_{\mu, i}$ leads to the following equality holds:
\begin{equation*}
\sum_{y \notin \curl*{\hh_i(x), \pp_i(x)}} p(x, y) \wt \ell_{\rm{hinge}}(h, x, y) = \sum_{y \notin \curl*{\hh_i(x), \pp_i(x)}}  p(x, y) \wt \ell_{\rm{hinge}}(h_{\mu, i}, x, y).
\end{equation*}
Let $q(x, \pp_i(x)) = 1 - p(x, \pp_i(x))$ and $q(x, \hh_i(x)) = 1 - p(x, \hh_i(x))$.
Therefore, for any $i \in [k]$, the conditional regret of the constrained hinge loss can be lower bounded as
\begin{align*}
\Delta\sC_{\wt \ell_{\rm{hinge}}, \sH}(h, x) 
&   \geq  \inf_{h \in \sH}  \sup_{\mu\in \Rset} \bigg\{q(x, \pp_i(x))\paren*{\max\curl*{0, 1 + h(x, \pp_i(x))}-\max\curl*{0, 1 + h(x,\hh_i(x))-\mu} }\\
& \qquad + q(x,\hh_i(x))\paren*{\max\curl*{0, 1 + h(x,\hh_i(x))}-\max\curl*{0, 1 + h(x, \pp_i(x))+\mu}}\bigg\}\\
& \geq  q(x,\hh_i(x))-q(x, \pp_i(x))
\tag{differentiating with respect to $\mu$, $h$ to optimize}\\
& = p(x,\pp_i(x))-p(x, \hh_i(x))
\end{align*}
Therefore, by Lemma~\ref{lemma:regret-target}, the conditional regret of the top-$k$ loss can be upper bounded as follows:
\begin{equation*}
\Delta \sC_{\ell_k, \sH}(h, x) = \sum_{i = 1}^k \paren*{p(x, \pp_i(x)) - p(x, \hh_i(x))} \leq k \Delta\sC_{\wt \ell_{\rm{hinge}}, \sH}(h, x).
\end{equation*}
By the concavity, taking expectations on both sides of the preceding equation, we obtain
\begin{equation*}
\sE_{\ell_k}(h) - \sE^*_{\ell_k}(\sH) + \sM_{\ell_k}(\sH) \leq k \paren*{ \sE_{\wt \ell_{\rm{hinge}}}(h) - \sE^*_{\wt \ell_{\rm{hinge}}}(\sH) + \sM_{\wt \ell_{\rm{hinge}}}(\sH)}.
\end{equation*}
The second
part follows from the fact that when the hypothesis set $\sH$ is
sufficiently rich such that $\sA_{\wt \ell_{\rm{hinge}}}(\sH) = 0$, we
have $\sM_{\wt \ell_{\rm{hinge}}}(\sH) = 0$. 

\textbf{Case IV: $\wt \ell_{\rm{cstnd}} = \wt \ell_{\rho}$.} For the constrained $\rho$-margin loss $\wt \ell_{\rho}$, the conditional regret can be written as 
\begin{align*}
\Delta\sC_{\wt \ell_{\rho}, \sH}(h, x) 
& = \sum_{y = 1}^n p(x, y) \wt \ell_{\rho}(h, x, y) - \inf_{h \in \sH} \sum_{y = 1}^n p(x, y) \wt \ell_{\rho}(h, x, y)\\
& \geq \sum_{y = 1}^n p(x, y) \wt \ell_{\rho}(h, x, y) - \inf_{\mu \in \Rset} \sum_{y = 1}^n p(x, y) \wt \ell_{\rho}(h_{\mu, i}, x, y),
\end{align*}
where for any $i \in [k]$, $h_{\mu, i}(x, y) = \begin{cases}
h(x, y), & y \notin \curl*{\pp_i(x), \hh_i(x)}\\
h(x, \pp_i(x)) + \mu & y = \hh_i(x)\\
h(x, \hh_i(x)) - \mu & y = \pp_i(x).
\end{cases}$
Note that such a choice of $h_{\mu, i}$ leads to the following equality holds:
\begin{equation*}
\sum_{y \notin \curl*{\hh_i(x), \pp_i(x)}} p(x, y) \wt \ell_{\rho}(h, x, y) = \sum_{y \notin \curl*{\hh_i(x), \pp_i(x)}}  p(x, y) \wt \ell_{\rho}(h_{\mu, i}, x, y).
\end{equation*}
Let $q(x, \pp_i(x)) = 1 - p(x, \pp_i(x))$ and $q(x, \hh_i(x)) = 1 - p(x, \hh_i(x))$.
Therefore, for any $i \in [k]$, the conditional regret of the constrained $\rho$-margin loss can be lower bounded as
\begin{align*}
& \Delta\sC_{\wt \ell_{\rho}, \sH}(h, x)\\ 
&  \geq  \inf_{h \in \sH} \sup_{\mu\in \Rset} \bigg\{q(x, \pp_i(x))\paren*{\min\curl*{\max\curl*{0,1 + \frac{h(x, \pp_i(x))}{\rho}},1}-\min\curl*{\max\curl*{0,1 + \frac{h(x,\hh_i(x))-\mu}{\rho}},1}}\\
&+q(x,\hh_i(x))\paren*{\min\curl*{\max\curl*{0,1 + \frac{h(x,\hh_i(x))}{\rho}},1}-\min\curl*{\max\curl*{0,1 + \frac{h(x, \pp_i(x))+\mu}{\rho}},1}}\bigg\}\\
& \geq q(x,\hh_i(x))-q(x, \pp_i(x))
\tag{differentiating with respect to $\mu$, $h$ to optimize}\\
& = p(x,\pp_i(x))-p(x, \hh_i(x))
\end{align*}
Therefore, by Lemma~\ref{lemma:regret-target}, the conditional regret of the top-$k$ loss can be upper bounded as follows:
\begin{equation*}
\Delta \sC_{\ell_k, \sH}(h, x) = \sum_{i = 1}^k \paren*{p(x, \pp_i(x)) - p(x, \hh_i(x))} \leq k \Delta\sC_{\wt \ell_{\rho}, \sH}(h, x).
\end{equation*}
By the concavity, taking expectations on both sides of the preceding equation, we obtain
\begin{equation*}
\sE_{\ell_k}(h) - \sE^*_{\ell_k}(\sH) + \sM_{\ell_k}(\sH) \leq k \paren*{ \sE_{\wt \ell_{\rho}}(h) - \sE^*_{\wt \ell_{\rho}}(\sH) + \sM_{\wt \ell_{\rho}}(\sH)}.
\end{equation*}
The second part
follows from the fact that when the hypothesis set $\sH$ is
sufficiently rich such that $\sA_{\wt \ell_{\rho}}(\sH) = 0$, we have
$\sM_{\wt \ell_{\rho}}(\sH) = 0$.
\end{proof}

\newpage
\section{Technical challenges and novelty in Section~\ref{sec:comp}}
\label{app:novelty}

The technical challenges and novelty of proofs in Section~\ref{sec:comp} lie in the following three aspects:

(1) Conditional regret of the top-$k$ loss: This involves a comprehensive analysis of the conditional regret associated with the top-$k$ loss, which is significantly more complex than that of the zero-one loss in a standard setting. The conditional regret of the top-$k$ loss incorporates both the top-$k$ conditional probabilities $\mathsf{p}_i(x)$, for $i = 1, \ldots, k$, and the top-$k$ scores $\mathsf{h}_i(x)$, for $i = 1, \ldots, k$, as characterized in Lemma~\ref{lemma:regret-target}.

(2) Relating to the conditional regret of the surrogate loss: To establish $\sH$-consistency bounds, it is necessary to upper bound the conditional regret of the top-$k$ loss with that of the surrogate loss. This task is particularly challenging in the top-$k$ setting due to the intricate nature of the top-$k$ loss's conditional regret. A pivotal observation is that the conditional regret of the top-$k$ loss can be expressed as the sum of $k$ terms $\left( p(x, \mathsf{p}_i(x)) - p(x, \mathsf{h}_i(x)) \right)$ for $i =1, \ldots, k$. Each term $\left( p(x, \mathsf{p}_i(x)) - p(x, \mathsf{h}_i(x)) \right)$ exhibits structural similarities to the conditional regret of the zero-one loss, $\left( p(x, \mathsf{p}_1(x)) - p(x, \mathsf{h}_1(x)) \right)$. Consequently, we introduce a series of auxiliary hypotheses $h_{\mu, i}$, each dependent on $\mathsf{h}_i(x)$ and $\mathsf{p}_i(x)$ for $i \in [k]$. This approach transforms the challenge of upper bounding the conditional regret of the top-$k$ loss into $k$ subproblems, each focusing on upper bounding the term $\left( p(x, \mathsf{p}_i(x)) - p(x, \mathsf{h}_i(x)) \right)$ with the conditional regret of the surrogate loss.

(3) Upper bounding each term $\left( p(x, \mathsf{p}_i(x)) - p(x, \mathsf{h}_i(x)) \right)$: Following the approach in prior work \citep{mao2023cross} for top-$1$ classification, we define $h_{\mu, i}(x, y)$ as: 
\[h_{\mu, i}(x, y) = \begin{cases}
h(x, y), & y \notin \curl*{\mathsf{p}_i(x)), \mathsf{h}_i(x))}\\
\log\paren*{e^{h(x, \mathsf{p}_i(x))} + \mu} & y = \mathsf{h}_i(x))\\
\log\paren*{e^{h(x, \mathsf{h}_i(x)))} - \mu} & y = \mathsf{p}_i(x)).
\end{cases}\]
for the proof of comp-sum losses (Theorem~\ref{thm:bound-comp}). The subsequent proof is considered straightforward.

However, for the proof of constrained losses (Theorem~\ref{thm:bound-cstnd}), we adopt a different hypothesis formulation for $h_{\mu, i}(x, y)$, leveraging the constraint that the scores sum to zero and the specific structure of constrained losses. The hypothesis is defined as:
\[h_{\mu, i}(x, y) = 
 \begin{cases}
h(x, y), & y \notin \curl*{\mathsf{p}_i(x)), \mathsf{h}_i(x))}\\
h(x, \mathsf{p}_i(x)) + \mu & y = \mathsf{h}_i(x))\\
h(x, \mathsf{h}_i(x))) - \mu & y = \mathsf{p}_i(x)).
\end{cases}
\]
The remainder of the proof then specifically addresses the peculiarities of constrained losses, which significantly diverges from the previous work.

In summary, aspects (1) and (2) are novel and represent significant advancements that have not been explored previously. For aspect (3), the proof for comp-sum loss closely follows the approach in \citep{mao2023cross}, which appears straightforward due to the innovative ideas presented in aspects (1) and (2). However, the proof for constrained losses significantly deviates from the previous work, particularly in terms of the new auxiliary hypothesis formulation and the specific constrained losses examined.

We would like to further emphasize that these results are significant and useful. They demonstrate that comp-sum losses, which include the cross-entropy loss commonly used in top-$1$ classification, and constrained losses, are $\sH$-consistent in top-$k$ classification for any $k$. Notably, the cross-entropy loss is the only Bayes-consistent smooth surrogate loss for top-$k$ classification identified to date. Furthermore, the Bayes-consistency of loss functions within the constrained loss family is a novel exploration in the context of top-$k$ classification. These findings are pivotal as they highlight two broad families of smooth loss functions that are Bayes-consistent in top-$k$ classification. Additionally, they reveal that these families, including the cross-entropy loss, benefit from stronger, non-asymptotic and hypothesis set-specific guarantees—$\sH$-consistency bounds—in top-$k$ classification.

\newpage
\section{Generalization bounds}
\label{app:generalization}

Given a finite sample $S = \paren*{(x_1, y_1), \ldots, (x_m, y_m)}$ drawn from
$\sD^m$,  let $\h h_S$ be the minimizer of the
empirical loss within $\sH$ with respect to the top-$k$ surrogate loss $\wt \ell$:
$
\h h_S = \argmin_{h \in \sH} \h \sE_{\wt \ell, S}(h) = \argmin_{h\in \sH} \frac{1}{m}\sum_{i = 1}^m \wt \ell(h, x_i,y _i).
$
Next, we will show that we can use $\sH$-consistency bounds for $\wt \ell$ to derive generalization bounds for the top-$k$ loss by upper bounding the surrogate
estimation error $\sE_{\wt \ell}(\h h_S) - \sE_{\wt \ell}^*(\sH)$ with the
complexity (e.g. the Rademacher complexity) of the family of functions
associated with $\wt \ell$ and $\sH$: $\sH_{\wt \ell}=\curl*{(x, y) \mapsto
  \wt \ell(h, x, y) \colon h \in \sH}$.

Let $\Rad_m^{\wt \ell}(\sH)$ be the Rademacher complexity of
$\sH_{\wt \ell}$ and $B_{\wt \ell}$ an upper bound of the surrogate loss
$\wt \ell$. Then, we obtain the following generalization bounds for the top-$k$
loss.

\begin{restatable}[\textbf{Generalization bound with comp-sum losses}]{theorem}{GBoundComp}
\label{Thm:Gbound-comp}
Assume that $\sH$ is symmetric and complete. Then, for any $1 \leq k
\leq n$, the following top-$k$ generalization bound holds for $\h h_S$: for any
$\delta > 0$, with probability at least $1-\delta$ over the draw of an
i.i.d sample $S$ of size $m$:
\begin{equation*}
\sE_{\ell_k}(\h h_S) - \sE_{\ell_k}^*( \sH) + \sM_{\ell_k}( \sH) \leq k \psi^{-1}\paren*{4
    \Rad_m^{\wt \ell}(\sH) + 2 B_{\wt \ell} \sqrt{\tfrac{\log
        \frac{2}{\delta}}{2m}} + \sM_{\wt \ell}( \sH)}.
\end{equation*}
where $\psi(t) = \frac{1 - t}{2}\log(1 - t) + \frac{1 + t}{2}\log(1+
t)$, $t \in [0,1]$ when $\wt \ell$ is $\wt \ell_{\log}$; $\psi(t) = 1 - \sqrt{1 - t^2}$, $t \in [0,1]$ when $\wt \ell$ is $\wt \ell_{\exp}$; $\psi(t) = t / n$ when $\wt \ell$ is $\wt \ell_{\rm{mae}}$; and $\psi(t) = \frac{1}{\q n^{\q}}
\bracket*{\bracket*{\frac{\paren*{1 + t}^{\frac1{1 - \q }} +
      \paren*{1 - t}^{\frac1{1 - \q }}}{2}}^{1 - \q } -1}$,
for all $\q \in (0,1)$, $t \in [0, 1]$ when $\wt \ell$ is $\wt \ell_{\rm{gce}}$.
\end{restatable}
\begin{proof}
  By using the standard Rademacher complexity bounds \citep{mohri2018foundations}, for any $\delta>0$,
  with probability at least $1 - \delta$, the following holds for all $h \in \sH$:
\[
\abs*{\sE_{\wt \ell}(h) - \h\sE_{\wt \ell, S}(h)}
\leq 2 \Rad_m^{\wt \ell}(\sH) +
B_{\wt \ell} \sqrt{\tfrac{\log (2/\delta)}{2m}}.
\]
Fix $\e > 0$. By the definition of the infimum, there exists $h^* \in
\sH$ such that $\sE_{\wt \ell}(h^*) \leq
\sE_{\wt \ell}^*(\sH) + \e$. By definition of
$\h h_S$, we have
\begin{align*}
  & \sE_{\wt \ell}(\h h_S) - \sE_{\wt \ell}^*(\sH)\\
  & = \sE_{\wt \ell}(\h h_S) - \h\sE_{\wt \ell, S}(\h h_S) + \h\sE_{\wt \ell, S}(\h h_S) - \sE_{\wt \ell}^*(\sH)\\
  & \leq \sE_{\wt \ell}(\h h_S) - \h\sE_{\wt \ell, S}(\h h_S) + \h\sE_{\wt \ell, S}(h^*) - \sE_{\wt \ell}^*(\sH)\\
  & \leq \sE_{\wt \ell}(\h h_S) - \h\sE_{\wt \ell, S}(\h h_S) + \h\sE_{\wt \ell, S}(h^*) - \sE_{\wt \ell}^*(h^*) + \e\\
  & \leq
  2 \bracket*{2 \Rad_m^{\wt \ell}(\sH) +
B_{\wt \ell} \sqrt{\tfrac{\log (2/\delta)}{2m}}} + \e.
\end{align*}
Since the inequality holds for all $\e > 0$, it implies:
\[
\sE_{\wt \ell}(\h h_S) - \sE_{\wt \ell}^*(\sH)
\leq 
4 \Rad_m^{\wt \ell}(\sH) +
2 B_{\wt \ell} \sqrt{\tfrac{\log (2/\delta)}{2m}}.
\]
Plugging in this inequality in the bounds of Theorem~\ref{thm:bound-comp} completes the proof.
\end{proof}
\begin{restatable}[\textbf{Generalization bound with constrained losses}]{theorem}{GBoundCstnd}
\label{Thm:Gbound-cstnd}
Assume that $\sH$ is symmetric and complete. Then, for any $1 \leq k
\leq n$, the following top-$k$ generalization bound holds for $\h h_S$: for any
$\delta > 0$, with probability at least $1-\delta$ over the draw of an
i.i.d sample $S$ of size $m$:
\begin{equation*}
\sE_{\ell_k}(\h h_S) - \sE_{\ell_k}^*( \sH) + \sM_{\ell_k}( \sH) \leq k \gamma \paren*{4
    \Rad_m^{\wt \ell}(\sH) + 2 B_{\wt \ell} \sqrt{\tfrac{\log
        \frac{2}{\delta}}{2m}} + \sM_{\wt \ell}( \sH)}.
\end{equation*}
where $\gamma(t) = 2\sqrt{t}$ when
$\wt \ell$ is either $\wt \ell^{\mathrm{cstnd}}_{\exp}$ or
$\wt \ell_{\rm{sq-hinge}}$; $\gamma(t) = t$ when
$\wt \ell$ is either $\wt \ell_{\rm{hinge}}$ or $\wt \ell_{\rho}$.
\end{restatable}
\begin{proof}
  By using the standard Rademacher complexity bounds \citep{mohri2018foundations}, for any $\delta>0$,
  with probability at least $1 - \delta$, the following holds for all $h \in \sH$:
\[
\abs*{\sE_{\wt \ell}(h) - \h\sE_{\wt \ell, S}(h)}
\leq 2 \Rad_m^{\wt \ell}(\sH) +
B_{\wt \ell} \sqrt{\tfrac{\log (2/\delta)}{2m}}.
\]
Fix $\e > 0$. By the definition of the infimum, there exists $h^* \in
\sH$ such that $\sE_{\wt \ell}(h^*) \leq
\sE_{\wt \ell}^*(\sH) + \e$. By definition of
$\h h_S$, we have
\begin{align*}
  & \sE_{\wt \ell}(\h h_S) - \sE_{\wt \ell}^*(\sH)\\
  & = \sE_{\wt \ell}(\h h_S) - \h\sE_{\wt \ell, S}(\h h_S) + \h\sE_{\wt \ell, S}(\h h_S) - \sE_{\wt \ell}^*(\sH)\\
  & \leq \sE_{\wt \ell}(\h h_S) - \h\sE_{\wt \ell, S}(\h h_S) + \h\sE_{\wt \ell, S}(h^*) - \sE_{\wt \ell}^*(\sH)\\
  & \leq \sE_{\wt \ell}(\h h_S) - \h\sE_{\wt \ell, S}(\h h_S) + \h\sE_{\wt \ell, S}(h^*) - \sE_{\wt \ell}^*(h^*) + \e\\
  & \leq
  2 \bracket*{2 \Rad_m^{\wt \ell}(\sH) +
B_{\wt \ell} \sqrt{\tfrac{\log (2/\delta)}{2m}}} + \e.
\end{align*}
Since the inequality holds for all $\e > 0$, it implies:
\[
\sE_{\wt \ell}(\h h_S) - \sE_{\wt \ell}^*(\sH)
\leq 
4 \Rad_m^{\wt \ell}(\sH) +
2 B_{\wt \ell} \sqrt{\tfrac{\log (2/\delta)}{2m}}.
\]
Plugging in this inequality in the bounds of Theorem~\ref{thm:bound-cstnd} completes the proof.
\end{proof}
To the best
of our knowledge, Theorems~\ref{Thm:Gbound-comp} and \ref{Thm:Gbound-cstnd} provide the first
finite-sample guarantees for the estimation error of the minimizer of comp-sum losses and constrained losses, with respect to the top-$k$ loss, for any $1 \leq k \leq n$.  The
proofs use our $\sH$-consistency bounds with respect to the top-$k$ loss, as well as standard Rademacher complexity guarantees.

\newpage
\section{Proofs of \texorpdfstring{$\sH$}{H}-consistency bounds for cost-sensitive losses}
\label{app:cost}

We first characterize the best-in class conditional error and the
conditional regret of the target cardinality aware loss function \eqref{eq:target-cardinality}, which will be used in the analysis
of $\sH$-consistency bounds.
\begin{restatable}{lemma}{RegretTargetCost}
\label{lemma:regret-target-cost}
Assume that $\sR$ is symmetric and complete. Then, for any $r \in \sK$ and $x \in
\sX$, the best-in class conditional error and the conditional regret
of the target cardinality aware loss function can be expressed as follows:
\begin{align*}
  \sC^*_{\ell}(\sR, x)
  & = \min_{k  \in \sK} \sum_{y\in \sY}  p(x, y) c(x, k, y)\\
\Delta \sC_{\ell, \sR}(r, x)
& = \sum_{y\in \sY}  p(x, y) c(x, \rr(x), y) - \min_{k  \in \sK} \sum_{y\in \sY}  p(x, y) c(x, k, y).
\end{align*}
\end{restatable}
\begin{proof}
By definition, for any $r \in \sR$ and $x \in \sX$, the conditional
error of the target cardinality aware loss function can be written as
\begin{equation*}
\sC_{ \ell }(r, x) =  \sum_{y\in \sY} p(x,y) c(x, \rr(x), y).
\end{equation*}
Since $\sR$ is symmetric and complete, we have
\begin{equation*}
  \sC^*_{\ell}(\sR, x)
  = \inf_{r \in \sR} \sum_{y\in \sY} p(x,y) c(x, \rr(x), y) = \min_{k  \in \sK} \sum_{y \in \sY} p(x, y) c(x, k, y).
\end{equation*}
Furthermore, the calibration gap can be expressed as
\begin{align*}
\Delta\sC_{\ell, \sR}(r, x)  = \sC_{ \ell }(r, x) - \sC^*_{ \ell }(\sR, x) = \sum_{y\in \sY}  p(x, y) c(x, \rr(x), y) - \min_{k  \in \sK} \sum_{y\in \sY}  p(x, y) c(x, k, y),
\end{align*}
which completes the proof.
\end{proof}

\subsection{Proof of Theorem~\ref{thm:bound-cost-comp}}
\label{app:bound-cost-comp}
For convenience, we let $\ov c(x, k, y) = 1 - c(x, k, y)$, $\ov q(x, k) = \sum_{y\in \sY} p(x, y) \ov c(x, k, y) \in [0, 1]$ and $\sS(x, k) = \frac{e^{r(x, k)}}{\sum_{k' \in \sK}e^{r(x, k')}}$. We also let $k_{\min}(x)  =  \argmin_{k \in \sK} \paren*{1 - \ov q(x, k)} = \argmin_{k \in \sK} \sum_{y\in \sY} p(x, y) c(x, k, y)$. 
\BoundCostComp*
\begin{proof}
\textbf{Case I: $\wt \ell_{\rm{c-comp}} = \wt \ell_{\rm{c-log}}$.} For the cost-sensitive logistic loss $\wt \ell_{\rm{c-log}}$, the conditional error can be written as 
\begin{align*}
\sC_{\wt \ell_{\rm{c-log}} }(r, x) = -\sum_{y\in \sY} p(x, y) \sum_{k \in \sK} \ov c(x, k, y) \log\paren*{\frac{e^{r(x, k)}}{\sum_{k'\in \sK}e^{r(x, k')}}}  =   - \sum_{k \in \sK}\log\paren*{\sS(x, k)}\ov q(x, k).
\end{align*}
The conditional regret can be written as 
\begin{align*}
\Delta\sC_{\wt \ell_{\rm{c-log}}, \sR}(r, x) 
& = - \sum_{k \in \sK}\log\paren*{\sS(x, k)}\ov q(x, k) - \inf_{r \in \sR} \paren*{- \sum_{k \in \sK}\log\paren*{\sS(x, k)}\ov q(x, k)}\\
& \geq - \sum_{k \in \sK}\log\paren*{\sS(x, k)}\ov q(x, k) - \inf_{\mu \in \bracket*{-\sS(x, k_{\min}(x)), \sS(x, \rr(x))}} \paren*{- \sum_{k \in \sK}\log\paren*{\sS_{\mu}(x, k)}\ov q(x, k)},
\end{align*}
where for any $x \in \sX$ and $k \in \sK$, $\sS_{\mu}(x, k) = \begin{cases}
\sS(x, y), & y \notin \curl*{k_{\min}(x), \rr(x)}\\
\sS(x, k_{\min}(x)) + \mu & y = \rr(x)\\
\sS(x, \rr(x)) - \mu & y = k_{\min}(x).
\end{cases}$
Note that such a choice of $\sS_{\mu}$ leads to the following equality holds:
\begin{equation*}
\sum_{k \notin \curl*{\rr(x), k_{\min}(x)}} \log\paren*{\sS(x, k)}\ov q(x, k) = \sum_{k \notin \curl*{\rr(x), k_{\min}(x)}}  \log\paren*{\sS_{\mu}(x, k)}\ov q(x, k).
\end{equation*}
Therefore, the conditional regret of cost-sensitive logistic loss can be lower bounded as
\begin{align*}
\Delta\sC_{\wt \ell_{\rm{c-log}}, \sH}(h, x)  & \geq \sup_{\mu \in [-\sS(x, k_{\min}(x)),\sS(x,\rr(x))]} \bigg\{\ov q(x, k_{\min}(x))\bracket*{-\log\paren*{\sS(x, k_{\min}(x))} + \log\paren*{\sS(x,\rr(x)) - \mu}}\\
& \qquad + \ov q(x,\rr(x))\bracket*{-\log\paren*{\sS(x,\rr(x))}
    + \log\paren*{\sS(x, k_{\min}(x))+\mu}}\bigg\}.
\end{align*}
By the concavity of the function, differentiate with respect to $\mu$, we obtain that the supremum is achieved by $\mu^* = \frac{\ov q(x,\rr(x))\sS(x,\rr(x))-\ov q(x, k_{\min}(x))\sS(x,
  k_{\min}(x))}{\ov q(x, k_{\min}(x))+\ov q(x,\rr(x))}$. Plug in $\mu^*$, we obtain
\begin{align*}
& \Delta\sC_{\wt \ell_{\rm{c-log}}, \sH}(h, x)\\ 
& \geq \ov q(x, k_{\min}(x))\log\frac{\paren*{\sS(x,\rr(x))+\sS(x, k_{\min}(x))}\ov q(x, k_{\min}(x))}{\sS(x, k_{\min}(x))\paren*{\ov q(x, k_{\min}(x))+\ov q(x,\rr(x))}}\\
& \qquad +\ov q(x,\rr(x))\log\frac{\paren*{\sS(x,\rr(x))+\sS(x, k_{\min}(x))}\ov q(x,\rr(x))}{\sS(x,\rr(x))\paren*{\ov q(x, k_{\min}(x))+\ov q(x,\rr(x))}}\\
& \geq \ov q(x, k_{\min}(x))\log\frac{2\ov q(x, k_{\min}(x))}{\ov q(x, k_{\min}(x))+\ov q(x,\rr(x))} +\ov q(x,\rr(x))\log\frac{2\ov q(x,\rr(x))}{\ov q(x, k_{\min}(x))+\ov q(x,\rr(x))}
\tag{minimum is achieved when $\sS(x, \rr(x)) = \sS(x, k_{\min}(x))$}\\
& \geq \frac{\paren*{\ov q(x,\rr(x))-\ov q(x, k_{\min}(x))}^2}{2\paren*{\ov q(x,\rr(x))+\ov q(x, k_{\min}(x))}}
\tag{$a\log \frac{2a}{a+b}+b\log \frac{2b}{a+b}\geq \frac{(a-b)^2}{2(a+b)}, \forall a,b\in[0,1]$ \citep[Proposition~E.7]{mohri2018foundations}}\\
& \geq \frac{\paren*{\ov q(x,\rr(x))-\ov q(x, k_{\min}(x))}^2}{4} \tag{$0\leq \ov q(x,\rr(x))+\ov q(x, k_{\min}(x))\leq 2$}.
\end{align*}
Therefore, by Lemma~\ref{lemma:regret-target-cost}, the conditional regret of the target cardinality aware loss function can be upper bounded as follows:
\begin{equation*}
\Delta \sC_{\ell, \sH}(r, x) =  \ov q(x, k_{\min}(x)) - \ov q(x, \rr(x)) \leq 2 \paren*{\Delta\sC_{\wt \ell_{\rm{c-log}}, \sR}(r, x) }^{\frac12}.
\end{equation*}
By the concavity, taking expectations on both sides of the preceding equation, we obtain
\begin{equation*}
\sE_{\ell}(r) - \sE^*_{\ell}(\sR) + \sM_{\ell}(\sR) \leq 2 \paren*{ \sE_{\wt \ell_{\rm{c-log}}}(r) - \sE^*_{\wt \ell_{\rm{c-log}}}(\sR) + \sM_{\wt \ell_{\rm{c-log}}}(\sR) }^{\frac12}.
\end{equation*}
The second part follows from the fact that 
$\sM_{\wt \ell_{\rm{c-log}}}(\sR_{\rm{all}}) = 0$.

\textbf{Case II: $\wt \ell_{\rm{c-comp}} = \wt \ell_{\rm{c-exp}}$.} For the cost-sensitive sum exponential loss $\wt \ell_{\rm{c-exp}}$, the conditional error can be written as 
\begin{align*}
\sC_{\wt \ell_{\rm{c-exp}} }(r, x) = \sum_{y\in \sY} p(x, y) \sum_{k \in \sK} \ov c(x, k, y) \sum_{k'\neq k'}e^{r(x, k')-r(x, k)}  =   \sum_{k \in \sK}\paren*{\frac{1}{\sS(x, k)}-1}\ov q(x, k).
\end{align*}
The conditional regret can be written as 
\begin{align*}
\Delta \sC_{\wt \ell_{\rm{c-exp}}, \sR}(r, x) 
& = \sum_{k \in \sK}\paren*{\frac{1}{\sS(x, k)}-1}\ov q(x, k) - \inf_{r \in \sR} \paren*{\sum_{k \in \sK}\paren*{\frac{1}{\sS(x, k)}-1}\ov q(x, k)}\\
& \geq \sum_{k \in \sK}\paren*{\frac{1}{\sS(x, k)}-1}\ov q(x, k) - \inf_{\mu \in \bracket*{-\sS(x, k_{\min}(x)), \sS(x, \rr(x))}} \paren*{\sum_{k \in \sK}\paren*{\frac{1}{\sS_{\mu}(x, k)}-1}\ov q(x, k)},
\end{align*}
where for any $x \in \sX$ and $k \in \sK$, $\sS_{\mu}(x, k) = \begin{cases}
\sS(x, y), & y \notin \curl*{k_{\min}(x), \rr(x)}\\
\sS(x, k_{\min}(x)) + \mu & y = \rr(x)\\
\sS(x, \rr(x)) - \mu & y = k_{\min}(x).
\end{cases}$
Note that such a choice of $\sS_{\mu}$ leads to the following equality holds:
\begin{equation*}
\sum_{k \notin \curl*{\rr(x), k_{\min}(x)}} \paren*{\frac{1}{\sS(x, k)}-1}\ov q(x, k) = \sum_{k \notin \curl*{\rr(x), k_{\min}(x)}}  \paren*{\frac{1}{\sS_{\mu}(x, k)}-1}\ov q(x, k).
\end{equation*}
Therefore, the conditional regret of cost-sensitive sum exponential loss can be lower bounded as
\begin{align*}
\Delta\sC_{\wt \ell_{\rm{c-exp}}, \sH}(h, x)  & \geq \sup_{\mu \in [-\sS(x, k_{\min}(x)),\sS(x,\rr(x))]} \bigg\{\ov q(x, k_{\min}(x))\bracket*{\frac{1}{\sS(x, k_{\min}(x))}-\frac{1}{\sS(x,\rr(x))-\mu}}\\
&\qquad +\ov q(x, \rr(x))\bracket*{\frac{1}{\sS(x,\rr(x))}-\frac{1}{\sS(x, k_{\min}(x))+ \mu}}\bigg\}.
\end{align*}
By the concavity of the function, differentiate with respect to $\mu$, we obtain that the supremum is achieved by $\mu^* = \frac{\sqrt{\ov q(x,\rr(x)})\sS(x,\rr(x))-\sqrt{\ov q(x, k_{\min}(x))}\sS(x, k_{\min}(x))}{\sqrt{\ov q(x, k_{\min}(x))}+ \sqrt{\ov q(x, \rr(x))}}$. Plug in $\mu^*$, we obtain
\begin{align*}
& \Delta\sC_{\wt \ell_{\rm{c-exp}}, \sH}(h, x)\\ 
& \geq \frac{\ov q(x, k_{\min}(x))}{\sS(x, k_{\min}(x))}+ \frac{\ov q(x, \rr(x)))}{\sS(x,\rr(x)))}-\frac{\paren*{\sqrt{\ov q(x, k_{\min}(x))}+ \sqrt{\ov q(x, \rr(x)))}}^2}{\sS(x, k_{\min}(x))+ \sS(x,\rr(x)))}\\
& \geq \paren*{\sqrt{\ov q(x, k_{\min}(x))}-\sqrt{\ov q(x, \rr(x)))}}^2
\tag{minimum is achieved when $\sS(x, \rr(x)) = \sS(x, k_{\min}(x)) = \frac12$}\\
& \geq \frac{\paren*{\ov q(x, \rr(x)))-\ov q(x, k_{\min}(x))}^2}{\paren*{\sqrt{\ov q(x, \rr(x)))}+ \sqrt{\ov q(x, k_{\min}(x))}}^2}\\
& \geq \frac{\paren*{\ov q(x, \rr(x)))-\ov q(x, k_{\min}(x))}^2}{4}
\tag{$\sqrt{a}+ \sqrt{b}\leq 2, \forall a,b\in[0,1], a+b\leq 2$}.
\end{align*}
Therefore, by Lemma~\ref{lemma:regret-target-cost}, the conditional regret of the target cardinality aware loss function can be upper bounded as follows:
\begin{equation*}
\Delta \sC_{\ell, \sH}(r, x) =  \ov q(x, k_{\min}(x)) - \ov q(x, \rr(x)) \leq 2 \paren*{\Delta\sC_{\wt \ell_{\rm{c-exp}}, \sR}(r, x) }^{\frac12}.
\end{equation*}
By the concavity, taking expectations on both sides of the preceding equation, we obtain
\begin{equation*}
\sE_{\ell}(r) - \sE^*_{\ell}(\sR) + \sM_{\ell}(\sR) \leq 2 \paren*{ \sE_{\wt \ell_{\rm{c-exp}}}(r) - \sE^*_{\wt \ell_{\rm{c-exp}}}(\sR) + \sM_{\wt \ell_{\rm{c-exp}}}(\sR) }^{\frac12}.
\end{equation*}
The second part follows from the fact that 
$\sM_{\wt \ell_{\rm{c-exp}}}(\sR_{\rm{all}}) = 0$.

\textbf{Case III: $\wt \ell_{\rm{c-comp}} = \wt \ell_{\rm{c-gce}}$.} For the cost-sensitive generalized cross-entropy  loss $\wt \ell_{\rm{c-gce}}$, the conditional error can be written as 
\begin{align*}
\sC_{\wt \ell_{\rm{c-gce}} }(r, x) = \sum_{y\in \sY} p(x, y) \sum_{k \in \sK} \ov c(x, k, y)\frac{1}{\q}\paren*{1 - \paren*{\frac{e^{r(x, k)}}{\sum_{k'\in \sK}e^{r(x, k')}}}^{\q}}  =   \frac{1}{\q} \sum_{k \in \sK}\paren*{1 - \sS(x, k)^{\q}}\ov q(x, k).
\end{align*}
The conditional regret can be written as 
\begin{align*}
\Delta\sC_{\wt \ell_{\rm{c-gce}}, \sR}(r, x) 
& = \frac{1}{\q} \sum_{k \in \sK}\paren*{1 - \sS(x, k)^{\q}}\ov q(x, k) - \inf_{r \in \sR} \paren*{\frac{1}{\q} \sum_{k \in \sK}\paren*{1 - \sS(x, k)^{\q}}\ov q(x, k)}\\
& \geq \frac{1}{\q} \sum_{k \in \sK}\paren*{1 - \sS(x, k)^{\q}}\ov q(x, k) - \inf_{\mu \in \bracket*{-\sS(x, k_{\min}(x)), \sS(x, \rr(x))}} \paren*{\frac{1}{\q} \sum_{k \in \sK}\paren*{1 - \sS_{\mu}(x, k)^{\q}}\ov q(x, k)},
\end{align*}
where for any $x \in \sX$ and $k \in \sK$, $\sS_{\mu}(x, k) = \begin{cases}
\sS(x, y), & y \notin \curl*{k_{\min}(x), \rr(x)}\\
\sS(x, k_{\min}(x)) + \mu & y = \rr(x)\\
\sS(x, \rr(x)) - \mu & y = k_{\min}(x).
\end{cases}$
Note that such a choice of $\sS_{\mu}$ leads to the following equality holds:
\begin{equation*}
\sum_{k \notin \curl*{\rr(x), k_{\min}(x)}} \frac{1}{\q} \sum_{k \in \sK}\paren*{1 - \sS(x, k)^{\q}}\ov q(x, k) = \sum_{k \notin \curl*{\rr(x), k_{\min}(x)}}  \frac{1}{\q} \sum_{k \in \sK}\paren*{1 - \sS_{\mu}(x, k)^{\q}}\ov q(x, k).
\end{equation*}
Therefore, the conditional regret of cost-sensitive generalized cross-entropy loss can be lower bounded as
\begin{align*}
\Delta\sC_{\wt \ell_{\rm{c-gce}}, \sH}(h, x)  &   =  \frac{1}{\q} \sup_{\mu \in [-\sS(x, k_{\min}(x)),\sS(x, \rr(x))]} \bigg\{\ov q(x, k_{\min}(x))\bracket*{-\sS(x, k_{\min}(x))^{\q}+\paren*{\sS(x, \rr(x))-\mu}^{\q}}\\
&\qquad +\ov q(x, \rr(x))\bracket*{-\sS(x, \rr(x))^{\q}+ \paren*{\sS(x, k_{\min}(x))+ \mu}^{\q}}\bigg\}.
\end{align*}
By the concavity of the function, differentiate with respect to $\mu$, we obtain that the supremum is achieved by $\mu^* = \frac{\ov q(x, \rr(x))^{\frac{1}{1-\q}}\sS(x, \rr(x))-\ov q(x, k_{\min}(x))^{\frac{1}{1-\q}}\sS(x, k_{\min}(x))}{\ov q(x, k_{\min}(x))^{\frac{1}{1-\q}}+\ov q(x, \rr(x))^{\frac{1}{1-\q}}}$. Plug in $\mu^*$, we obtain
\begin{align*}
& \Delta\sC_{\wt \ell_{\rm{c-gce}}, \sH}(h, x)\\ 
&   \geq  \frac{1}{\q}\paren*{\sS(x, \rr(x))+ \sS(x, k_{\min}(x))}^{\q}\paren*{\ov q(x, k_{\min}(x))^{\frac{1}{1-\q}}+\ov q(x, \rr(x))^{\frac{1}{1-\q}}}^{1-\q}\\
&\qquad-\frac{1}{\q}\ov q(x, k_{\min}(x))\sS(x, k_{\min}(x))^{\q}-\frac{1}{\q}\ov q(x, \rr(x))\sS(x, \rr(x))^{\q}\\
& \geq \frac{1}{\q \abs*{\sK}^{\q}}\bracket*{2^{\q}\paren*{\ov q(x, k_{\min}(x))^{\frac{1}{1-\q}}+\ov q(x, \rr(x))^{\frac{1}{1-\q}}}^{1-\q}-\ov q(x, k_{\min}(x))-\ov q(x, \rr(x))}
\tag{minimum is achieved when $\sS(x, \rr(x)) = \sS(x, k_{\min}(x)) = \frac1{\abs*{\sK}}$}\\
& \geq \frac{\paren*{\ov q(x, \rr(x))-\ov q(x, k_{\min}(x))}^2}{4\abs*{\sK}^{\q}}
\tag{$\paren*{\frac{a^{\frac{1}{1-\q}}+b^{\frac{1}{1-\q}}}{2}}^{1-\q}-\frac{a+b}{2}\geq \frac{\q}{4}(a-b)^2, \forall a,b\in[0,1]$, $0\leq a+b\leq 1$}.
\end{align*}
Therefore, by Lemma~\ref{lemma:regret-target-cost}, the conditional regret of the target cardinality aware loss function can be upper bounded as follows:
\begin{equation*}
\Delta \sC_{\ell, \sH}(r, x) =  \ov q(x, k_{\min}(x)) - \ov q(x, \rr(x)) \leq 2 \abs*{\sK}^{\frac{\q}{2}} \paren*{\Delta\sC_{\wt \ell_{\rm{c-gce}}, \sR}(r, x) }^{\frac12}.
\end{equation*}
By the concavity, taking expectations on both sides of the preceding equation, we obtain
\begin{equation*}
\sE_{\ell}(r) - \sE^*_{\ell}(\sR) + \sM_{\ell}(\sR) \leq 2 \abs*{\sK}^{\frac{\q}{2}}\paren*{ \sE_{\wt \ell_{\rm{c-gce}}}(r) - \sE^*_{\wt \ell_{\rm{c-gce}}}(\sR) + \sM_{\wt \ell_{\rm{c-gce}}}(\sR) }^{\frac12}.
\end{equation*}
The second part follows from the fact that 
$\sM_{\wt \ell_{\rm{c-gce}}}(\sR_{\rm{all}}) = 0$.

\textbf{Case IV: $\wt \ell_{\rm{c-comp}} = \wt \ell_{\rm{c-mae}}$.} For the cost-sensitive mean absolute error loss $\wt \ell_{\rm{c-mae}}$, the conditional error can be written as 
\begin{align*}
\sC_{\wt \ell_{\rm{c-mae}} }(r, x) = \sum_{y\in \sY} p(x, y) \sum_{k \in \sK} \ov c(x, k, y) \paren*{1 - \paren*{\frac{e^{r(x, k)}}{\sum_{k'\in \sK}e^{r(x, k')}}}}  =  \sum_{k \in \sK}\paren*{1 - \sS(x, k)}\ov q(x, k).
\end{align*}
The conditional regret can be written as 
\begin{align*}
\Delta\sC_{\wt \ell_{\rm{c-mae}}, \sR}(r, x) 
& =\sum_{k \in \sK}\paren*{1 - \sS(x, k)}\ov q(x, k) - \inf_{r \in \sR} \paren*{\sum_{k \in \sK}\paren*{1 - \sS(x, k)}\ov q(x, k)}\\
& \geq \sum_{k \in \sK}\paren*{1 - \sS(x, k)}\ov q(x, k) - \inf_{\mu \in \bracket*{-\sS(x, k_{\min}(x)), \sS(x, \rr(x))}} \paren*{\sum_{k \in \sK}\paren*{1 - \sS_{\mu}(x, k)}\ov q(x, k)},
\end{align*}
where for any $x \in \sX$ and $k \in \sK$, $\sS_{\mu}(x, k) = \begin{cases}
\sS(x, y), & y \notin \curl*{k_{\min}(x), \rr(x)}\\
\sS(x, k_{\min}(x)) + \mu & y = \rr(x)\\
\sS(x, \rr(x)) - \mu & y = k_{\min}(x).
\end{cases}$
Note that such a choice of $\sS_{\mu}$ leads to the following equality holds:
\begin{equation*}
\sum_{k \in \sK}\paren*{1 - \sS(x, k)}\ov q(x, k) = \sum_{k \in \sK}\paren*{1 - \sS_{\mu}(x, k)}\ov q(x, k).
\end{equation*}
Therefore, the conditional regret of cost-sensitive mean absolute error can be lower bounded as
\begin{align*}
\Delta\sC_{\wt \ell_{\rm{c-mae}}, \sH}(h, x)  & \geq \sup_{\mu \in [-\sS(x, k_{\min}(x)),\sS(x, \rr(x))]} \bigg\{\ov q(x, k_{\min}(x))\bracket*{-\sS(x, k_{\min}(x))+\sS(x, \rr(x))-\mu}\\
&\qquad +\ov q(x, \rr(x))\bracket*{-\sS(x, \rr(x))+ \sS(x, k_{\min}(x))+ \mu}\bigg\}.
\end{align*}
By the concavity of the function, differentiate with respect to $\mu$, we obtain that the supremum is achieved by $\mu^* = -\sS(x, k_{\min}(x))$. Plug in $\mu^*$, we obtain
\begin{align*}
& \Delta\sC_{\wt \ell_{\rm{c-mae}}, \sH}(h, x)\\ 
& \geq \ov q(x, k_{\min}(x))\sS(x, \rr(x))-\ov q(x, \rr(x))\sS(x, \rr(x))\\
& \geq \frac{1}{\abs*{\sK}}\paren*{\ov q(x, k_{\min}(x))-\ov q(x, \rr(x))}
\tag{minimum is achieved when $\sS(x, \rr(x)) = \frac1{\abs*{\sK}}$}.
\end{align*}
Therefore, by Lemma~\ref{lemma:regret-target-cost}, the conditional regret of the target cardinality aware loss function can be upper bounded as follows:
\begin{equation*}
\Delta \sC_{\ell, \sH}(r, x) =  \ov q(x, k_{\min}(x)) - \ov q(x, \rr(x)) \leq \abs*{\sK} \paren*{\Delta\sC_{\wt \ell_{\rm{c-mae}}, \sR}(r, x) }.
\end{equation*}
By the concavity, taking expectations on both sides of the preceding equation, we obtain
\begin{equation*}
\sE_{\ell}(r) - \sE^*_{\ell}(\sR) + \sM_{\ell}(\sR) \leq \abs*{\sK} \paren*{ \sE_{\wt \ell_{\rm{c-mae}}}(r) - \sE^*_{\wt \ell_{\rm{c-mae}}}(\sR) + \sM_{\wt \ell_{\rm{c-mae}}}(\sR) }.
\end{equation*}
The second part follows from the fact that 
$\sM_{\wt \ell_{\rm{c-mae}}}(\sR_{\rm{all}}) = 0$.
\end{proof}

\subsection{Proof of Theorem~\ref{thm:bound-cost-cstnd}}
\label{app:bound-cost-cstnd}

The conditional error for the cost-sensitive constrained loss can be expressed as follows:
\begin{align*}
\sC_{\wt \ell_{\rm{c-cstnd}}}(r, x) 
&=  \sum_{y\in \sY} p(x, y)  \wt \ell_{\rm{c-cstnd}}(r, x, y)\\
&=  \sum_{y\in \sY} p(x, y) \sum_{k \in \sK} c(x, k, y) \Phi\paren*{-r(x, k)}\\
& = \sum_{k \in \sK} \wt q(x, k)\Phi\paren*{-r(x, k)},
\end{align*}
where $\wt q(x, k) =   \sum_{y\in \sY} p(x, y)  c(x, k, y) \in [0, 1]$. Let $k_{\min}(x)  =  \argmin_{k \in \sK} \wt q(x, k)$.
We denote by $\Phi_{\exp} \colon t \mapsto e^{-t}$ the exponential loss function, $\Phi_{\rm{sq-hinge}} \colon t \mapsto \max \curl*{0, 1 - t}^2$ the squared hinge loss function, $\Phi_{\rm{hinge}} \colon t
\mapsto \max \curl*{0, 1 - t}$ the hinge loss function, and $\Phi_{\rho} \colon t \mapsto \min\curl*{\max\curl*{0,
    1 - t/\rho}, 1}$, $\rho > 0$ the $\rho$-margin loss function.
\BoundCostCstnd*
\begin{proof}
\textbf{Case I: $\ell = \wt \ell^{\rm{cstnd}}_{\rm{c}{-\exp}}$.}
For the cost-sensitive constrained exponential loss $\wt \ell^{\rm{cstnd}}_{\rm{c}{-\exp}}$, the conditional regret can be written as 
\begin{align*}
\Delta\sC_{\wt \ell^{\rm{cstnd}}_{\rm{c}{-\exp}}, \sR}(r, x) 
& = \sum_{k \in \sK} \wt q(x, k) \Phi_{\exp}\paren*{-r(x, k)} - \inf_{r \in \sR} \sum_{k \in \sK} \wt q(x, k)\Phi_{\exp}\paren*{-r(x, k)}\\
& \geq \sum_{k \in \sK} \wt q(x, k)\Phi_{\exp}\paren*{-r(x, k)} - \inf_{\mu \in \Rset} \sum_{k \in \sK} \wt q(x, k)\Phi_{\exp}\paren*{-r_{\mu}(x, k)},
\end{align*}
where for any $k \in \sK$, $r_{\mu}(x, k) = \begin{cases}
r(x, y), & y \notin \curl*{k_{\min}(x), \rr(x)}\\
r(x, k_{\min}(x)) + \mu & y = \rr(x)\\
r(x, \rr(x)) - \mu & y = k_{\min}(x).
\end{cases}$
Note that such a choice of $r_{\mu}$ leads to the following equality holds:
\begin{equation*}
\sum_{k \notin \curl*{\rr(x), k_{\min}(x)}}  \wt q(x, k)\Phi_{\exp}\paren*{-r(x, k)} = \sum_{k \notin \curl*{\rr(x), k_{\min}(x)}}  \sum_{k \in \sK} \wt q(x, k)\Phi_{\exp}\paren*{-r_{\mu}(x, k)}.
\end{equation*}
Therefore, the conditional regret of cost-sensitive constrained exponential loss can be lower bounded as
\begin{align*}
& \Delta\sC_{\wt \ell^{\rm{cstnd}}_{\rm{c}{-\exp}}, \sR}(r, x)\\ 
& \geq \inf_{r \in \sR} \sup_{\mu\in \Rset} \curl*{\wt q(x, k_{\min}(x))\paren*{e^{r(x, k_{\min}(x))}-e^{r(x,\rr(x))-\mu}}+\wt q(x,\rr(x))\paren*{e^{r(x,\rr(x))}-e^{r(x, k_{\min}(x))+\mu}}}\\
& = \paren*{\sqrt{\wt q(x, k_{\min}(x))}-\sqrt{\wt q(x,\rr(x))}}^2 \tag{differentiating with respect to $\mu$, $r$ to optimize}\\
&   =   \paren*{\frac{\wt q(x,\rr(x)) - \wt q(x, k_{\min}(x))}{\sqrt{\wt q(x, k_{\min}(x))} + \sqrt{\wt q(x, \rr(x))}}}^2\\
& \geq \frac1{4} \paren*{\wt q(x,\rr(x)) - \wt q(x, k_{\min}(x))}^2 \tag{$0 \leq \wt q(x, k)\leq 1$}.
\end{align*}
Therefore, by Lemma~\ref{lemma:regret-target-cost}, the conditional regret of the target cardinality aware loss function can be upper bounded as follows:
\begin{equation*}
\Delta \sC_{\ell, \sH}(r, x) = \wt q(x, \rr(x)) - \wt q(x, k_{\min}(x)) \leq 2 \paren*{\Delta\sC_{\wt \ell^{\rm{cstnd}}_{\rm{c}{-\exp}}, \sR}(r, x) }^{\frac12}.
\end{equation*}
By the concavity, taking expectations on both sides of the preceding equation, we obtain
\begin{equation*}
\sE_{\ell}(r) - \sE^*_{\ell}(\sR) + \sM_{\ell}(\sR) \leq 2 \paren*{ \sE_{\wt \ell^{\rm{cstnd}}_{\rm{c}{-\exp}}}(r) - \sE^*_{\wt \ell^{\rm{cstnd}}_{\rm{c}{-\exp}}}(\sR) + \sM_{\wt \ell^{\rm{cstnd}}_{\rm{c}{-\exp}}}(\sR) }^{\frac12}.
\end{equation*}
The second part follows from the fact that 
$\sM_{\wt \ell^{\rm{cstnd}}_{\rm{c}{-\exp}}}(\sR_{\rm{all}}) = 0$.

\textbf{Case II: $\ell = \wt \ell_{c-\rm{sq-hinge}}$.}
For the cost-sensitive constrained squared hinge loss $\wt \ell_{c-\rm{sq-hinge}}$, the conditional regret can be written as 
\begin{align*}
\Delta\sC_{\wt \ell_{c-\rm{sq-hinge}}, \sR}(r, x) 
& = \sum_{k \in \sK} \wt q(x, k)\Phi_{\rm{sq-hinge}}\paren*{-r(x, k)} - \inf_{r \in \sR} \sum_{k \in \sK} \wt q(x, k)\Phi_{\rm{sq-hinge}}\paren*{-r(x, k)}\\
& \geq \sum_{k \in \sK} \wt q(x, k)\Phi_{\rm{sq-hinge}}\paren*{-r(x, k)} - \inf_{\mu \in \Rset} \sum_{k \in \sK} \wt q(x, k)\Phi_{\rm{sq-hinge}}\paren*{-r_{\mu}(x, k)},
\end{align*}
where for any $k \in \sK$, $r_{\mu}(x, k) = \begin{cases}
r(x, y), & y \notin \curl*{k_{\min}(x), \rr(x)}\\
r(x, k_{\min}(x)) + \mu & y = \rr(x)\\
r(x, \rr(x)) - \mu & y = k_{\min}(x).
\end{cases}$
Note that such a choice of $r_{\mu}$ leads to the following equality holds:
\begin{equation*}
\sum_{k \notin \curl*{\rr(x), k_{\min}(x)}}  \wt q(x, k)\Phi_{\rm{sq-hinge}}\paren*{-r(x, k)} = \sum_{k \notin \curl*{\rr(x), k_{\min}(x)}}  \sum_{k \in \sK} \wt q(x, k)\Phi_{\rm{sq-hinge}}\paren*{-r_{\mu}(x, k)}.
\end{equation*}
Therefore, the conditional regret of cost-sensitive constrained squared hinge loss can be lower bounded as
\begin{align*}
& \Delta\sC_{\wt \ell_{c-\rm{sq-hinge}}, \sR}(r, x)\\ 
& \geq \inf_{r \in \sR} \sup_{\mu\in \Rset} \bigg\{\wt q(x, k_{\min}(x))\paren*{\max\curl*{0, 1 + r(x, k_{\min}(x))}^2-\max\curl*{0, 1 + r(x,\rr(x))-\mu}^2 }\\
& \qquad + \wt q(x,\rr(x))\paren*{\max\curl*{0, 1 + r(x,\rr(x))}^2-\max\curl*{0, 1 + r(x, k_{\min}(x))+\mu}^2}\bigg\}\\
& \geq \frac14 \paren*{\wt q(x,k_{\min}(x))-\wt q(x, \rr(x))}^2
\tag{differentiating with respect to $\mu$, $r$ to optimize}.
\end{align*}
Therefore, by Lemma~\ref{lemma:regret-target-cost}, the conditional regret of the target cardinality aware loss function can be upper bounded as follows:
\begin{equation*}
\Delta \sC_{\ell, \sH}(r, x) = \wt q(x, \rr(x)) - \wt q(x, k_{\min}(x)) \leq 2 \paren*{\Delta\sC_{\wt \ell_{c-\rm{sq-hinge}}, \sR}(r, x) }^{\frac12}.
\end{equation*}
By the concavity, taking expectations on both sides of the preceding equation, we obtain
\begin{equation*}
\sE_{\ell}(r) - \sE^*_{\ell}(\sR) + \sM_{\ell}(\sR) \leq 2 \paren*{ \sE_{\wt \ell_{c-\rm{sq-hinge}}}(r) - \sE^*_{\wt \ell_{c-\rm{sq-hinge}}}(\sR) + \sM_{\wt \ell_{c-\rm{sq-hinge}}}(\sR) }^{\frac12}.
\end{equation*}
The second part follows from the fact that 
$\sM_{\wt \ell_{c-\rm{sq-hinge}}}(\sR_{\rm{all}}) = 0$.

\textbf{Case III: $\ell = \wt \ell_{c-\rm{hinge}}$.}
For the cost-sensitive constrained hinge loss $\wt \ell_{c-\rm{hinge}}$, the conditional regret can be written as 
\begin{align*}
\Delta\sC_{\wt \ell_{c-\rm{hinge}}, \sR}(r, x) 
& = \sum_{k \in \sK} \wt q(x, k)\Phi_{\rm{hinge}}\paren*{-r(x, k)} - \inf_{r \in \sR} \sum_{k \in \sK} \wt q(x, k)\Phi_{\rm{hinge}}\paren*{-r(x, k)}\\
& \geq \sum_{k \in \sK} \wt q(x, k)\Phi_{\rm{hinge}}\paren*{-r(x, k)} - \inf_{\mu \in \Rset} \sum_{k \in \sK} \wt q(x, k)\Phi_{\rm{hinge}}\paren*{-r_{\mu}(x, k)},
\end{align*}
where for any $k \in \sK$, $r_{\mu}(x, k) = \begin{cases}
r(x, y), & y \notin \curl*{k_{\min}(x), \rr(x)}\\
r(x, k_{\min}(x)) + \mu & y = \rr(x)\\
r(x, \rr(x)) - \mu & y = k_{\min}(x).
\end{cases}$
Note that such a choice of $r_{\mu}$ leads to the following equality holds:
\begin{equation*}
\sum_{k \notin \curl*{\rr(x), k_{\min}(x)}}  \wt q(x, k)\Phi_{\rm{hinge}}\paren*{-r(x, k)} = \sum_{k \notin \curl*{\rr(x), k_{\min}(x)}}  \sum_{k \in \sK} \wt q(x, k)\Phi_{\rm{hinge}}\paren*{-r_{\mu}(x, k)}.
\end{equation*}
Therefore, the conditional regret of cost-sensitive constrained hinge loss can be lower bounded as
\begin{align*}
& \Delta\sC_{\wt \ell_{c-\rm{hinge}}, \sR}(r, x)\\ 
& \geq \inf_{r \in \sR} \sup_{\mu\in \Rset} \bigg\{q(x, k_{\min}(x))\paren*{\max\curl*{0, 1 + r(x, k_{\min}(x))}-\max\curl*{0, 1 + r(x,\rr(x))-\mu} }\\
& \qquad + q(x,\rr(x))\paren*{\max\curl*{0, 1 + r(x,\rr(x))}-\max\curl*{0, 1 + r(x, k_{\min}(x))+\mu}}\bigg\}\\
& \geq  q(x,\rr(x))-q(x, k_{\min}(x))
\tag{differentiating with respect to $\mu$, $r$ to optimize}.
\end{align*}
Therefore, by Lemma~\ref{lemma:regret-target-cost}, the conditional regret of the target cardinality aware loss function can be upper bounded as follows:
\begin{equation*}
\Delta \sC_{\ell, \sH}(r, x) = \wt q(x, \rr(x)) - \wt q(x, k_{\min}(x)) \leq \Delta\sC_{\wt \ell_{c-\rm{hinge}}, \sR}(r, x).
\end{equation*}
By the concavity, taking expectations on both sides of the preceding equation, we obtain
\begin{equation*}
\sE_{\ell}(r) - \sE^*_{\ell}(\sR) + \sM_{\ell}(\sR) \leq \sE_{\wt \ell_{c-\rm{hinge}}}(r) - \sE^*_{\wt \ell_{c-\rm{hinge}}}(\sR) + \sM_{\wt \ell_{c-\rm{hinge}}}(\sR).
\end{equation*}
The second part follows from the fact that 
$\sM_{\wt \ell_{c-\rm{hinge}}}(\sR_{\rm{all}}) = 0$.

\textbf{Case IV: $\ell = \wt \ell_{c-\rho}$.}
For the cost-sensitive constrained $\rho$-margin loss $\wt \ell_{c-\rho}$, the conditional regret can be written as 
\begin{align*}
\Delta\sC_{\wt \ell_{c-\rho}, \sR}(r, x) 
& = \sum_{k \in \sK} \wt q(x, k)\Phi_{\rho}\paren*{-r(x, k)} - \inf_{r \in \sR} \sum_{k \in \sK} \wt q(x, k)\Phi_{\rho}\paren*{-r(x, k)}\\
& \geq \sum_{k \in \sK} \wt q(x, k)\Phi_{\rho}\paren*{-r(x, k)} - \inf_{\mu \in \Rset} \sum_{k \in \sK} \wt q(x, k)\Phi_{\rho}\paren*{-r_{\mu}(x, k)},
\end{align*}
where for any $k \in \sK$, $r_{\mu}(x, k) = \begin{cases}
r(x, y), & y \notin \curl*{k_{\min}(x), \rr(x)}\\
r(x, k_{\min}(x)) + \mu & y = \rr(x)\\
r(x, \rr(x)) - \mu & y = k_{\min}(x).
\end{cases}$
Note that such a choice of $r_{\mu}$ leads to the following equality holds:
\begin{equation*}
\sum_{k \notin \curl*{\rr(x), k_{\min}(x)}}  \wt q(x, k)\Phi_{\rho}\paren*{-r(x, k)} = \sum_{k \notin \curl*{\rr(x), k_{\min}(x)}}  \sum_{k \in \sK} \wt q(x, k)\Phi_{\rho}\paren*{-r_{\mu}(x, k)}.
\end{equation*}
Therefore, the conditional regret of cost-sensitive constrained $\rho$-margin loss can be lower bounded as
\begin{align*}
& \Delta\sC_{\wt \ell_{c-\rho}, \sR}(r, x)\\ 
& \geq \inf_{r \in \sR} \sup_{\mu\in \Rset} \bigg\{\wt q(x, k_{\min}(x))\paren*{\min\curl*{\max\curl*{0,1 + \frac{r(x, k_{\min}(x))}{\rho}},1}-\min\curl*{\max\curl*{0,1 + \frac{r(x,\rr(x))-\mu}{\rho}},1}}\\
&+\wt q(x,\rr(x))\paren*{\min\curl*{\max\curl*{0,1 + \frac{r(x,\rr(x))}{\rho}},1}-\min\curl*{\max\curl*{0,1 + \frac{r(x, k_{\min}(x))+\mu}{\rho}},1}}\bigg\}\\
& \geq \wt q(x,\rr(x))-\wt q(x, k_{\min}(x))
\tag{differentiating with respect to $\mu$, $r$ to optimize}.
\end{align*}
Therefore, by Lemma~\ref{lemma:regret-target-cost}, the conditional regret of the target cardinality aware loss function can be upper bounded as follows:
\begin{equation*}
\Delta \sC_{\ell, \sH}(r, x) = \wt q(x, \rr(x)) - \wt q(x, k_{\min}(x)) \leq \Delta\sC_{\wt \ell_{c-\rho}, \sR}(r, x).
\end{equation*}
By the concavity, taking expectations on both sides of the preceding equation, we obtain
\begin{equation*}
\sE_{\ell}(r) - \sE^*_{\ell}(\sR) + \sM_{\ell}(\sR) \leq \sE_{\wt \ell_{c-\rho}}(r) - \sE^*_{\wt \ell_{c-\rho}}(\sR) + \sM_{\wt \ell_{c-\rho}}(\sR).
\end{equation*}
The second part follows from the fact that 
$\sM_{\wt \ell_{c-\rho}}(\sR_{\rm{all}}) = 0$.
\end{proof}

\newpage

\section{Additional experimental results: top-\texorpdfstring{$k$}{k} classifiers}
\label{app:add}

\begin{figure}[t]
\vskip -.05in
\begin{center}
\begin{tabular}{@{}cc@{}}
\includegraphics[scale=0.35]{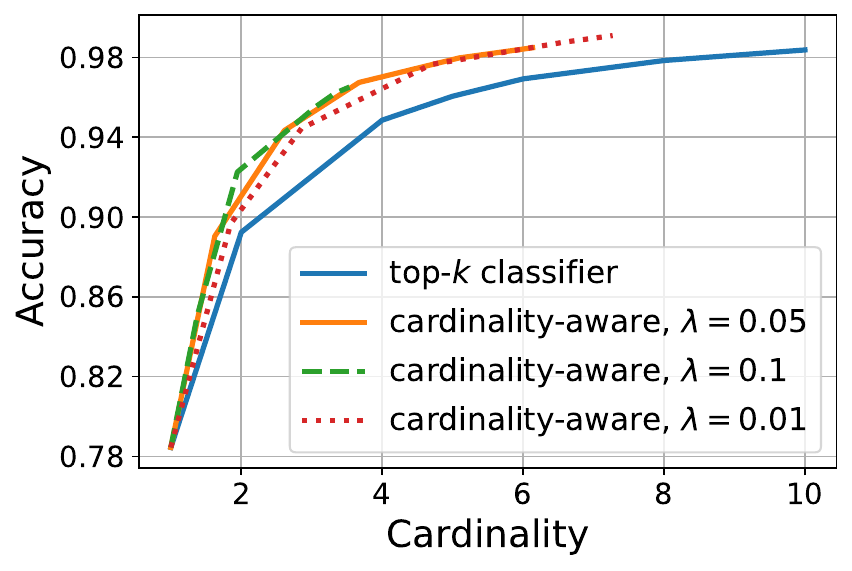}& 
\includegraphics[scale=0.35]{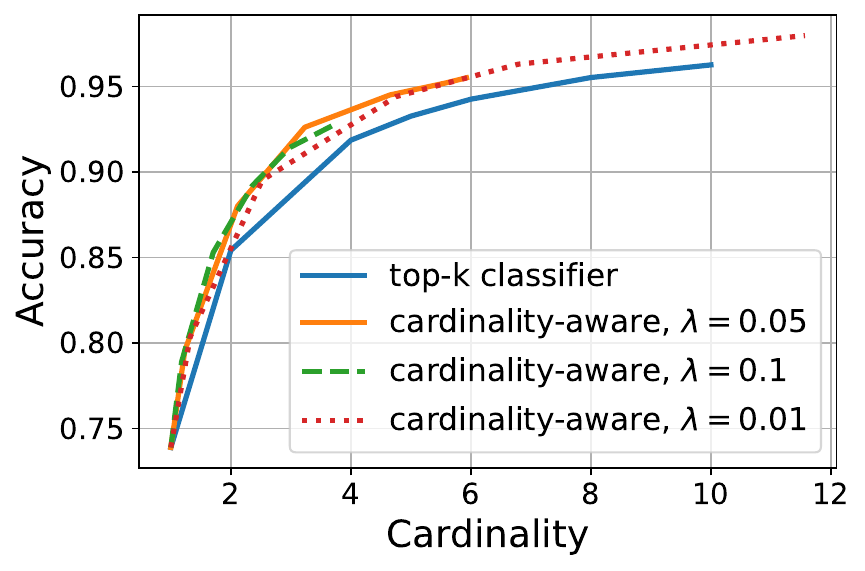}\\[-0.15cm]
{\small CIFAR-100} & {\small ImageNet} \\
\includegraphics[scale=0.35]{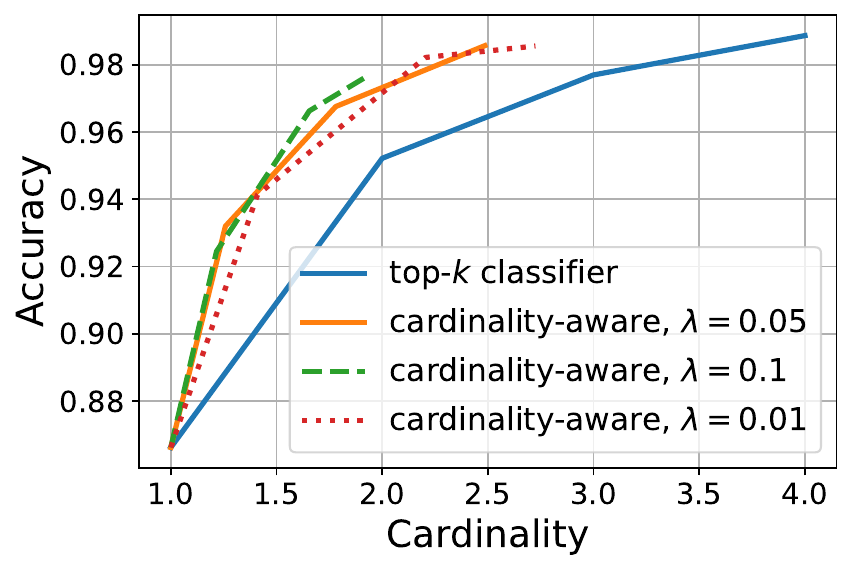}&
\includegraphics[scale=0.35]{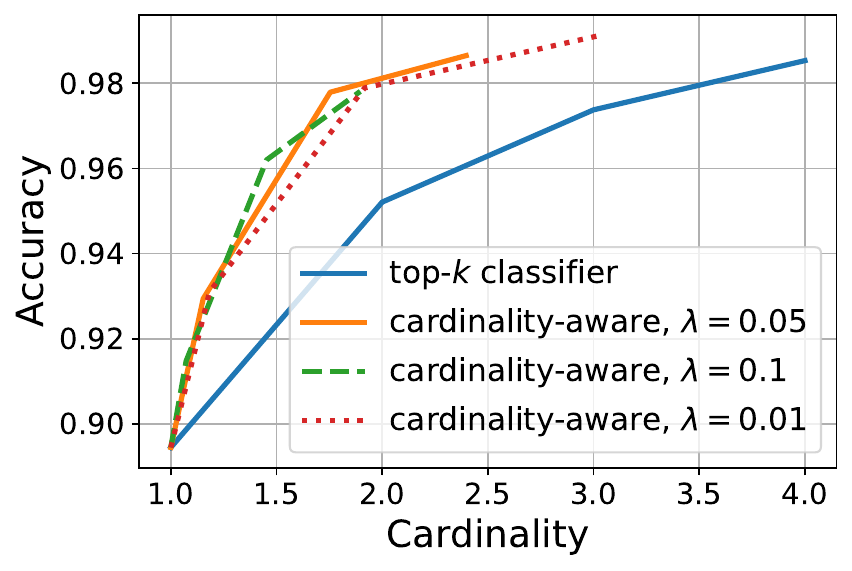}\\[-0.15cm]
{\small CIFAR-10} & {\small SVHN} 
\end{tabular}
\caption{Accuracy versus cardinality on various datasets for $\cost(\abs*{\g_k(x)}) = \log k$. Each curve of the cardinality-aware algorithm is for a fixed value of $\lambda$ and the points on the curve are obtained by varying the number of experts.}
\label{fig:topk_lambda}
\end{center}
\end{figure} 

\begin{figure}[t]
\begin{center}
\begin{tabular}{@{}cc@{}}
\includegraphics[scale=0.35]{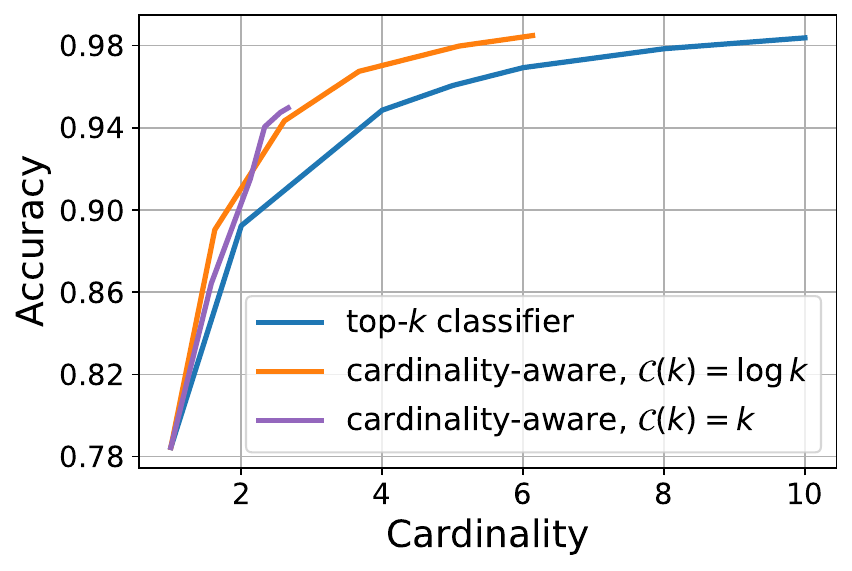}& 
\includegraphics[scale=0.35]{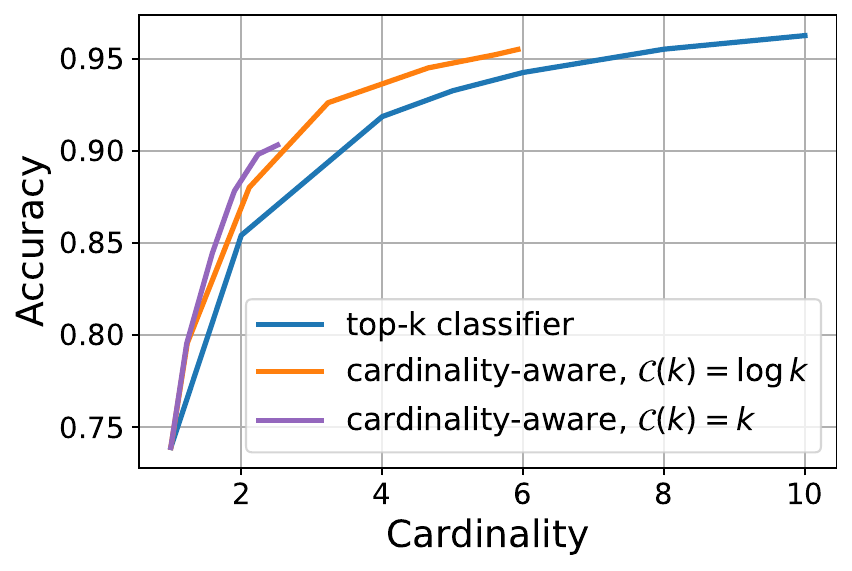}\\[-0.15cm]
{\small CIFAR-100} & {\small ImageNet} \\
\includegraphics[scale=0.35]{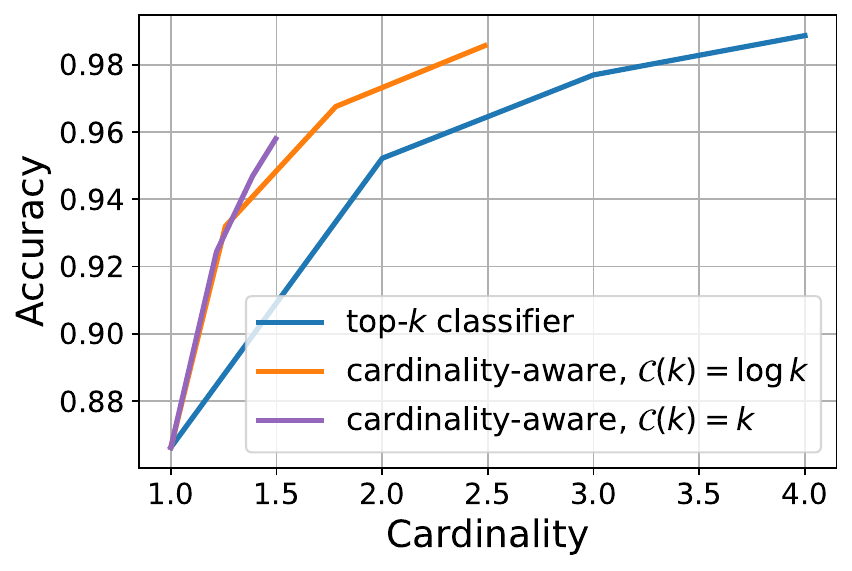}&
\includegraphics[scale=0.35]{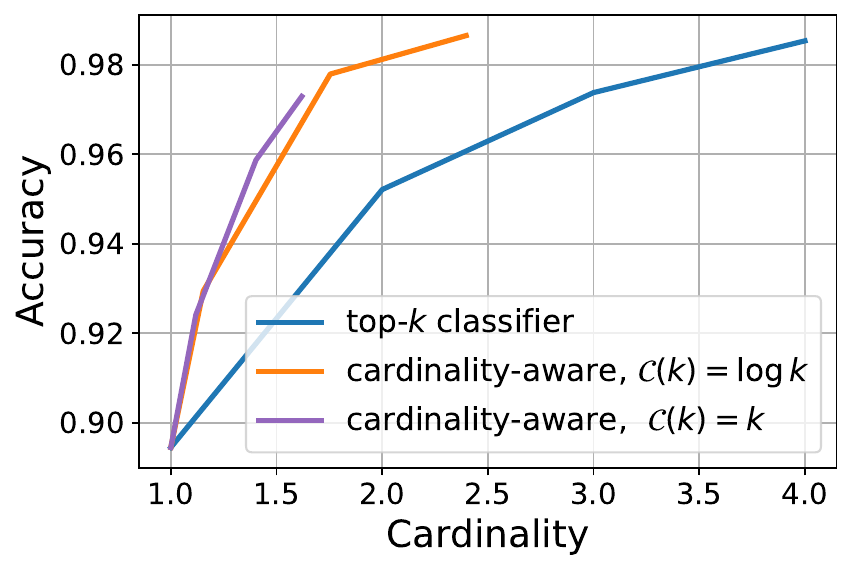}\\[-0.15cm]
{\small CIFAR-10} & {\small SVHN} 
\end{tabular}
\caption{Accuracy versus cardinality on various datasets for
  $\cost(\abs*{\g_k(x)}) = \log k$ and $\cost(\abs*{\g_k(x)}) = k$, with
  $\lambda = 0.05$. The points on each curve of the cardinality-aware algorithm are obtained by varying the number of experts.}
\label{fig:topk_Ck}
\end{center}
\vskip -0.2in
\end{figure} 

Here, we report additional experimental results with different choices
of set $\sK$ and $\cost(\abs*{\g_k(x)})$ on benchmark datasets CIFAR-10,
CIFAR-100 \citep{Krizhevsky09learningmultiple}, SVHN
\citep{Netzer2011}, and ImageNet \citep{deng2009imagenet} and show that
our cardinality-aware algorithm consistently outperforms top-$k$
classifiers across all configurations.

In Figure~\ref{fig:topk_lambda} and Figure~\ref{fig:topk_Ck}, we began
with a set $\sK = \curl*{1}$ for the loss function and then
progressively expanded it by adding choices of larger cardinality,
each of which doubles the largest value currently in $\sK$. The
largest set $\sK$ for the CIFAR-100 and ImageNet datasets is
$\curl*{1, 2, 4, 8, 16, 32, 64}$, whereas for the CIFAR-10 and SVHN
datasets, it is $\curl*{1, 2, 4, 8}$. As the set $\sK$ expands, there
is an increase in both the average cardinality and the
accuracy. 
Figure~\ref{fig:topk_lambda} shows that the accuracy versus
cardinality curve of the cardinality-aware algorithm is above that of
top-$k$ classifiers for various values of
$\lambda$. Figure~\ref{fig:topk_Ck} presents the comparison of
$\cost(\abs*{\g_k(x)}) = k$ and $\cost(\abs*{\g_k(x)}) = \log k$ for
$\lambda = 0.05$.
These results demonstrate that different $\lambda$ and different $\cost(\abs*{\g_k(x)})$ basically lead to the same curve, which verifies the effectiveness and benefit of our algorithm.

\newpage
\section{Additional experimental results: threshold-based classifiers}
\label{app:add-conformal}

We first characterize the Bayes predictor $r^*$ in this setting.
We say that the scenario is deterministic if for all $x \in \sX$, there exists some true label $y \in \sY$ such that $p(x, y) = 1$; otherwise, we say that the scenario is stochastic.
To simplify the
discussion, we will assume that $\abs*{\g_k(x)}$ is an increasing
function of $k$, for any $x$. 

\begin{lemma}
\label{lemma:det}
Consider the deterministic scenario.
Assume that $\lambda \cost(\abs*{\g_{k}(x)}) \leq 1$ for all $k$ and $x \in \sX$. Then, the Bayes predictor $r^*$ for the cardinality-aware loss function $\ell$ satisfies: $\rr^*(x) = \argmin_{k \colon y \in \g_k(x)} k $, that is the smallest $k$ such that the true label $y$ is in $\g_k(x)$.
\end{lemma}
\begin{proof}
By the assumption, for $k < \rr^*(x)$, we can write
$c(x, \rr^*(x), y) = \lambda \cost(\abs*{\g_{\rr^*(x)}(x)}) \leq 1 \leq 1_{y \not \in \g_k(x)} + \lambda \cost(\abs*{\g_k(x)}) = c(x, k, y)$.
Furthermore, since $\abs*{\g_k(x)}$ is an increasing
function of $k$, we have
$c(x, \rr^*(x), y) = \lambda \cost(\abs*{\g_{\rr^*(x)}(x)}) \leq \lambda \cost(\abs*{\g_k'(x)}) = c(x, k', y)$ for $k' > \rr^*(x)$.
\end{proof}

\begin{figure}[t]
\vskip -.05in
\begin{center}
\includegraphics[scale=0.35]{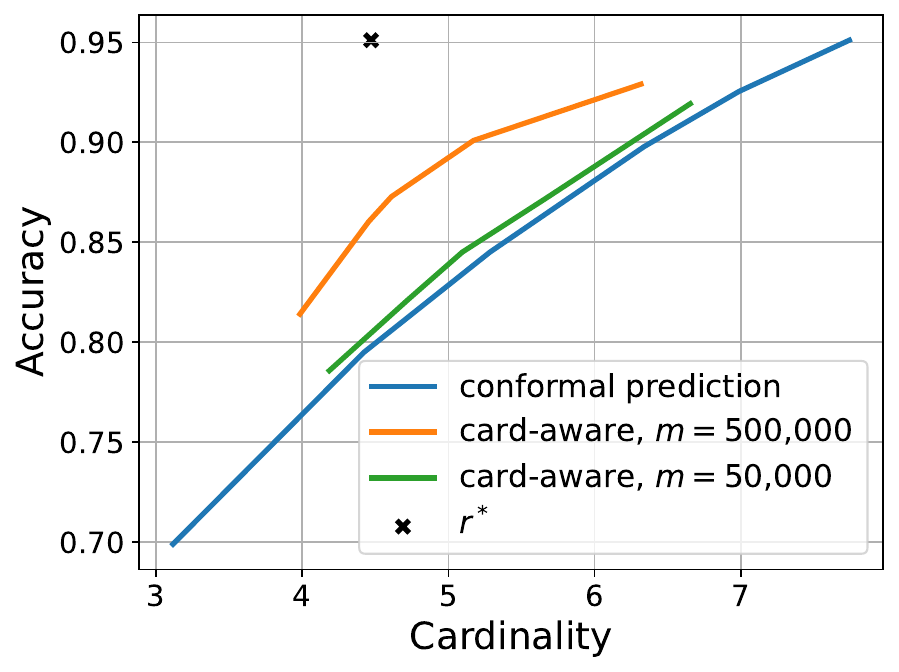}
\caption{Accuracy versus cardinality on an artificial dataset for different training sample sizes $m$.}
\label{fig:artificial}
\end{center}
\vskip -0.2in
\end{figure}

\begin{lemma}
\label{lemma:sto}
Consider the stochastic scenario. The Bayes predictor $r^*$ for the cardinality-aware loss function $\ell$ satisfies: \[\rr^*(x) = \argmin_{k \in \sK} \paren*{\lambda \cost(\abs*{\g_k(x)}) - \sum_{y \in \g_{k}(x)} p(x, y) }.\] 
\end{lemma}
\begin{proof}
The conditional error can be written as follows:
\begin{align*}
\sC_{\ell}(r, x, y) 
& = \sum_{y \in \sY}p(x, y)  c(x, \rr(x), y)\\
& = \sum_{y \in \sY}p(x, y) \paren*{1_{y \notin \g_{\rr(x)}(x)} + \lambda \cost(\abs*{\g_{\rr(x)}(x)})}\\
& = \sum_{y \in \sY}p(x, y) 1_{y \notin \g_{\rr(x)}(x)}  + \lambda \cost(\abs*{\g_{\rr(x)}(x)})\\
& = 1 - \sum_{y \in \g_{\rr(x)}(x)} p(x, y) + \lambda \cost(\abs*{\g_{\rr(x)}(x)}).
\end{align*}
Thus, the Bayes predictor can be characterized as 
\[\rr^*(x) = \argmin_{k \in \sK} \paren*{\lambda \cost(\abs*{\g_k(x)}) - \sum_{y \in \g_{k}(x)} p(x, y) }.\] 
\end{proof}
It is clear that Lemma~\ref{lemma:sto} implies Lemma~\ref{lemma:det} when there exists some true $y \in \sY$ such that $p(x, y) = 1$ and $\lambda \cost(\abs*{\g_k(x)}) \leq 1$. 

We first consider an artificial dataset containing 10 classes. Each class is modeled by a Gaussian distribution in a 100-dimensional space. As in Section~\ref{sec:experiments}, we plot the accuracy versus cardinality curve of the cardinality-aware algorithm by varying $\lambda$, where the set predictors used are threshold-based classifiers, and compare with that of conformal prediction. In Figure~\ref{fig:artificial}, we also indicate the point corresponding to $r^*$. 
The problem is close to being realizable, as we can train a predictor that performs almost as well as $r^*$ on the test set. Thus, the minimizability gaps vanish, and our $\sH$-consistency bounds (Theorems~\ref{thm:bound-cost-comp} and \ref{thm:bound-cost-cstnd}) then suggest that with sufficient training data, we can get close to the optimal solution and therefore outperform conformal prediction. For some tasks, however, the problem is hard, and it appears that a very large training sample would be needed. Figure~\ref{fig:artificial} demonstrates that on the artificial dataset, with training sample size $m = 50{,}000$, the performance of our cardinality-aware algorithm is only slightly better than that of conformal prediction. If we increase the training sample size to $m = 500{,}000$, then the curve of our algorithm becomes much closer to the optimal point and significantly outperforms conformal prediction.

\begin{figure}[t]
\vskip -.05in
\begin{center}
\begin{tabular}{@{}cc@{}}
\includegraphics[scale=0.35]{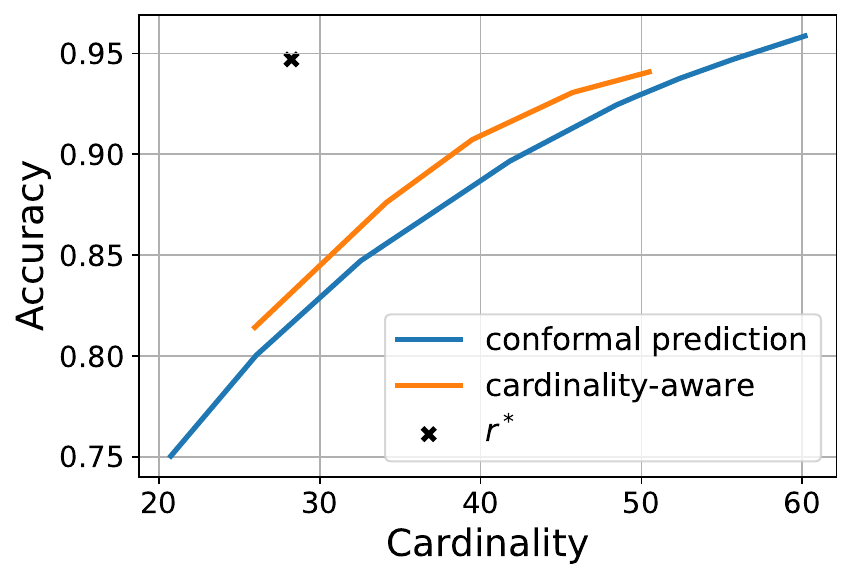}& 
\includegraphics[scale=0.35]{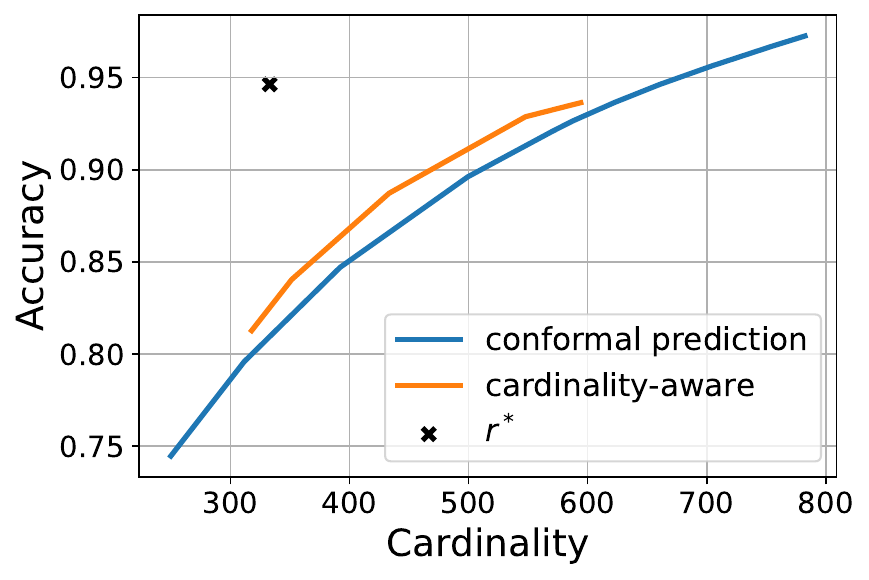}\\[-0.15cm]
{\small CIFAR-100} & {\small ImageNet} \\
\includegraphics[scale=0.35]{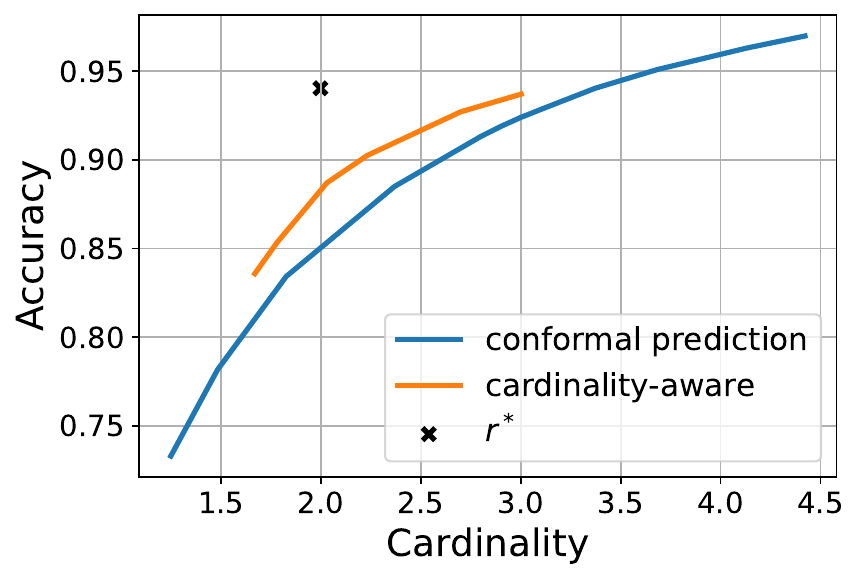}&
\includegraphics[scale=0.35]{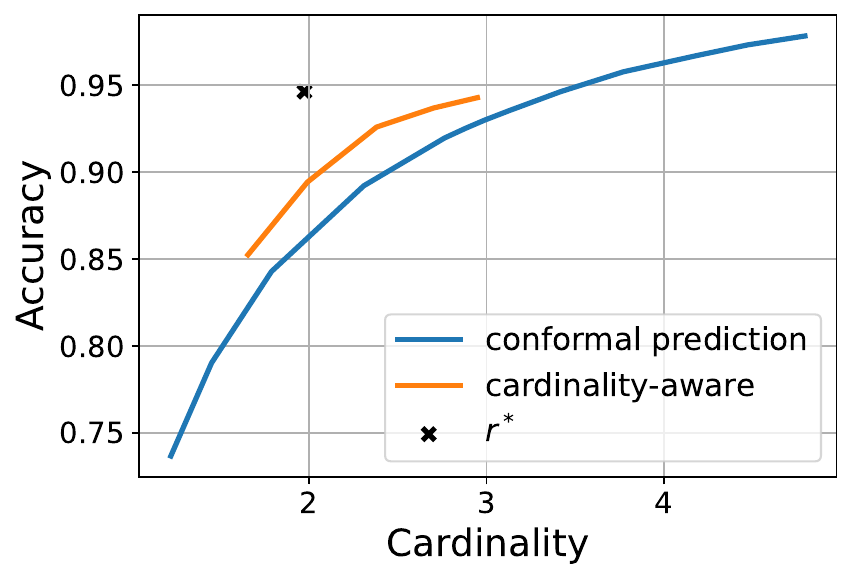}\\[-0.15cm]
{\small CIFAR-10} & {\small SVHN} 
\end{tabular}
\caption{Accuracy versus cardinality on CIFAR-100, ImageNet, CIFAR-10, and SVHN datasets.}
\label{fig:thresh_lambda}
\end{center}
\vskip -0.2in
\end{figure} 

Additionally, for a weaker scoring function, a smaller training sample suffices in many cases, and our cardinality-aware algorithm can outperform conformal prediction on real datasets as shown in Figure~\ref{fig:thresh_lambda}.

\section{Future work}
\label{app:future_work}

While our framework of cardinality-aware set prediction is very general—applicable to any collection of set predictors (Section~\ref{sec:cardinality})—and leads to novel cardinality-aware algorithms (Section~\ref{sec:cardinality-aware-algorithms}), benefits from theoretical guarantees with sufficient training data (Section~\ref{sec:pre}), and demonstrates effectiveness and empirical advantages in top-$k$ classification (Section~\ref{sec:experiments}), the learning problem can be challenging for certain tasks, often requiring a very large training sample (as shown in Appendix~\ref{app:add-conformal}). This underscores the need for a more detailed investigation to enhance our algorithms in these scenarios.

\end{document}